\documentclass[sigconf]{acmart}

\AtBeginDocument{%
  }

\copyrightyear{2025}
\acmYear{2025}
\setcopyright{acmlicensed}\acmConference[MM '25]{Proceedings of the 33rd ACM International Conference on Multimedia}{October 27--31, 2025}{Dublin, Ireland}
\acmBooktitle{Proceedings of the 33rd ACM International Conference on Multimedia (MM '25), October 27--31, 2025, Dublin, Ireland}
\acmDOI{10.1145/3746027.3754582}
\acmISBN{979-8-4007-2035-2/2025/10}

\acmSubmissionID{7272}

\usepackage{algorithm,algorithmic}
\usepackage{amsmath, amsthm, bm}
\usepackage{thmtools, thm-restate}

\usepackage{booktabs}
\usepackage{xkcdcolors}
\usepackage{graphicx}
\usepackage{amsmath}
\usepackage{mathtools}
\usepackage{amsthm}
\usepackage{multirow}
\usepackage{colortbl}
\usepackage{algorithm,algorithmic}
\usepackage[capitalize,noabbrev]{cleveref}
\usepackage{mathtools}
\usepackage{tikz}
\usepackage{tabularx}
\usepackage{pifont}

\newcommand{\cmark}{\ding{51}}%
\newcommand{\xmark}{\ding{55}}%

\newcommand{\Acal}{\mathcal{A}}
\newcommand{\Real}{\mathbb{R}}
\newcommand{\norm}[1]{\|#1\|_2} 
\newcommand{\wprime}{\mathbf{w}'} 

\theoremstyle{plain}
\newtheorem{theorem}{Theorem}[section]
\newtheorem{lemma}[theorem]{Lemma}

\newtheorem{proposition}[theorem]{Proposition}
\theoremstyle{definition}

\theoremstyle{remark}

\usepackage{subcaption} 

\newcommand{\I}{\mathbf{I}}

\newcommand{\g}{\mathbf{g}}

\newcommand{\x}{\mathbf{x}}

\newcommand{\y}{\mathbf{y}}

\newcommand{\w}{\mathbf{w}}

\newcommand{\thetab}{\boldsymbol{\theta}}
\newcommand{\epsilonb}{\boldsymbol{\epsilon}}

\usepackage{tikz}
\usetikzlibrary{positioning}
\usepackage{capt-of}
\usepackage{diagbox}

\definecolor{C0}{rgb}{0.121569, 0.466667, 0.705882}
\definecolor{C1}{rgb}{1.000000, 0.498039, 0.054902}
\definecolor{C2}{rgb}{0.172549, 0.627451, 0.172549}
\definecolor{C3}{rgb}{0.839216, 0.152941, 0.156863}
\definecolor{C4}{rgb}{0.580392, 0.403922, 0.741176}
\definecolor{C5}{rgb}{0.549020, 0.337255, 0.294118}
\definecolor{C6}{rgb}{0.890196, 0.466667, 0.760784}
\definecolor{C7}{rgb}{0.498039, 0.498039, 0.498039}
\definecolor{C8}{rgb}{0.737255, 0.741176, 0.133333}
\definecolor{C9}{rgb}{0.090196, 0.745098, 0.811765}

\definecolor{trolleygrey}{rgb}{0.5, 0.5, 0.5}
\definecolor{darkpastelgreen}{rgb}{0.01, 0.75, 0.24}
\definecolor{darkpink}{rgb}{0.91, 0.33, 0.5}
\definecolor{alizarin}{rgb}{0.82, 0.1, 0.26}
\definecolor{americanrose}{rgb}{1.0, 0.01, 0.24}

\definecolor{myblue}{rgb}{0, 0.45, 0.74}
\definecolor{myorange}{rgb}{0.85, 0.32, 0.1}
\definecolor{mygreen}{rgb}{0, 0.62, 0.45}

\definecolor{lightgrey}{rgb}{0.85, 0.85, 0.85}
\definecolor{lightgreen}{rgb}{0.85, 0.95, 0.85}
\definecolor{lightyellow}{rgb}{0.97, 0.97, 0.85}
\definecolor{lightorange}{rgb}{0.98, 0.92, 0.84}
\definecolor{lightpink}{rgb}{0.98, 0.88, 0.90}

\definecolor{extremelylightcolor}{rgb}{0.98, 0.90, 0.86}
\definecolor{extremelylightblue}{rgb}{0.90, 0.95, 0.98}
\definecolor{extremelylightgreen}{rgb}{0.90, 0.98, 0.92}
\definecolor{extremelylightmint}{rgb}{0.90, 0.98, 0.95}
\definecolor{lightblue}{rgb}{0.90, 0.95, 0.98}

\begin{document}

\title{Enhancing Diffusion Model Stability for Image Restoration via Gradient Management}


\author{Hongjie Wu}  \orcid{0009-0007-3203-1521}
 \email{wuhongjie0818@gmail.com}
\affiliation{
 \institution{College of Computer Science,\\ Sichuan University} 
 \city{Chengdu} \country{China}
 }
 \additionalaffiliation{\institution{Engineering Research Center of Machine Learning and Industry Intelligence, Ministry of Education, China}}

\author{Mingqin Zhang}  \orcid{0009-0008-7273-9302}  
\authornotemark[1] 
 \email{zhangmingqin@stu.scu.edu.cn}
 \affiliation{%
\institution{College of Computer Science,\\ Sichuan University} 
 \city{Chengdu} \country{China}
 }
 
\author{Linchao He}   \orcid{0000-0002-0562-3026}
 \email{hlc@stu.scu.edu.cn}
\affiliation{
 \institution{National Key Laboratory of Fundamental Science on Synthetic Vision, Sichuan University} 
 \city{Chengdu} \country{China}
 }

  \author{Ji-Zhe Zhou} 
  \authornotemark[1]
   \email{jzzhou@scu.edu.cn} 
   \orcid{0000-0002-2447-1806}
\affiliation{
\institution{College of Computer Science,\\ Sichuan University}  
 \city{Chengdu} \country{China}
 }

\author{Jiancheng Lv} 
\authornotemark[1]  
\orcid{0000-0001-6551-3884}
 \email{lvjiancheng@scu.edu.cn}
\authornote{Corresponding author.}
\affiliation{
\institution{College of Computer Science,\\ Sichuan University}  
 \city{Chengdu}
 \country{China}
 }

\renewcommand\shortauthors{Hongjie Wu et al.}

\begin{abstract}
Diffusion models have shown remarkable promise for image restoration by leveraging powerful priors. Prominent methods typically frame the restoration problem within a Bayesian inference framework, which iteratively combines a denoising step with a likelihood guidance step. 
However, the interactions between these two components in the generation process remain underexplored.
In this paper, we analyze the underlying gradient dynamics of these components and identify significant instabilities. Specifically, we demonstrate conflicts between the prior and likelihood gradient directions, alongside temporal fluctuations in the likelihood gradient itself. 
We show that these instabilities disrupt the generative process and compromise restoration performance.
To address these issues, we propose Stabilized Progressive Gradient Diffusion (SPGD), a novel gradient management technique. SPGD integrates two synergistic components: (1) a progressive likelihood warm-up strategy to mitigate gradient conflicts; and (2) adaptive directional momentum (ADM) smoothing to reduce fluctuations in the likelihood gradient. Extensive experiments across diverse restoration tasks demonstrate that SPGD significantly enhances generation stability, leading to state-of-the-art performance in quantitative metrics and visually superior results. 
Code is available at \href{https://github.com/74587887/SPGD}{https://github.com/74587887/SPGD}.

\end{abstract}





\maketitle

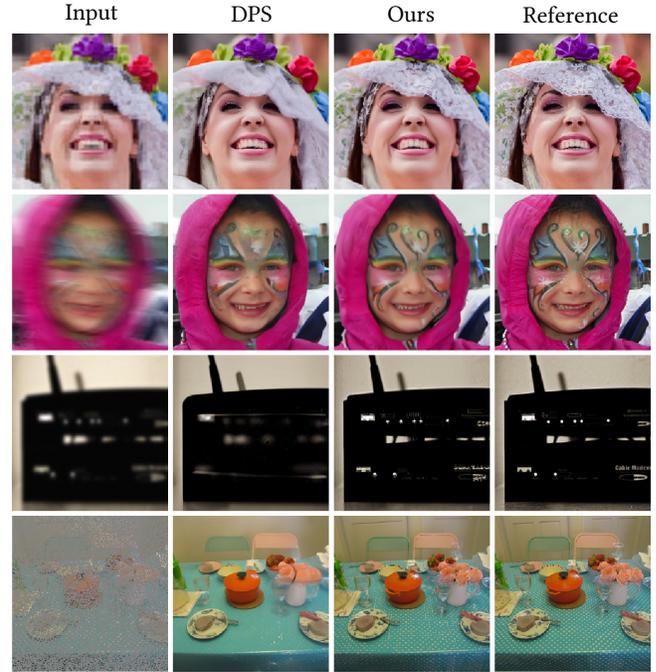
\begin{figure}[t]
    \centering
  \begin{tikzpicture}
        \node[anchor=south west,inner sep=0] (image) at (0,0) {\includegraphics[width=\linewidth]{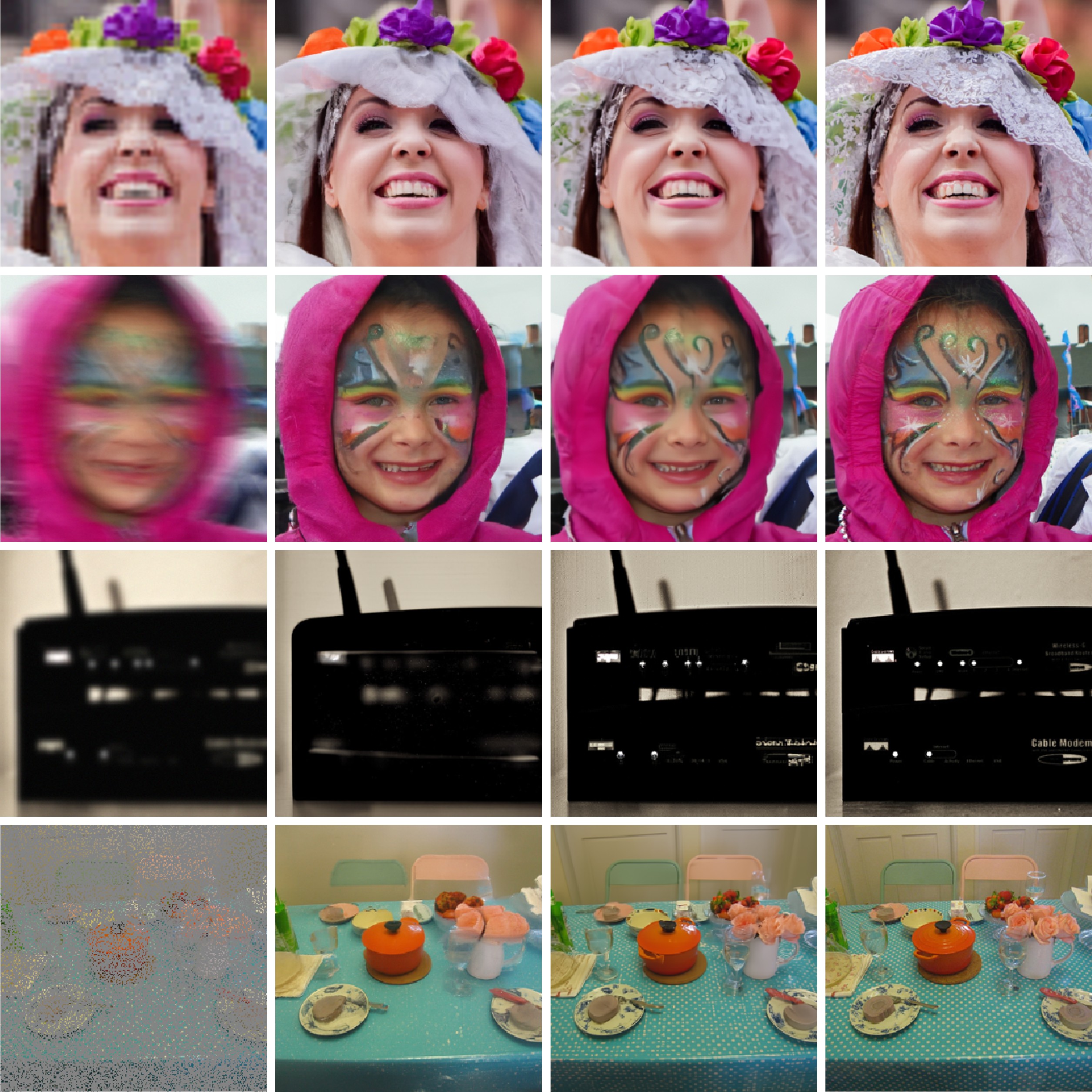}};
        \begin{scope}[x={(image.south east)},y={(image.north west)}]
            \node at (0.12,1.03) {{ Input}};
            \node at (0.375,1.03) {{DPS}};
            \node at (0.62,1.03) {{ Ours}};
            \node at (0.875,1.03) {{Reference}};
        \end{scope}
    \end{tikzpicture}
        \caption{Illustrative comparison of diffusion-based restoration results: DPS (without gradient management), Ours (with the proposed gradient management), and the reference. Our method demonstrates improved restoration quality, particularly in challenging restoration scenarios.}
    \label{fig:conceptual_intro}
    \vspace{-0.3cm}
\end{figure}

\section{Introduction}
Image restoration~\cite{zhang2017beyond,lehtinen2018noise2noise,liang2021swinir}, encompassing tasks such as inpainting~\cite{yu2018generative,elharrouss2020image}, deblurring~\cite{chen2024hierarchical,wu2024id}, and super-resolution~\cite{saharia2022image,wang2024exploiting}, is a fundamental problem in computer vision and image processing. The goal is to recover a high-quality, clean image $\mathbf{x}_0$ from a degraded measurement $\mathbf{y}$, which is often modeled by the following equation:
\begin{equation}
    \mathbf{y} = \mathcal{A}(\mathbf{x}_0) + \mathbf{n},
\label{eq:degradation_model}
\end{equation}
where $\mathcal{A}$ represents a known degradation operator (e.g., a blurring kernel, a downsampling matrix, or a masking operator), and $\mathbf{n} \sim \mathcal{N}(\mathbf{0},{\sigma_y^2}\mathbf{I})$ represents additive noise. Traditional approaches to solve this ill-posed inverse problem~\cite{milanfar2012tour, zhang2017beyond, tai2017memnet, zhang2021plug} often rely on handcrafted priors and regularization terms, which may not fully capture the complex statistics of natural images.

Diffusion models~\citep{sohl2015deep, ho2020denoising, song2021scorebased} have emerged as a transformative force in generative modeling, achieving state-of-the-art results in image synthesis~\citep{rombach2022high,poole2022dreamfusion,saharia2022palette,meng2023distillation}. Their remarkable ability to model complex data distribution and generating realistic images, has naturally led to significant interest in adapting them for image restoration~\citep{choi2021ilvr, wen2024unpaired,lugmayr2022repaint,carrillo2023diffusart,cao2024deep}.
Among these methods, a dominant paradigm leverages the Bayesian framework~\citep{song2022solving,chung2023diffusion,song2023pseudoinverseguided,dou2024diffusion}, guiding the reverse diffusion process using the conditional score $\nabla_{\mathbf{x}_t} \log p_t(\mathbf{x}_t|\mathbf{y})$. This score elegantly decomposes via Bayes' theorem into two key components:
\begin{equation}
\underbrace{\nabla_{\mathbf{x}_t} \log p_t(\mathbf{x}_t|\mathbf{y})}_{\text{Conditional Score}} = \underbrace{\nabla_{\mathbf{x}_t} \log p_t(\mathbf{x}_t)}_{\text{Prior Score}} + \underbrace{\nabla_{\mathbf{x}_t} \log p_t(\mathbf{y}|\mathbf{x}_t)}_{\text{Likelihood Guidance}}.
\label{eq:bayes}
\end{equation}
Here, the prior score enforces conformity to the learned distribution of natural images, while the likelihood guidance ensures consistency with the observed measurement $\mathbf{y}$.


Existing methods~\citep{chung2023diffusion, song2023pseudoinverseguided, peng2024improving, wu2024diffusion} typically apply the gradients corresponding to the prior score and likelihood guidance iteratively, without explicitly accounting for their potentially detrimental interactions. 
However, despite the success of these approaches, we argue that their performance is limited by inherent \textbf{conflicts and instabilities within the gradient dynamics} governing the reverse process. Our detailed analysis in \cref{sec:Gradient}, which examines the reverse update from a gradient decomposition perspective, reveals two critical issues:
\begin{enumerate}
    \item \textbf{Gradient Conflict:} The gradient enforcing data consistency (likelihood guidance) often conflicts with the gradient enforcing the natural image prior in direction, particularly in the early stages of generation.
    \item \textbf{Gradient Fluctuation:} The likelihood guidance gradient exhibits significant temporal instability, fluctuating abruptly between consecutive timesteps.
\end{enumerate}
Such issues can hinder the effective combination of prior knowledge and data fidelity, leading to suboptimal updates and ultimately compromising generation quality. Addressing these instabilities is essential for unlocking the full potential of diffusion models in demanding restoration tasks.

This paper introduces \textbf{Stabilized Progressive Gradient Diffusion (SPGD)}, a novel gradient management technique designed to address these identified instabilities. SPGD integrates two synergistic components into the reverse diffusion step:
\begin{enumerate}
\item \textbf{Progressive Likelihood Warm-Up:} To mitigate gradient conflict, we introduce multiple small, iterative updates guided only by the likelihood gradient before each main denoising step. This allows the estimate to progressively adapt to measurement constraints, reducing potential directional clashes with the prior gradient.
\item \textbf{Adaptive Directional Momentum (ADM) Smoothing:} To counteract likelihood gradient instability, we employ ADM smoothing. This technique dynamically adjusts momentum based on directional consistency, effectively dampening erratic fluctuations and promoting more stable update guidance from the likelihood term. 
\end{enumerate}
These components work in concert to stabilize the reverse trajectory, leading to more effective and robust optimization, even with fewer computational resources.
Extensive experiments across diverse restoration tasks demonstrate that SPGD significantly enhances reconstruction quality, yielding substantial improvements in restoration metrics alongside visually superior results.

In summary, our primary contributions are:
\begin{itemize}
    \item We present a novel gradient-centric perspective on the reverse process in diffusion models for image restoration, identifying the critical, yet previously under-addressed, issues of gradient conflict and fluctuation.
    \item We propose Stabilized Progressive Gradient Diffusion (SPGD), a novel gradient management technique incorporating a progressive warm-up phase and momentum-based smoothing to mitigate these issues.
    \item We conduct extensive experiments on various image restoration tasks across two different datasets, demonstrating that SPGD significantly improves reconstruction quality both quantitatively and qualitatively.
\end{itemize}

\section{Background}
\label{sec:Background}

\begin{figure*}[t]
    \centering
        \begin{subfigure}{\textwidth}
        \centering
        \includegraphics[width=0.24\linewidth]{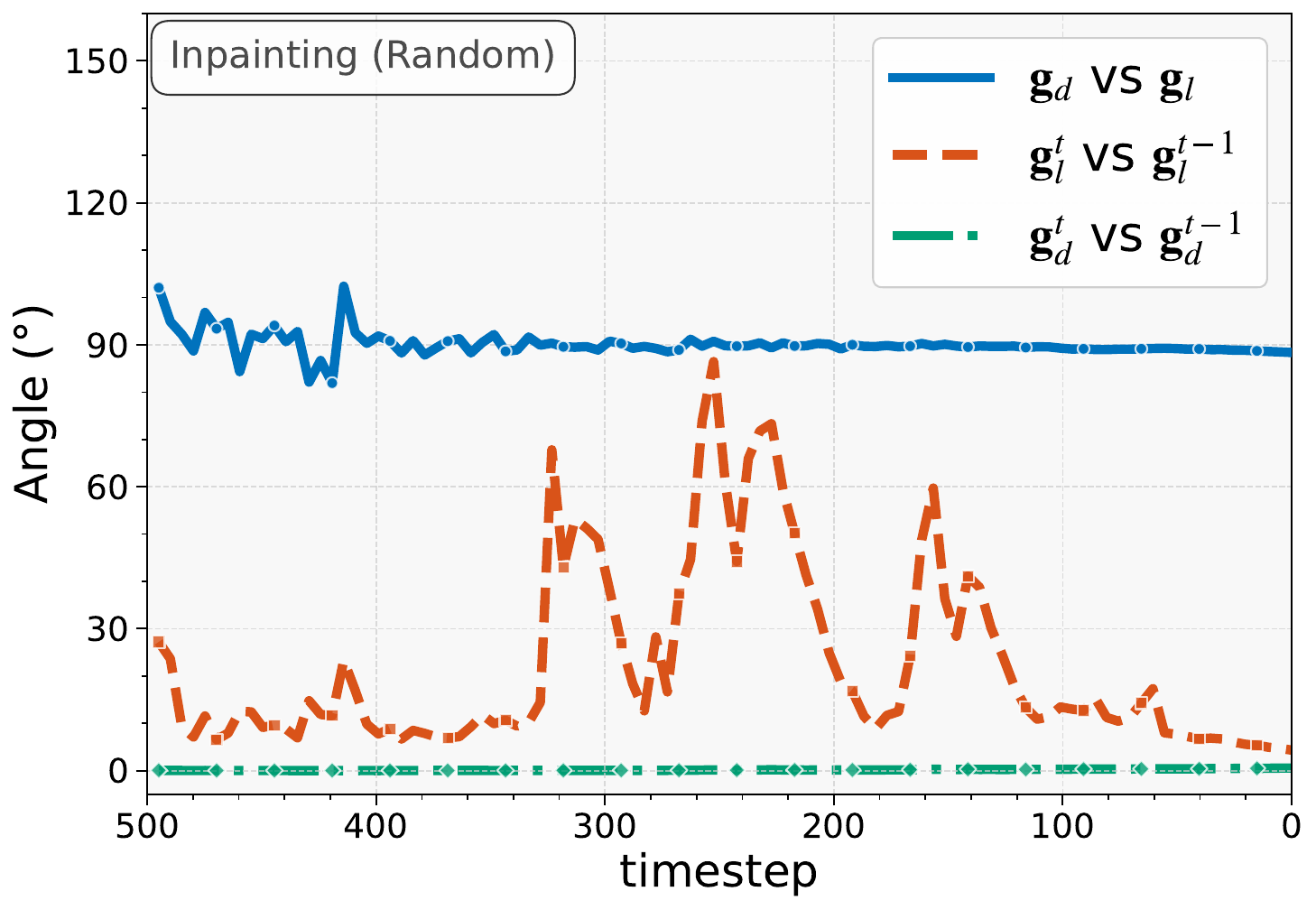}
        \includegraphics[width=0.24\linewidth]{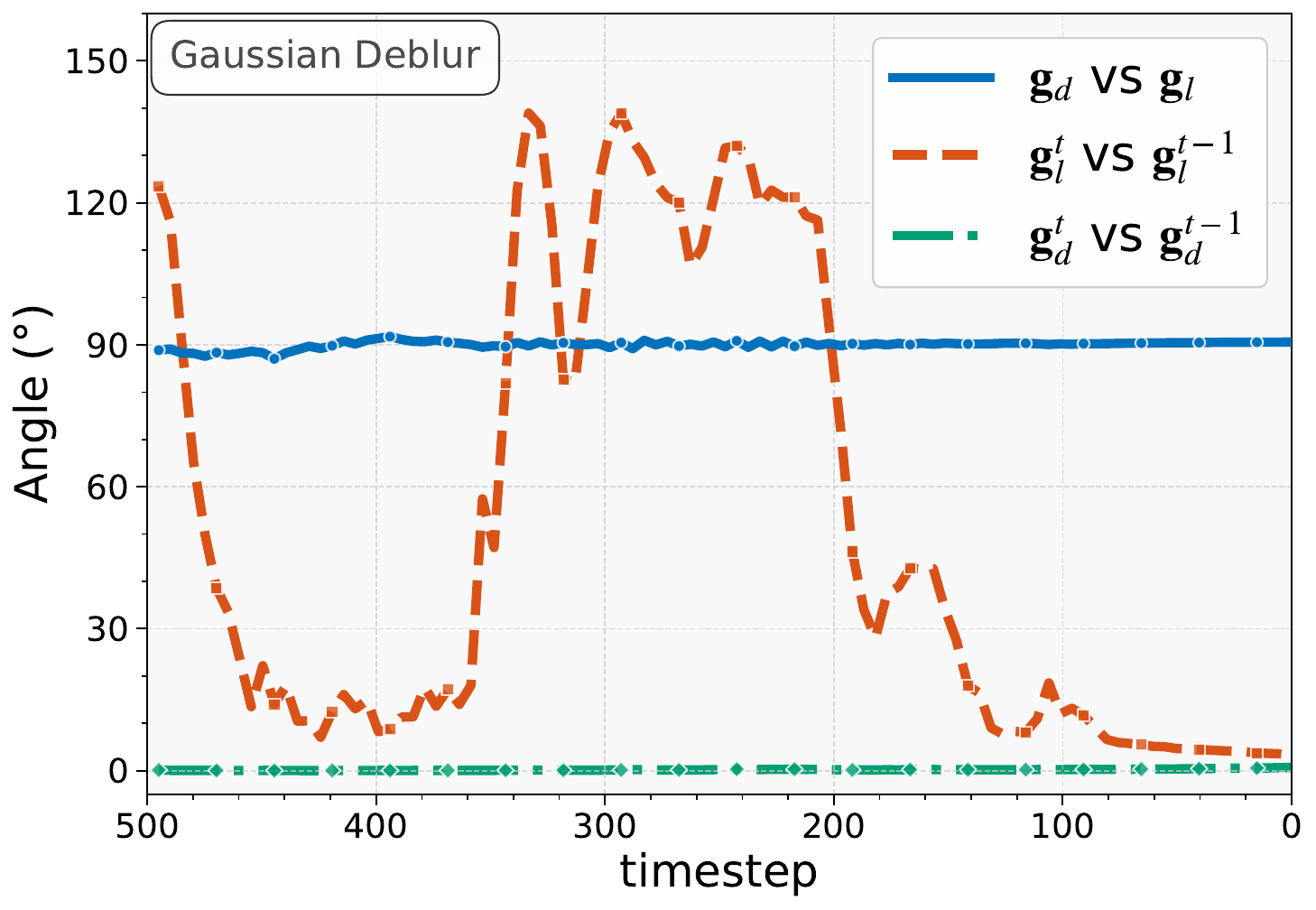}
        \includegraphics[width=0.24\linewidth]{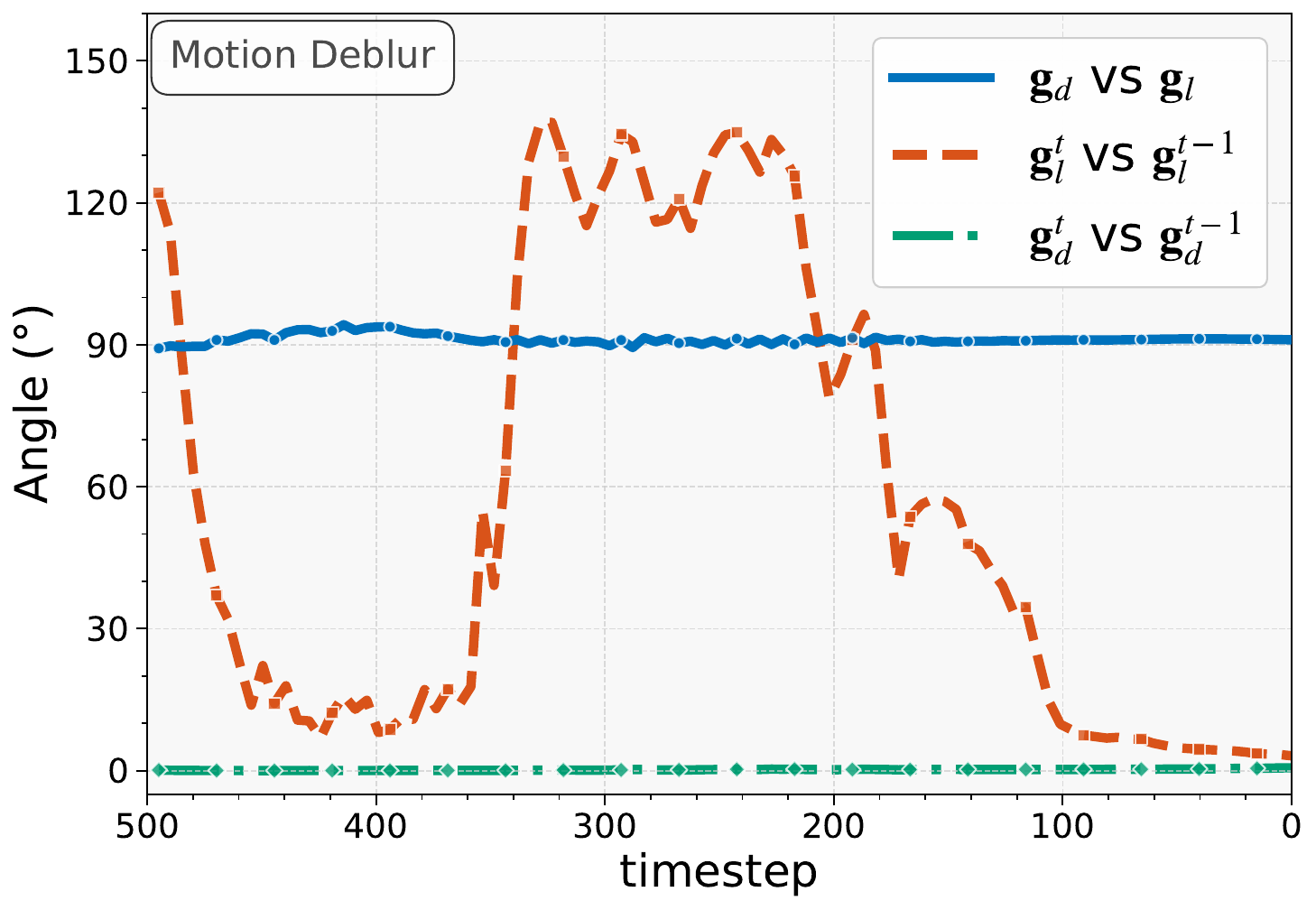}
        \includegraphics[width=0.24\linewidth]{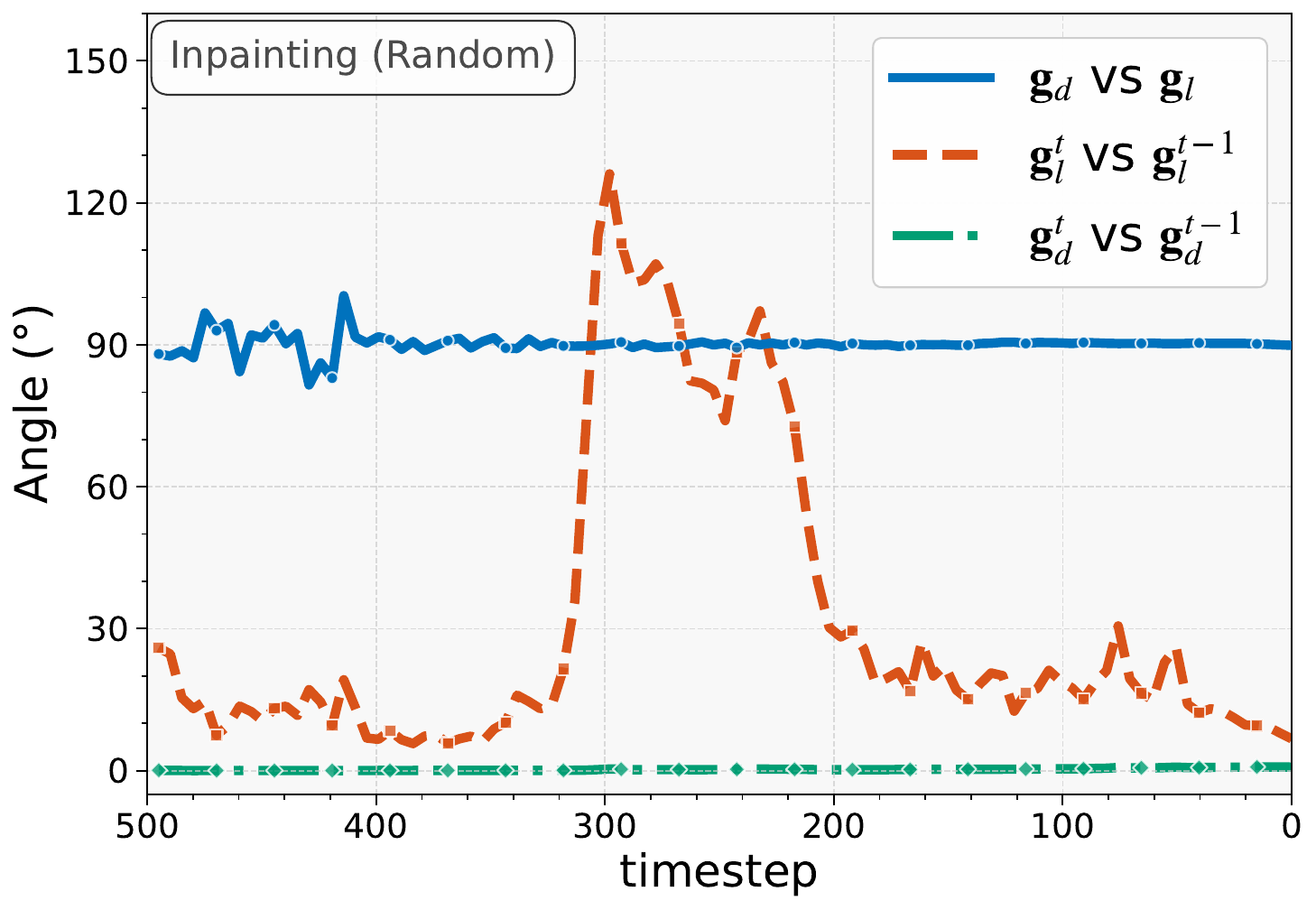}
        \caption{Gradient angles dynamics without smoothing.}
    \end{subfigure}
    \begin{subfigure}{\textwidth}
        \centering
        \includegraphics[width=0.24\linewidth]{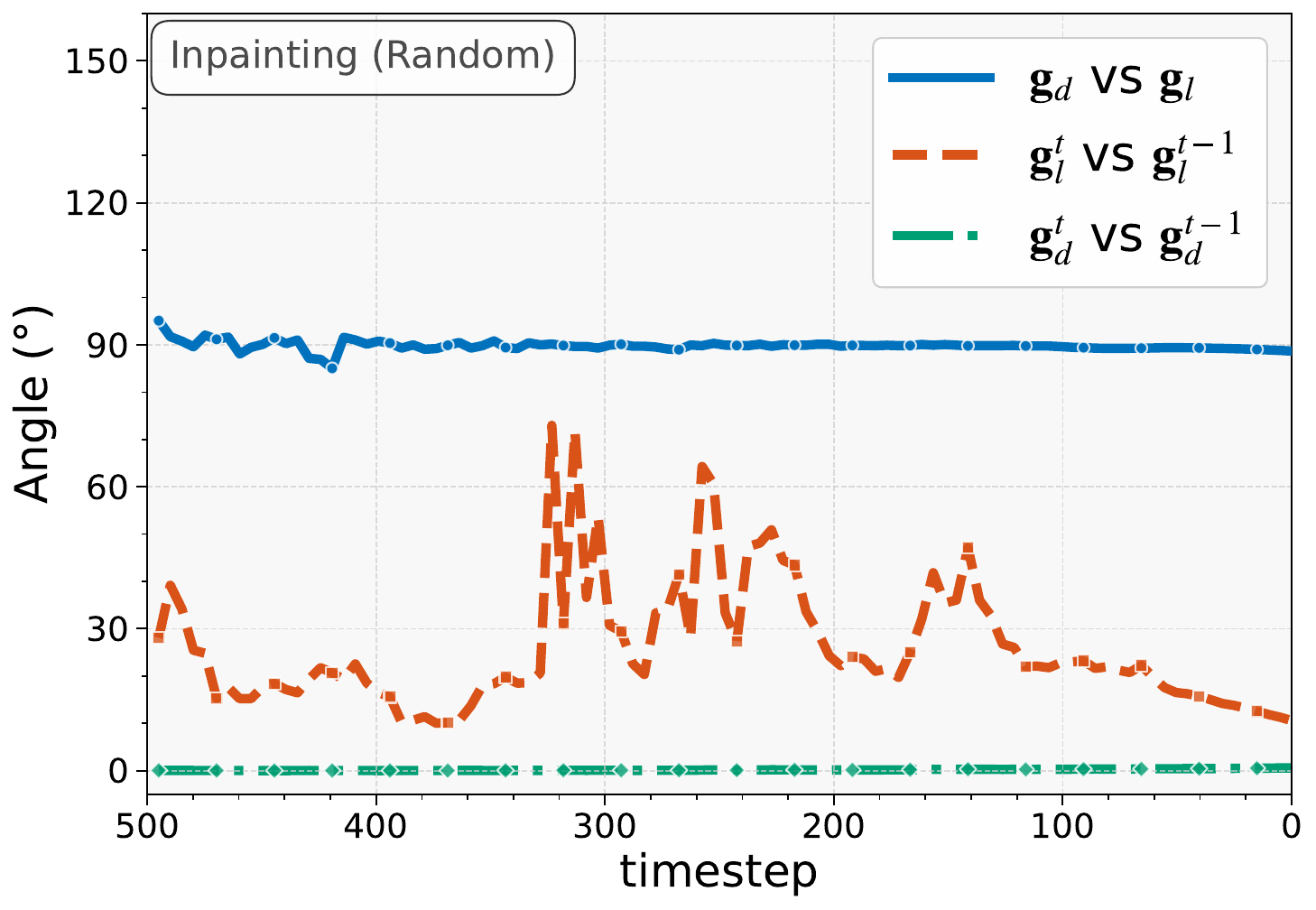}
        \includegraphics[width=0.24\linewidth]{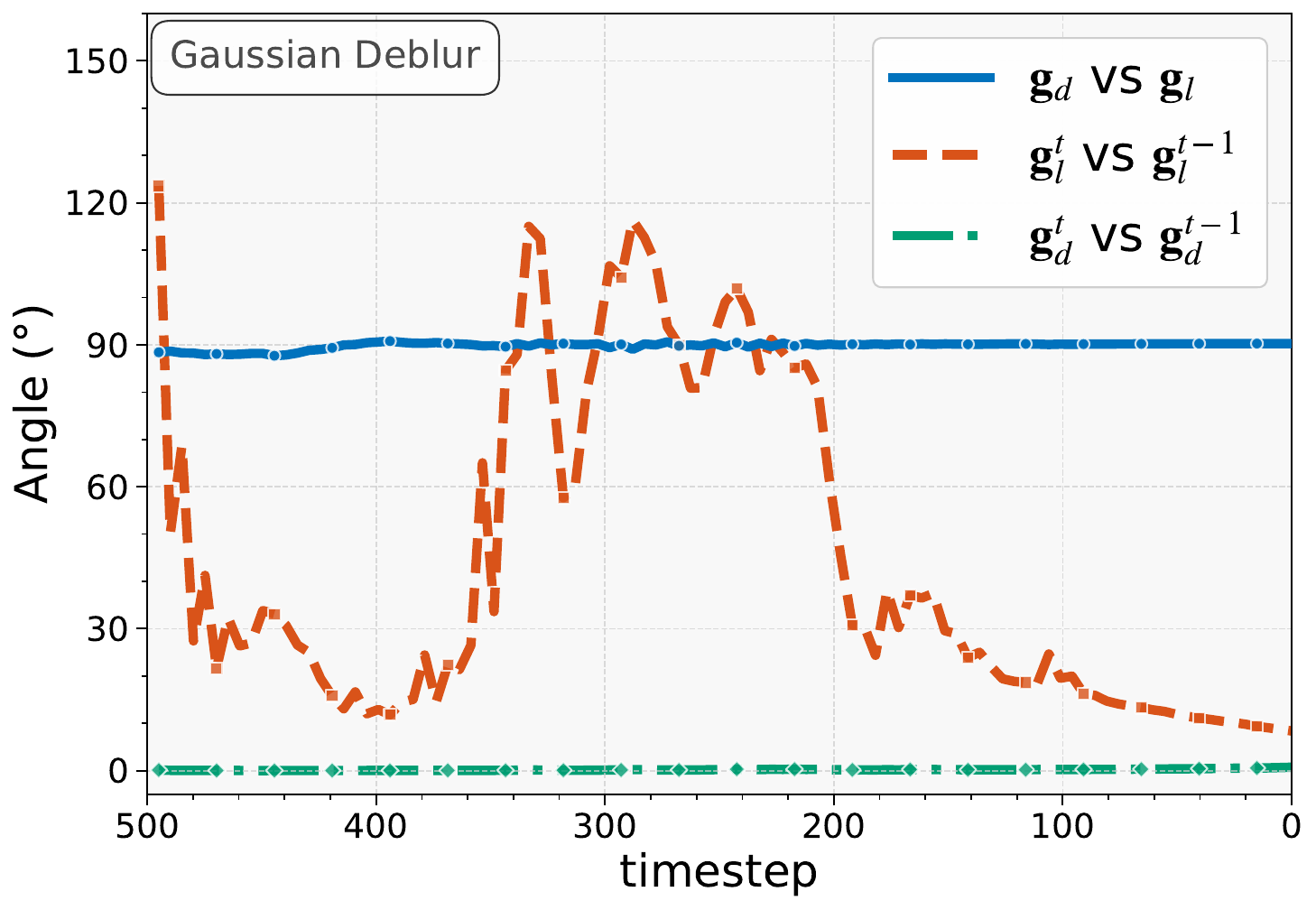}
        \includegraphics[width=0.24\linewidth]{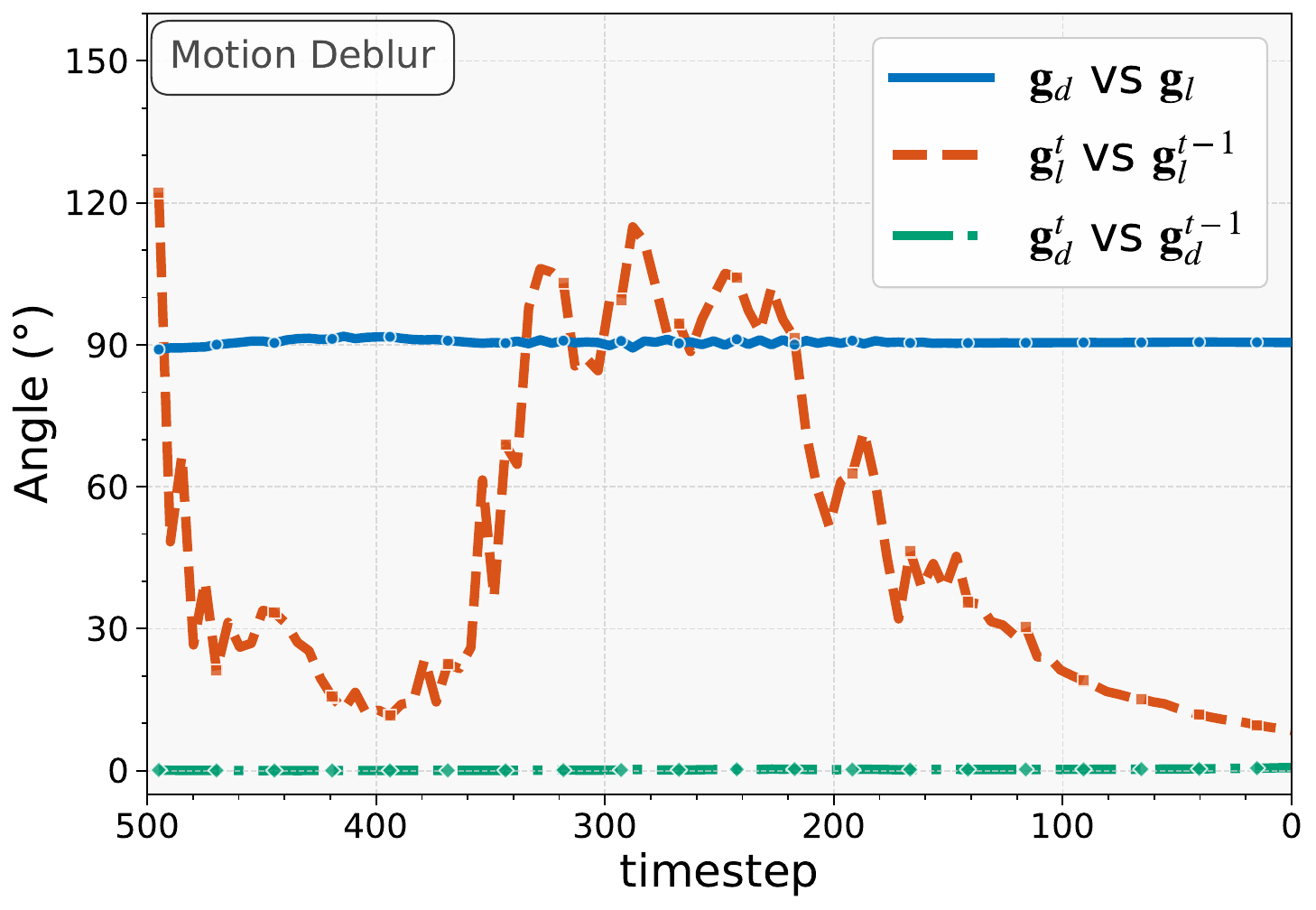}
        \includegraphics[width=0.24\linewidth]{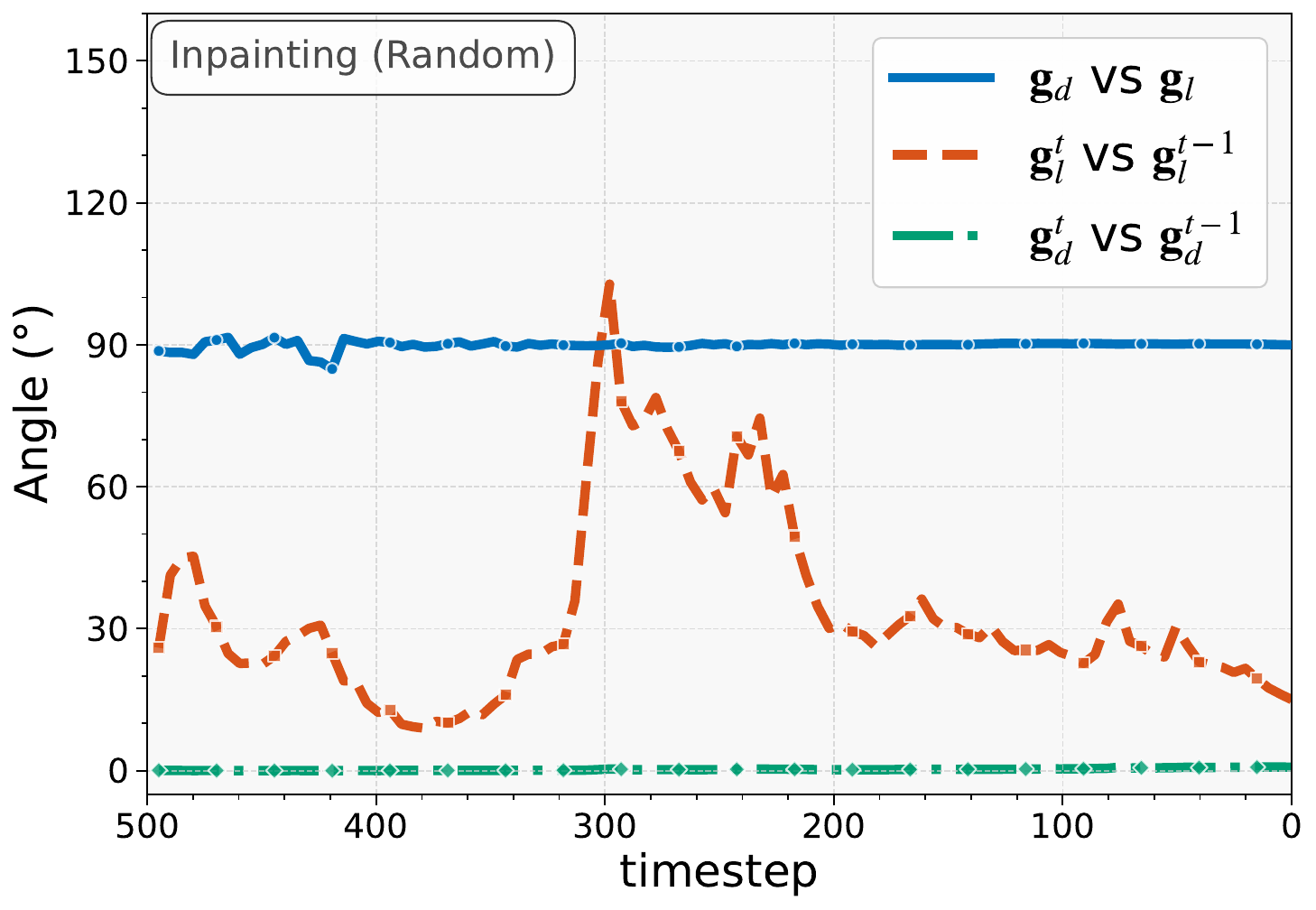}
        \caption{Gradient angles dynamics with ADM smoothing.}
        \vspace{-0.2cm}
    \end{subfigure}
    \caption{Angular relationships of gradients during the reverse process across different tasks, illustrating the gradient angles dynamics (a) without smoothing and (b) with our ADM smoothing.  \textcolor{myblue}{\textbf{Blue lines}}: angle between $\mathbf{g}_l$ and $\mathbf{g}_d$. \textcolor{myorange} {\textbf{Orange lines}}: angle between $\mathbf{g}_l(\mathbf{x}_t)$ and $\mathbf{g}_l(\mathbf{x}_{t-1})$. \textcolor{mygreen}{\textbf{Green lines}}: angle between $\mathbf{g}_d(\mathbf{x}_t)$ and $\mathbf{g}_d(\mathbf{x}_{t-1})$.}
    \label{fig:gradient_show}
    \vspace{-0.2cm}
\end{figure*}

\subsection{Diffusion Models}
Diffusion models~\cite{sohl2015deep, ho2020denoising, song2021scorebased}, also known as score-based models, operate by defining a stochastic process that gradually perturbs data samples into a noise distribution and subsequently learns to reverse this process to generate samples from the desired data distribution.
Consider a $T$-step forward process, where the noised sample $\mathbf{x}_t \in \mathbb{R}^d$ at time step $t$ can be modeled from the previous state $\mathbf{x}_{t-1}$:
\begin{equation}\label{eq:ddpmforward1}
    q(\mathbf{x}_{t}|\mathbf{x}_{t-1})=\mathcal{N}(\mathbf{x}_{t};\sqrt{1-\beta_{t}}\mathbf{x}_{t-1},\beta_{t}\mathbf{I}),    
\end{equation}
where $\mathcal{N}$ denotes the Gaussian distribution, and $\beta_{t}$ is a pre-defined parameter increasing with $t$. Through reparameterization, $\mathbf{x}_t$ can be directly expressed in terms of the original data $\mathbf{x}_0$ via:
\begin{equation}
    q(\mathbf{x}_{t}|\mathbf{x}_{0})=\mathcal{N}(\mathbf{x}_{t};\sqrt{\bar{\alpha}_{t}}\mathbf{x}_{0},(1-\bar{\alpha}_{t})\mathbf{I}),
\end{equation}
where $\bar{\alpha}_{t} = \prod_{i=0}^{t}\alpha_{i}$ and $\alpha_{t} = 1- \beta_{t}$.

The training objective of diffusion models is typically defined as:
\begin{equation}\label{eq:epsilon-matching}
    \mathbf{L}({\mathbf{\theta}}) = \mathbb{E}_{\mathbf{x}_0 \sim q(\mathbf{x}_0), \boldsymbol{\epsilon} \sim \mathcal{N} (\mathbf{0}, \mathbf{I}), t \sim \mathbb{U}(\{1,...,T\})} [\|\boldsymbol{\epsilon} - \boldsymbol{\epsilon}_{\mathbf{\theta}}(\mathbf{x}_t,t)  \|^2],
\end{equation}
where $\mathbf{x}_0$ is sampled from training data and $\mathbf{x}_t \sim q(\mathbf{x}_t|\mathbf{x}_0)$. Once the network $\boldsymbol{\epsilon}_\mathbf{\theta}(\mathbf{x}_t,t)$ is trained, a clean sample can be derived by evaluating the generative reverse process step by step.

It is worth noting that one can adopt DDIM \cite{song2021denoising} to perform the reverse sampling, as it preserves the same marginal distributions $q(\mathbf{x}_t|\mathbf{x}_0)$ as DDPM. The DDIM sampling procedure can be formulated as:
\begin{equation}
\mathbf{x}_{t-1} = \sqrt{\bar{\alpha}_{t-1}} \hat{\mathbf{x}}_{0}(\mathbf{x}_t) + \sqrt{1-\bar{\alpha}_{t-1}-\sigma_t^2}\mathbf{\epsilonb}_\mathbf{\theta}(\mathbf{x}_t, t)+ \sigma_t \epsilonb,
\end{equation}
where $\hat{\mathbf{x}}_{0}(\mathbf{x}_t) =  (\mathbf{x}_t-\sqrt{1-\bar{\alpha}_t}\boldsymbol{\epsilon}_\mathbf{\theta}(\mathbf{x}_{t}, t))/{\sqrt{\bar{\alpha}_t}}$ is the predicted clean image given $\mathbf{x}_t$.
The parameter $\sigma_t$ governs the variance of the reverse sampling process. Specifically, setting  $\sigma_t = 0$ yields a deterministic DDIM sampling, while $\sigma_t = \sqrt{ \frac{1 - \bar{\alpha}_{t-1}}{1 - \bar{\alpha}_t} \cdot \beta_t }$ recovers the original DDPM.




\subsection{Diffusion-Based Solvers for Image Restoration}

The strong generative power of diffusion models makes them excellent candidates for solving image restoration and inverse problem~\citep{xia2023diffir,ye2024learning,zhang2024diffusion}. 
To implement this, the reverse process is guided by the conditional score $\nabla_{\mathbf{x}_t} \log p_t(\mathbf{x}_t|\mathbf{y})$. As shown in \cref{eq:bayes}, this involves the unconditional score $\nabla_{\mathbf{x}_t} \log p_t(\mathbf{x}_t)$  and the likelihood term $\nabla_{\mathbf{x}_t} \log p_t(\mathbf{y}|\mathbf{x}_t)$.

The unconditional score $\nabla_{\mathbf{x}_t} \log p_t(\mathbf{x}_t)$ is directly related to $\frac{-\boldsymbol{\epsilon}_{\thetab}(\mathbf{x}_t, t)}{\sqrt{1-\bar{\alpha}_t}}$~\cite{song2021scorebased}. However, the likelihood term  is intractable because $p_t(\mathbf{y}|\mathbf{x}_t)$ is unknown, since there only exists an explicit connection between $\mathbf{y}$ and $\mathbf{x}_0$, not $\mathbf{x}_t$. To deal with this, a family of methods avoids explicit likelihood computation through interleaving optimization~\cite{song2024resample,zhu2023denoising,li2024decoupled,wu2024principled} or projecting the estimate onto a measurement-consistent subspace~\cite{kawar2022denoising,chung2022improving,wang2023zeroshot,cardoso2024monte}.

The another category of methods direct compute the likelihood under some mild assumptions~\cite{rout2023solving,rout2024beyond}. Of these,
\citet{chung2023diffusion} propose Diffusion Posterior Sampling (DPS) to approximate the likelihood using a Laplacian approximation:
$p_t(\mathbf{y}|\mathbf{x}_t) \simeq p_t(\mathbf{y}|\hat{\mathbf{x}}_{0}(\mathbf{x}_t))$, $\hat{\mathbf{x}}_{0}(\mathbf{x}_t)$ is the predicted clean image of $\mathbf{x}_t$ via Tweedie's formula~\cite{stein1981estimation,efron2011tweedie}.
Then the likelihood gradient can be calculated by:
\begin{equation}
\nabla_{\mathbf{x}_t} \log p_t(\mathbf{y}|\mathbf{x}_t) \simeq  -{\zeta} \nabla_{\mathbf{x}_t} \|\mathbf{y} - \mathcal{A}(\hat{\mathbf{x}}_{0}(\mathbf{x}_t))\|_2^2,
\end{equation}
where $\zeta$ is a tunable step size that controls the strength of the guidance.
The complete DPS algorithm with DDIM deterministic update is shown below:
\begin{equation}
\begin{aligned} \label{eq:dps_ddim}
\mathbf{x}_{t-1}^{\prime} &= \sqrt{\bar{\alpha}_{t-1}} \big(\frac{\mathbf{x}_t-\sqrt{1-\bar{\alpha}_t}\boldsymbol{\epsilon}_\mathbf{\theta}(\mathbf{x}_{t}, t)}{\sqrt{\bar{\alpha}_t}} \big) + \sqrt{1-\bar{\alpha}_{t-1}}\boldsymbol{\epsilon}_\mathbf{\theta}(\mathbf{x}_t, t),\\
\mathbf{x}_{t-1} &= \mathbf{x}_{t-1}^{\prime} -{\zeta} \nabla_{\mathbf{x}_t} \|\mathbf{y} - \mathcal{A}(\hat{\mathbf{x}}_{0}(\mathbf{x}_t))\|_2^2.
\end{aligned}
\end{equation}

By iteratively updating via likelihood matching combined with denoising, one can effectively recover high-quality estimates of $\mathbf{x}_0$ from degraded observations $\mathbf{y}$.  However, as we will show, the direct application of these updates can be unstable.

\section{A Gradient Perspective on Reverse Process} 
\label{sec:Gradient}
\subsection{Reverse Update Decomposition}

Despite the empirical success of Bayesian guidance frameworks, the underlying behavior and interaction between the prior score and likelihood guidance warrant closer examination. In this section, we adopt a gradient-centric perspective~\cite{dinh2023pixelasparam} to analyze the reverse update process.

Consider the DPS likelihood guidance combined with the deterministic DDIM update ($\sigma_t=0$), we rearrange the update from $\mathbf{x}_t$ to $\mathbf{x}_{t-1}$ to explicitly highlight the contributing gradient terms. The insightful decomposition (see \cref{app:derivation} for derivation) is:
\begin{equation}
\begin{aligned} \label{eq:grad_view}
\mathbf{x}_{t-1} =& \underbrace{\frac{1}{\sqrt{\alpha_t}} \mathbf{x}_t}_{\text{fixed scaling}} - \underbrace{\left(\frac{\sqrt{1-\bar{\alpha}_t}}{\sqrt{{\alpha}_t}}-\sqrt{1-\bar{\alpha}_{t-1}} \right) \boldsymbol{\epsilon}_{\boldsymbol\theta}(\mathbf{x}_t, t)}_{\text{denoising gradient}~(\mathbf{g}_d)} \\ 
&- {\zeta}\underbrace{ \nabla_{\mathbf{x}_t} \|\mathbf{y} - \mathcal{A}(\hat{\mathbf{x}}_{0}(\mathbf{x}_t))\|_2^2}_{\text{likelihood gradient}~(\mathbf{g}_l)}, 
\end{aligned}
\end{equation}
Here, we explicitly define:
\begin{itemize}
    \item $\mathbf{g}_d(\mathbf{x}_t)$: The \textbf{denoising gradient}, derived from the learned prior $\boldsymbol{\epsilon}_{\thetab}$, guiding the sample towards the natural image manifold.
    \item $\mathbf{g}_l(\mathbf{x}_t)$: The \textbf{likelihood gradient}, enforcing consistency with the measurement $\mathbf{y}$ via the reconstruction error.
\end{itemize}
The fixed scaling term primarily adjusts the magnitude of $\mathbf{x}_t$. The core dynamics influencing the image content arise from the interplay between $\mathbf{g}_d$ and $\mathbf{g}_l$. While previous works~\citep{chung2023diffusion,song2023pseudoinverseguided,peng2024improving,wu2024diffusion} implicitly combine these effects, analyzing their individual dynamics reveals potential sources of performance degradation.


\subsection{Gradient Conflicts and Fluctuations}

We empirically examine the angular relationships between $\mathbf{g}_d$ and $\mathbf{g}_l$, as well as the temporal consistency of each gradient. The results are shown in \cref{fig:gradient_show} (see \cref{app:more} for more results). Our analysis reveals two critical issues:

(1) \textbf{Gradient Conflict:}  The angle between $\mathbf{g}_d(\mathbf{x}_t)$ and $\mathbf{g}_l(\mathbf{x}_t)$ (\textcolor{myblue}{\textbf{blue lines}} in \cref{fig:gradient_show}) frequently diverges from 90 degrees in the early stage, implying that the prior direction conflicts with the data constraint direction. While these conflicts diminish as $t$ decreases~\footnote{two randomly chosen vectors tend to be orthogonal in high dimensions~\citep{ledoux2001concentration,talagrand1995concentration}}, their presence in the early stages can disrupt the image structure formation~\cite{singh2022conditioning,everaert2024exploiting}, a crucial phase for overall reconstruction quality.

(2) \textbf{Likelihood Gradient Fluctuation:}  The temporal consistency between $\mathbf{g}_l(\mathbf{x}_t)$ and $\mathbf{g}_l(\mathbf{x}_{t+1})$ (\textcolor{myorange}{\textbf{orange lines}} in \cref{fig:gradient_show}) often exhibits pronounced instability in the intermediate stages.
We speculate that $\mathbf{g}_l$ is derived from $\|\mathbf{y} - \mathcal{A}(\hat{\mathbf{x}}_{0}(\mathbf{x}_t))\|_2^2$, which can be highly sensitive to small changes in the estimate $\hat{\mathbf{x}}_{0}(\mathbf{x}_t)$.
These erratic fluctuations hinder smooth convergence towards data consistency and can potentially amplify errors over time.

In contrast, the denoising gradient $\mathbf{g}_d$ derived from the well-trained $\boldsymbol{\epsilon}_{\thetab}$ demonstrates remarkable temporal stability (visualized by the \textcolor{mygreen}{\textbf{green lines}} in \cref{fig:gradient_show}), providing a consistent guidance from the learned prior.


These observed conflicts and fluctuations in the gradient dynamics are not merely theoretical concerns; they manifest as tangible artifacts in reconstructed images, slower convergence rates, and quantifiable reductions in evaluation metrics (see \cref{sec:experiments}). This motivates the development of a method that explicitly manages these gradient dynamics for more stable and effective restoration.

\section{Proposed Method}
 
Motivated by our analysis in \cref{sec:Gradient}, we propose \textbf{Stabilized Progressive Gradient Diffusion (SPGD)}, a novel gradient management technique to enhance the stability and effectiveness of the reverse diffusion process. SPGD introduces two synergistic components: a progressive likelihood warm-up strategy and a momentum-based likelihood gradient smoothing technique.



\subsection{Progressive Warm-Up Phase}
To mitigate the conflicts between the denoising gradient $\mathbf{g}_d$ and the likelihood gradient $\mathbf{g}_l$, SPGD introduces a \textit{progressive warm-up} phase for likelihood update before each primary denoising step.

Specifically, within a single conceptual reverse step from time $t$ to $t-1$, starting from the current estimate $\mathbf{x}_t^{(0)} = \mathbf{x}_t$, we perform $N$ iterative refinement steps focusing solely on the likelihood gradient using a reduced step size:
\begin{equation}
\mathbf{x}_t^{(j+1)} = \mathbf{x}_t^{(j)} - \frac{\zeta}{N} \widetilde{\mathbf{g}}_l(\mathbf{x}_t^{(j)}), \quad \text{for } j = 0, 1, \ldots, N-1,
\label{eq:warmup_with_momentum}
\end{equation}
where $\widetilde{\mathbf{g}}_l(\mathbf{x}_t^{(j)})$ is the ADM-smoothed likelihood gradient (detailed in \cref{eq:momentum_update}), and $\zeta/N$ serves as a small learning rate for each warm-up step.  This multi-step process allows the estimate $\mathbf{x}_t$ to gradually adapt towards satisfying the measurement constraint $\mathbf{y}$ \textit{before} the potentially conflicting denoising gradient $\mathbf{g}_d$ is applied. Moreover, the small step size helps prevent drastic deviations from the learned data manifold~\citep{alkhouri2024sitcom,huang2022multi}, fostering a more controlled adjustment.

After these $N$ warm-up iterations yield $\mathbf{x}_t^{(N)}$, the standard denoising update is applied to obtain the initial estimate for the next timestep, using $\mathbf{x}_t^{(N)}$ as input:
\begin{equation}
\mathbf{x}_{t-1}^{(0)} = 
\frac{1}{\sqrt{\alpha_t}} \mathbf{x}_t^{(N)} -\left(\frac{\sqrt{1-\bar{\alpha}_t}}{\sqrt{{\alpha}_t}}-\sqrt{1-\bar{\alpha}_{t-1}} \right) \boldsymbol{\epsilon}_{\boldsymbol\theta}(\mathbf{x}_t^{(N)}, t),
\label{eq:denoising_update}
\end{equation}
which is a standard denoising update with diffusion model, i.e.  DDIM in this paper. 
Given the observed stability of the denoising gradient $\mathbf{g}_d$, this multi-step warm-up mitigates the instability risk arising from potential conflicts, providing a more refined input $\mathbf{x}_t^{(N)}$ for the reliable denoising operation.


\begin{algorithm}[t]
\caption{Stabilized Progressive Gradient Diffusion}
\begin{algorithmic}[1] \label{algo:algo}
\REQUIRE Measurements $\y$, degrade operator $\mathcal{A}(\cdot)$, network $\boldsymbol{\epsilon}_{\theta}(\cdot)$, parameters $\zeta$
\STATE Initialize $\x_T \sim \mathcal{N}(\mathbf{0}, \I)$
\FOR{$t = T,  \dots, 1$}
    \FOR{$j = 0, \dots, N-1$}
        \STATE $\mathbf{g}_l(\mathbf{x}_t^{(j)}) = \nabla_{\mathbf{x}_t^{(j)}} \|\mathbf{y} - \mathcal{A}(\hat{\mathbf{x}}_{0}(\mathbf{x}_t^{(j)}))\|_2^2$
        \IF{j=0} \STATE $\widetilde{\mathbf{g}}_l(\mathbf{x}_t^{(0)})=\mathbf{g}_l(\mathbf{x}_t^{(0)})$
        \ELSE
            \STATE $\alpha_j =  (\text{sim}\big( \widetilde{\mathbf{g}}_l(\mathbf{x}_t^{(j-1)}),\, \mathbf{g}_l(\mathbf{x}_t^{(j)}) \big)+1)/2,$
            \STATE $\widetilde{\mathbf{g}}_l(\mathbf{x}_t^{(j)}) = \alpha_j \beta \, \widetilde{\mathbf{g}}_l(\mathbf{x}_t^{(j-1)}) + (1 - \alpha_j \beta) \,\mathbf{g}_l(\mathbf{x}_t^{(j)})$  \textcolor[rgb]{0.40,0.40,0.40}{\textit{// adaptive momentum-based gradient smoothing}}
        \ENDIF
       \STATE  $\mathbf{x}_t^{(j+1)} \gets \mathbf{x}_t^{(j)} - \frac{\zeta}{N} \widetilde{\mathbf{g}}_l(\mathbf{x}_t^{(j)})$ \textcolor[rgb]{0.40,0.40,0.40}{\textit{// warm-up likelihood update}}
    \ENDFOR
    \STATE $\mathbf{g}_d(\mathbf{x}_t^{(N)}) = \left( {\sqrt{1-\bar{\alpha}_t}}/{\sqrt{{\alpha}_t}}-\sqrt{1-\bar{\alpha}_{t-1}}\right) \boldsymbol{\epsilon}_{\theta}(\mathbf{x}_t^{(N)}, t)$
    \STATE  $\mathbf{x}_{t-1}^{(0)} = \frac{1}{\sqrt{\alpha_t}} \mathbf{x}_t^{(N)} - \mathbf{g}_d(\mathbf{x}_t^{(N)}) $
    \textcolor[rgb]{0.40,0.40,0.40}{\textit{// DDIM denoising update}}
\ENDFOR
\RETURN $\x_0$
\end{algorithmic}
\end{algorithm}

\subsection{Momentum-Based Likelihood Gradient Smoothing}

To counteract the temporal fluctuation observed in the likelihood gradient $\mathbf{g}_l$, SPGD incorporates an \textbf{Adaptive Directional Momentum (ADM)} smoothing technique. This produces the stabilized likelihood gradient $\widetilde{\mathbf{g}}_l$ by incorporating historical information, which is then used in the warm-up phase in \cref{eq:warmup_with_momentum}.

Standard momentum methods~\citep{polyak1964some, sutskever2013importance, he2024faststablediffusioninverse, wang2024boosting} use an exponentially weighted moving average to smooth gradients. However, constant momentum can slow adaptation when the optimal direction genuinely changes. Our proposed ADM enhances this by dynamically adjusting the momentum strength based on the directional consistency between the current gradient and the accumulated momentum.

Specifically, let $\mathbf{g}_l(\mathbf{x}_t^{(j)}) = \nabla_{\mathbf{x}_t^{(j)}} \|\mathbf{y} - \mathcal{A}(\hat{\mathbf{x}}_{0}(\mathbf{x}_t^{(j)}))\|_2^2$ be the raw likelihood gradient at warm-up step $j$. The ADM-smoothed gradient $\widetilde{\mathbf{g}}_l(\mathbf{x}_t^{(j)})$ is computed recursively:
\begin{equation}
\widetilde{\mathbf{g}}_l(\mathbf{x}_t^{(j)}) = \alpha_j \beta \, \widetilde{\mathbf{g}}_l(\mathbf{x}_t^{(j-1)}) + (1 - \alpha_j \beta) \,\mathbf{g}_l(\mathbf{x}_t^{(j)}),
\label{eq:momentum_update}
\end{equation}
with initialization $\widetilde{\mathbf{g}}_l(\mathbf{x}_t^{(0)}) = \mathbf{g}_l(\mathbf{x}_t^{(0)})$. Here, $\beta \in [0, 1)$ is the pre-defined base momentum coefficient. The key component is the adaptive weight $\alpha_j \in [0, 1]$, defined as:
\begin{equation}
    \alpha_j =  \frac{\text{sim}\big( \widetilde{\mathbf{g}}_l(\mathbf{x}_t^{(j-1)}),\, \mathbf{g}_l(\mathbf{x}_t^{(j)}) \big)+1}{2},
    \label{eq:alpha_def}
\end{equation}
where $\text{sim}(\mathbf{u}, \mathbf{v}) = \frac{\mathbf{u} \cdot \mathbf{v}}{\|\mathbf{u}\| \|\mathbf{v}\|} \in [-1, 1]$, denotes the cosine similarity.

\textbf{Theoretical intuition:}
When the current gradient $\mathbf{g}_l(\mathbf{x}_t^{(j)})$ aligns well ($\text{sim}(\cdot) \approx 1$) with the previous momentum $\widetilde{\mathbf{g}}_l(\mathbf{x}_t^{(j-1)})$, then $\alpha_j \approx 1$, and ADM behaves like standard momentum, smoothing the update by trusting the historical trend.
Conversely, when the directions differ significantly ($\text{sim}(\cdot) \ll 1$), $\alpha_j$ decreases towards 0, effectively down-weighting the potentially outdated historical momentum and prioritizing the current gradient information.
By adaptively modulating the momentum, ADM effectively dampens erratic gradient fluctuations while remaining responsive to significant directional shifts, yielding a more stable and reliable gradient $\widetilde{\mathbf{g}}_l$ for guiding the optimization towards data consistency.

The complete SPGD procedure is detailed in \cref{algo:algo}. In our experiments, we extend the warm-up strategy throughout the entire reverse process, and set the total number of outer diffusion steps $T=100$ and warm-up steps $N=5$, aligning with common Number of Function Evaluations (NFEs) in existing literature.

\subsection{Theoretical Analysis}
\label{subsec:theoretical_analysis}

Here, we provide theoretical support for our methodology. Our analysis highlights how the SPGD warm-up optimizes for data fidelity, distinguishing it from prior posterior sampling techniques.


\begin{figure}[t]
    \centering
    \includegraphics[width=1\linewidth]{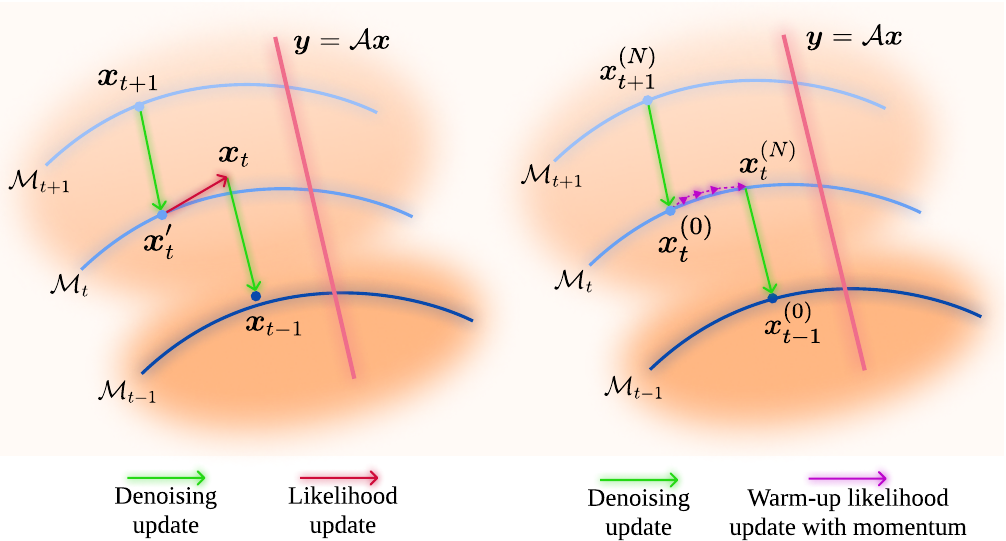} 
        \caption{High-level illustration of our proposed SPGD.  (a) the standard reverse process (Eq.~\ref{eq:dps_ddim}), and (b) our proposed SPGD (\cref{algo:algo}), showing the warm-up phase with smoothed gradient to enhance restoration stability.}
    \label{fig:manifold}
    \vspace{-0.2cm}
\end{figure}

\begin{proposition}\label{prop:warmup_optimality}
Let $L_t(\x_t) = \frac{1}{2} \norm{\y - \Acal(\hat{\x}_{0}(\x_t))}^2$ be the likelihood objective at diffusion timestep $t$. Assume the likelihood gradient $\g_l(\x_t) = \nabla_{\x_t} L_t(\x_t)$ is $L$-Lipschitz continuous with respect to $\x_t$ for some $L > 0$. Consider the SPGD warm-up phase without momentum ($\beta=0$) using $N \ge 1$ steps with step size $\eta = \zeta / N$. If the step size satisfies $\eta < 1/L$, then the state $\x_t^{(N)}$ obtained after $N$ warm-up steps guarantees a decrease in the likelihood objective:
\begin{equation}
L_t(\x_t^{(N)}) \le L_t(\x_t^{(0)}) - \eta \sum_{j=0}^{N-1} \left(1 - \frac{L\eta}{2}\right) \norm{\g_l(\x_t^{(j)})}^2 ,
\end{equation}
where $\x_t^{(0)}$ is the initial state at time $t$ before the warm-up.
\end{proposition}

\begin{proof}
    Please refer to \cref{sec:proof}.
\end{proof}

\Cref{prop:warmup_optimality} formally supports that the SPGD warm-up iterations progressively decrease the likelihood objective $L_t$, thereby enhancing the data consistency of the state $\x_t$ before denoising. In other words, $\x_t^{(N)}$ yields a better estimate $\hat{\x}_{0}$ with smaller likelihood objective value compared to using the initial state $\x_t^{(0)}$.

As conceptually illustrated in \cref{fig:manifold},
SPGD contrasts with standard posterior sampling methods like DPS~\citep{chung2023diffusion}. In DPS, the likelihood gradient $\g_l$ is computed based on the initial state $\x_t$ (equivalent to our $\x_t^{(0)}$), but the correction is applied only \textbf{after} the denoising step produces $\x_{t-1}'$. This timing mismatch, where a gradient relevant to $\x_t$ corrects the transformed state $\x_{t-1}'$, can potentially lead to instability and cause trajectories to deviate from the data manifold~\cite{wu2023pure}.
On the contrary, SPGD performs iterative refinement directly on $\x_t^{(j)}$ before invoking the denoiser. By providing the denoising step with an input $\x_t^{(N)}$ that is already better aligned with the measurements $\y$, SPGD potentially facilitates a more stable update process and a more accurate subsequent generation.


\begin{figure*}[t]
    \centering
    \begin{tikzpicture}
        \node[anchor=south west,inner sep=0] (image) at (0,0) {\includegraphics[width=\linewidth]{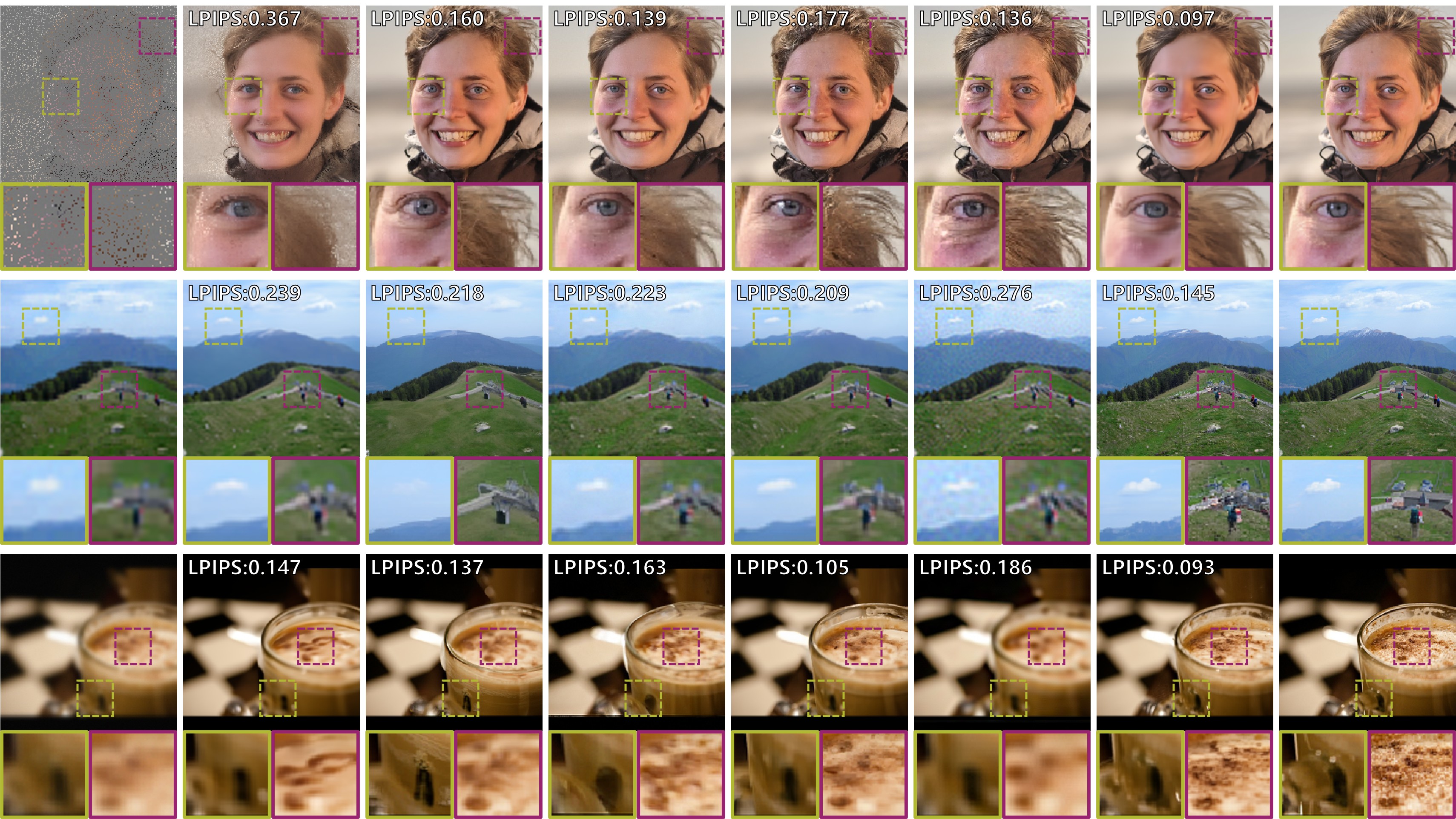}};
        \begin{scope}[x={(image.south east)},y={(image.north west)}]
            \node at (0.06,1.02) {\large{Input}};
            \node at (0.186,1.02) {\large{DDRM}};
            \node at (0.31,1.02) {\large{DPS}};
            \node at (0.435,1.02) {\large{DiffPIR}};
            \node at (0.56,1.02) {\large{DPPS}};
            \node at (0.687,1.02) {\large{RED-Diff}};
            \node at (0.81,1.02) {\large{Ours}};
            \node at (0.937,1.02) {\large{Reference}};
        \end{scope}
    \end{tikzpicture}
    \caption{Image restoration results with $\sigma_y=0.01$. Row 1: random inpainting, Row 2: SR ($\times$4), Row 3: Gaussian deblurring.}
    \label{fig:show_mains}
    \vspace{-0.3cm}
\end{figure*}
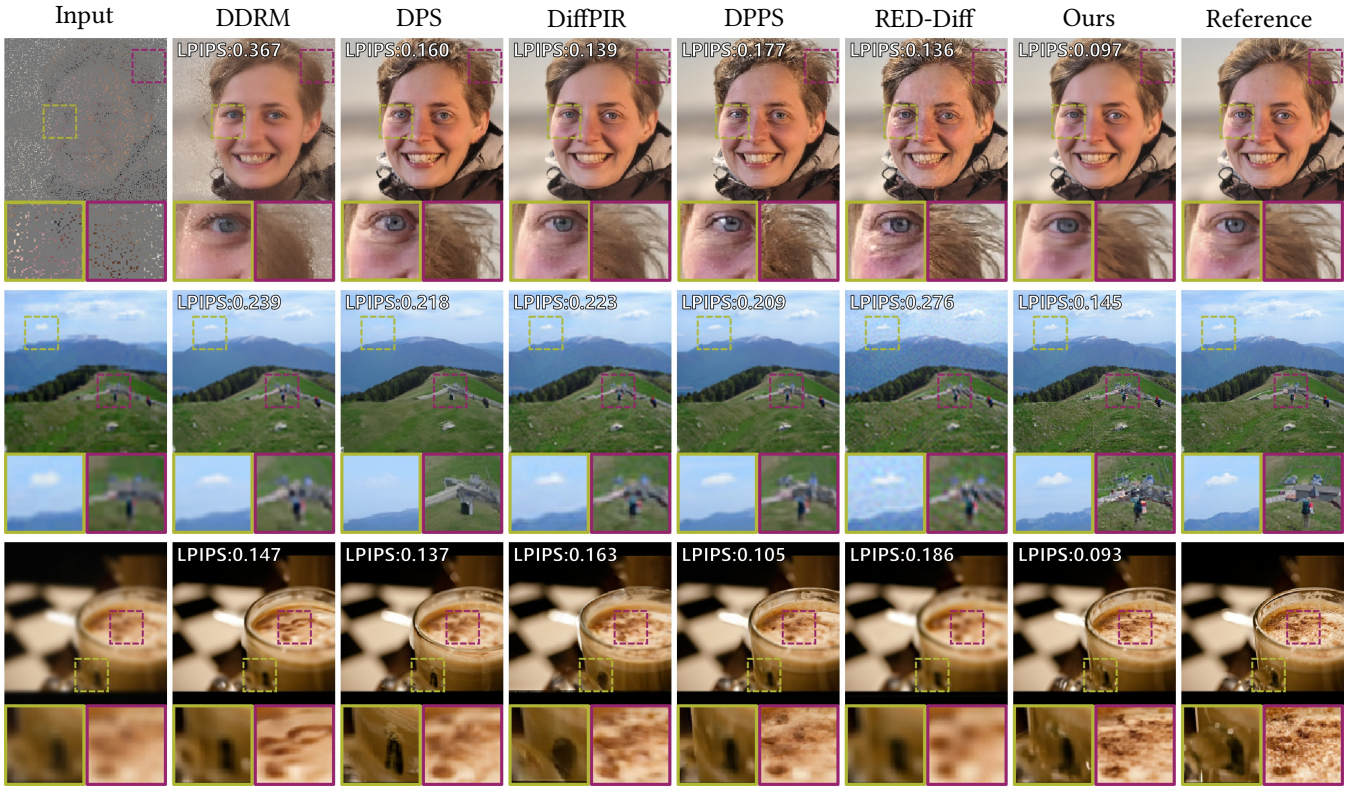

\begin{table*}[t]
\centering 
\caption{Quantitative evaluation of image restoration tasks on FFHQ 256$\times$256 and ImageNet 256$\times$256 with $\sigma_y=0.01$.Best results are marked in \textbf{bold}, and second-best results are \underline{underlined}.
}
\vspace{-0.2cm}
\setlength{\tabcolsep}{4 pt}
\resizebox{1.0\textwidth}{!}{
\begin{tabular}{lcccccccccccc}
\toprule
{\textbf{}} & \multicolumn{3}{c}{\textbf{Inpaint (random)}} & \multicolumn{3}{c}{\textbf{Deblur (Gaussian)}} &
\multicolumn{3}{c}{\textbf{Deblur (motion)}} &\multicolumn{3}{c}{\textbf{SR ($\times$ 4)}} \\
\cmidrule(lr){2-4}
\cmidrule(lr){5-7}
\cmidrule(lr){8-10}
\cmidrule(lr){11-13}
{\textbf{Method}} & PSNR $\uparrow$ & SSIM $\uparrow$ & LPIPS $\downarrow$ & PSNR $\uparrow$ & SSIM $\uparrow$ & LPIPS $\downarrow$ & PSNR $\uparrow$ & SSIM $\uparrow$ & LPIPS $\downarrow$ & PSNR $\uparrow$ & SSIM $\uparrow$ & LPIPS $\downarrow$ \\
\toprule
\multicolumn{13}{c}{\textbf{FFHQ}} \\
\toprule
PnP-ADMM~\cite{chan2016plug} & 27.99 & 0.729 & 0.306 & 26.07 & 0.758 & 0.260 & 25.86 &\underline{ 0.772} & 0.278 & 27.75 & \underline{0.835} & 0.246 \\
Score-SDE~\cite{song2021scorebased}~ & 22.04 & 0.613 & 0.370 & 24.32 & 0.745 & 0.336 & 15.87 & 0.495 & 0.541 & 18.53 & 0.272 & 0.598 \\
MCG~\cite{chung2022improving}& 27.74 & 0.825 & \underline{0.147} & 24.28 & 0.737 & 0.292 & 19.78 & 0.540 & 0.476 & 18.10 & 0.261 & 0.606 \\
DDRM~\cite{kawar2022denoising} & 22.02 & 0.652 & 0.362 & \underline{27.11} & 0.779 & 0.244 &- &- &- & \textbf{29.49} & \textbf{0.853} & 0.190 \\
DPS~\cite{chung2023diffusion} & 26.11 & 0.802 & 0.180 & 26.51 & \underline{0.782} & 0.181 & 25.58 & 0.752 & 0.212 & 27.06 & 0.803 & 0.187 \\
DiffPIR~\cite{zhu2023denoising} & \underline{28.64} & \underline{0.842} & 0.182 & 25.92 & 0.722 & 0.234 & \underline{27.51} & 0.635 & 0.264 & 28.21 & 0.801 & 0.211 \\
DPPS~\cite{wu2024diffusion} & 25.47 & 0.769 & 0.191 & 26.92 & \textbf{0.793} & \underline{0.173} & 26.25 & 0.771 & \underline{0.194} & 27.52 & 0.820 & \underline{0.153} \\
RED-Diff~\cite{mardani2024red-diff} & 27.17 & 0.799 & 0.159 & 24.69 & 0.672 & 0.288 & 26.24 & 0.706 & 0.255 & 29.06 & 0.800 & 0.243  \\
\midrule
\rowcolor{lightblue}
Ours & \textbf{30.87} & \textbf{0.889} & \textbf{0.120} & \textbf{27.83} & 0.775& \textbf{0.172}& \textbf{29.41}&\textbf{0.834} & \textbf{0.158} & \underline{29.35} & 0.831& \textbf{0.137} \\
\bottomrule
\multicolumn{13}{c}{\textbf{ImageNet}} \\
\toprule
PnP-ADMM~\citep{chan2016plug} & 25.14 & 0.643 & 0.355 & 22.01 & 0.541 & 0.418 & 21.88 & 0.563 & 0.408 & 23.95 & 0.718 & 0.323 \\
Score-SDE\cite{song2021scorebased}~ & 16.25 & 0.234 & 0.629 & 21.31 & 0.593 & 0.446 & 13.56 & 0.371 & 0.584 & 17.69 & 0.285 &  0.582\\
MCG~\cite{chung2022improving}& 23.21 & 0.647 & 0.294 & 12.31  & 0.340 & 0.578 & 18.32 & 0.425 & 0.502 & 17.08 & 0.264 & 0.591 \\
DDRM~\cite{kawar2022denoising} & 19.34 & 0.449 & 0.525  & \underline{23.67} & \underline{0.632} & 0.378 & -& -&- & \underline{25.49} & \textbf{0.736} & 0.298 \\
DPS~\citep{chung2023diffusion} & 25.09 & 0.773 & 0.217 & 19.49 & 0.485 & 0.399 & 18.68 & 0.454 & 0.432 & 23.47 & 0.665 & 0.310 \\
DiffPIR~\cite{zhu2023denoising} & \textbf{26.59} & \underline{0.793} & 0.204 & 22.51 & 0.557 & 0.363 & \underline{25.37} & 0.616 & 0.238 & 25.44 & \underline{0.723} & 0.272 \\
DPPS~\cite{wu2024diffusion} & 24.55 & 0.755 & 0.193 & 21.75 & 0.587 & 0.352 & 20.09 & 0.530 & 0.392 & 24.36 & 0.718 & \underline{0.253} \\
RED-Diff~\cite{mardani2024red-diff} & 24.86 & 0.750 & \underline{0.171} & 22.80 & 0.577 & \underline{0.339} & 24.84 & \underline{0.687} & \underline{0.284} & 25.09 & 0.686 & 0.325   \\
\midrule
\rowcolor{lightblue}
Ours &\underline{26.28} &\textbf{0.798} & \textbf{0.165} & \textbf{24.80} & \textbf{0.651} & \textbf{0.229} & \textbf{28.11} & \textbf{0.728} & \textbf{0.154} & \textbf{25.71} &0.705 & \textbf{0.203} \\

\bottomrule
\end{tabular}
}
\label{tab:results_all}
\vspace{-0.2cm}
\end{table*}

\section{Experiments} 
\label{sec:experiments}

\subsection{Experimental Setup}

\textbf{Datasets \& Metrics.} We evaluate our proposed SPGD method and baseline approaches on 1,000 images from the FFHQ 256$\times$256 validation set~\cite{karras2019style} and 500 images from the ImageNet 256$\times$256 validation set~\cite{russakovsky2015imagenet}. The pre-trained diffusion models for FFHQ and ImageNet are obtained from~\citep{chung2023diffusion} and~\citep{dhariwal2021diffusion}, respectively, following established conventions.

Performance is assessed with quantitative metrics: Peak Signal-to-Noise Ratio (PSNR) and Structural Similarity Index Measure (SSIM)~\citep{wang2004image} for distortion evaluation, and Learned Perceptual Image Patch Similarity (LPIPS)~\cite{Zhang2018TheUE} for perceptual quality. 

\textbf{Baselines.} We compare our method against several state-of-the-art diffusion-based image restoration techniques, including PnP-ADMM~\cite{chan2016plug}, Score-SDE~\cite{song2022solving}, MCG~\cite{chung2022improving}, DDRM~\cite{kawar2022denoising}, DPS~\cite{chung2023diffusion}, DiffPIR~\cite{zhu2023denoising}, DPPS~\cite{wu2024diffusion}, and RED-Diff~\cite{mardani2024red-diff}. All experiments are conducted using the same pre-trained models and a fixed random seed to ensure fair comparisons.

\textbf{Tasks.} We consider four standard linear image restoration tasks:
(1) \textbf{Inpainting:} Randomly mask 80\% of the pixels across all RGB channels.
(2) \textbf{Gaussian Blur:} Apply a Gaussian blur kernel of size $61 \times 61$ with a standard deviation of $\sigma = 3.0$.
(3) \textbf{Motion Blur:} Use a motion blur kernel of size $61 \times 61$ and intensity 0.5, following the procedure of Chung et al.~\cite{chung2022improving}.
(4) \textbf{Super-Resolution (SR) $\times$4:} Perform $4\times$ bicubic downsampling.

In all scenarios, Gaussian noise with a variance of $\sigma_y = 0.01$ is added to the degraded measurements. 
For efficiency, all our ablation studies were conducted using 100 FFHQ test images.

\begin{table}[t]  
\caption{Ablation study on the effect of components.}  
    \setlength{\tabcolsep}{2.5pt}
        \centering
        \resizebox{\linewidth}{!}{%
        \begin{tabular}{cccccccc}
            \toprule        
            \multicolumn{2}{c}{Strategies} & \multicolumn{3}{c}{Inpainting} & \multicolumn{3}{c}{SR ($\times$ 4)}   \\
            \cmidrule(lr){1-2}
            \cmidrule(lr){3-5}
            \cmidrule(lr){6-8}
            \small{Warm-up}  & \small{ADM} & {PSNR$\uparrow$} &  {SSIM$\uparrow$} &{LPIPS$\downarrow$}  &{PSNR$\uparrow$}& {SSIM$\uparrow$}  & {LPIPS$\downarrow$}  \\
            \midrule
            \xmark  & \xmark  &30.35 & 0.861 & 0.155 & 28.41 & 0.785 & 0.194 \\ 
           \xmark & \cmark  & 30.34 & 0.859 & 0.157 & 28.36 & 0.783 & 0.197 \\   

		 \cmark  & \xmark &\underline{31.57} & \underline{0.902} & \underline{0.093} & \underline{29.79} & \underline{0.838} & \underline{0.131} \\      
		\rowcolor{lightblue}
		\cmark & \cmark  & \textbf{31.63} & \textbf{0.904} & \textbf{0.075} &  \textbf {30.08 }& \textbf {0.847} & \textbf {0.124} \\
        
            \bottomrule
        \end{tabular}} 
        \vspace{-0.2cm} \label{tab:Components}
\end{table}

\subsection{Quantitative Results}

As summarized in \cref{tab:results_all}, SPGD demonstrates consistently superior performance across the majority of tasks and metrics.

On the FFHQ dataset, our method achieves the best results for all metrics across all four degradation types. This indicates that our approach effectively minimizes distortion (high PSNR), preserves structural details (high SSIM), and generates results with superior perceptual similarity to the ground truth (low LPIPS) compared to the baseline methods on this dataset. The improvements are particularly significant in inpainting and motion deblurring, highlighting the robustness of our method in handling complex degradations.

On the more challenging ImageNet dataset, our method continues to outperform the baselines in most scenarios. It achieves the best SSIM and LPIPS scores for inpainting, Gaussian deblurring, and motion deblurring. For SR$\times$4, our method obtains the best PSNR and LPIPS, and competitive SSIM scores, demonstrating its effectiveness on diverse image content. While DDRM achieves slightly higher SSIM for SR~($\times$4), the reliance on the singular value decomposition (SVD) of operator $\mathcal{A}$ renders their methodology inapplicable in many scenarios. Our method shows a marked advantage in LPIPS and frequently outperforms others in distortion metrics, demonstrating that our gradient management directly leads to the enhanced stability with superior restoration quality.


\subsection{Qualitative Results}

Visual comparisons, presented in \cref{fig:show_mains}, corroborate the quantitative findings.
Across different restoration tasks, SPGD consistently produces visually superior results compared to baseline methods.
For instance, in random inpainting on FFHQ, SPGD generates significantly cleaner and more coherent reconstructions, successfully restoring intricate details like eye reflections and hair textures where competing methods often exhibit artifacts or blurriness.
In super-resolution, SPGD renders sharper textures and finer structural details compared to the outputs of methods like RED-Diff, which can appear overly smoothed.
Similarly, for Gaussian deblurring tasks, SPGD yields noticeably sharper edges and recovers finer details compared to DiffPIR or DPPS. These visual improvements are direct consequences of the stabilized reverse trajectory enabled by the proposed techniques.

Furthermore, we analyze the evolution of the predicted image ($\hat{\mathbf{x}}_{0}$) of the challenging task motion deblurring. As shown in \cref{fig:convenge}, the method DPS exhibits erratic behavior in intermediate stages (e.g., NFE = 150). Excessive or unstable likelihood corrections, driven by the fluctuating likelihood gradient $\mathbf{g}_l$, can introduce strong artifacts into $\hat{\mathbf{x}}_{0}$, potentially disrupting the global structure of the final result. In contrast, SPGD demonstrates a much smoother and more stable progression. The progressive warm-up and ADM smoothing enable gradual, controlled refinement guided by the likelihood, preserving key image features and more reliably leading to a perceptually convincing restoration, particularly in complex restoration scenarios. These results
strongly support our analysis in \cref{sec:Gradient}.



\begin{figure*}[t]
    \centering
    \includegraphics[width=\linewidth]{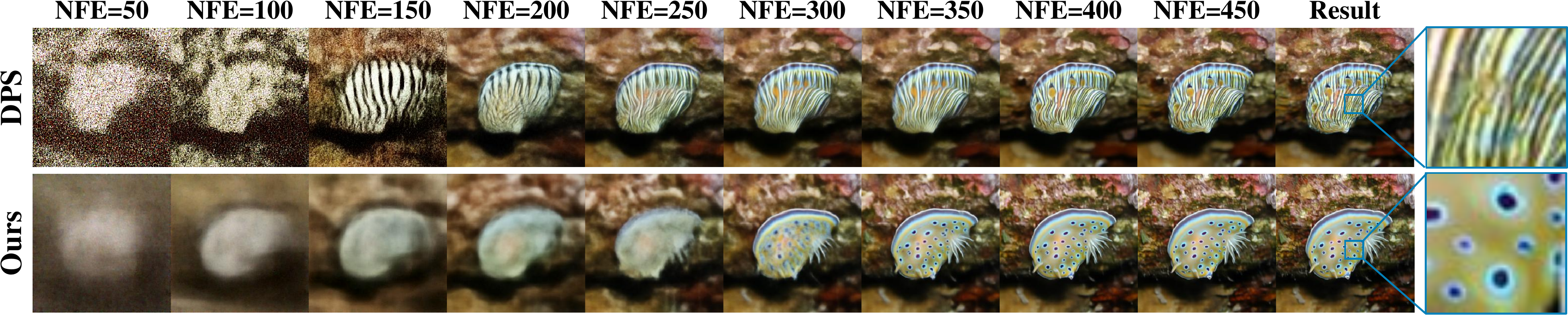} 
        \caption{Comparison of intermediate predictions $\hat{\mathbf{x}}_{0}(\mathbf{x}_t)$ from DPS and our SPGD method on motion deblurring.}
    \label{fig:convenge}
    \vspace{-0.3cm}
\end{figure*}

\begin{figure}[t]
    \centering
    \includegraphics[width=\linewidth]{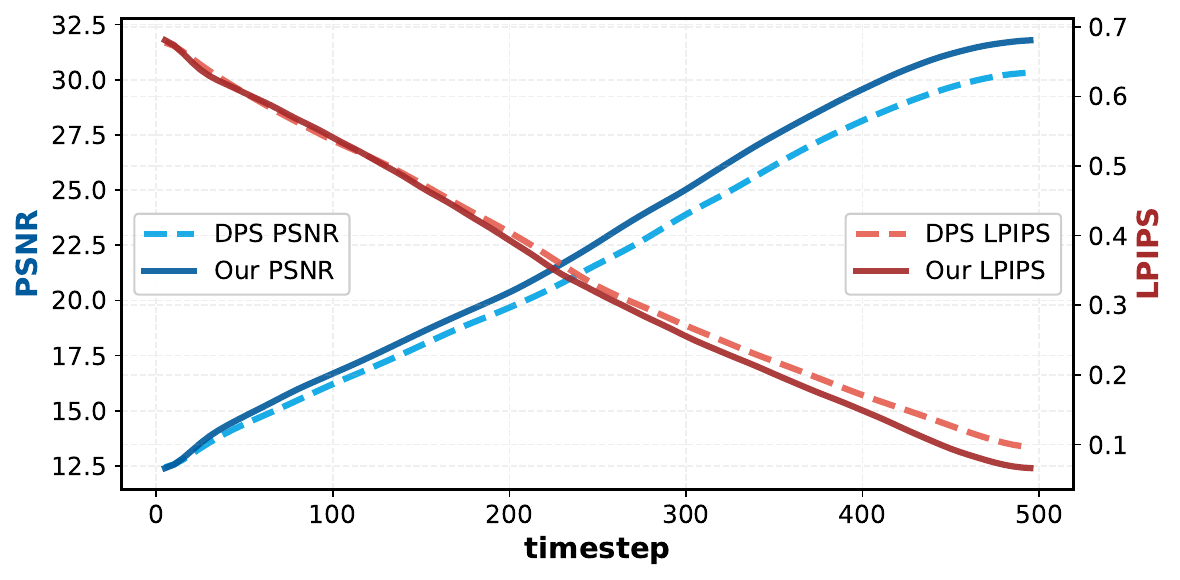}
        \caption{Evolution of PSNR and LPIPS of $\hat{\mathbf{x}}_{0}(\mathbf{x}_t)$ during the reverse process for SPGD (T=100, N=5) and the baseline without gradient management (T=500).}
    \label{fig:psnr_lpips}
    \vspace{-0.3cm}
\end{figure}

\subsection{Ablation Study}

\subsubsection{Impact of Design Components}

SPGD integrates two core ideas: the progressive warm-up strategy and the ADM smoothing technique. To isolate their respective contributions, we evaluated our ablation study on four variants with 100 FFHQ test images: baseline\footnote{The performance discrepancy between 'baseline' and DPS is attributable to our improved experimental settings, as detailed in~\cref{sec:Implementation_Details}.}; SPGD with only the warm-up strategy; SPGD with only ADM smoothing; and the full SPGD model.

The results, presented in \cref{tab:Components}, reveal the following observations: (i) Implementing only the warm-up strategy yields significant improvements across all metrics compared to the baseline. This confirms that performing multiple, smaller likelihood updates before the denoising step effectively mitigates gradient conflicts and enhances reconstruction quality.  (ii) Applying ADM smoothing directly fails to improve performance. We hypothesize this occurs because the momentum accumulates information based on likelihood gradients calculated at different states ($t-1$, $t$, etc.) where the underlying image estimate has significantly changed due to the intervening denoising steps. This makes the historical gradient information less relevant or potentially misleading for the current step's likelihood correction. (iii) Combining both techniques results in the best performance, substantially outperforming either component alone. This demonstrates a crucial synergy: ADM effectively stabilizes the likelihood gradient within the warm-up phase, where likelihood gradients $\mathbf{g}_l(\mathbf{x}_t^{(j)})$ are computed at the same effective timestep $t$.
These findings underscore that both components are essential and work synergistically, validating the design of SPGD.

\subsubsection{Ablation Study on $N$}
The number of inner warm-up steps, denoted as $N$, is a central hyperparameter in our proposed SPGD. To analyze its impact, we performed an ablation study varying $N$ while keeping the number of outer diffusion steps fixed at $T=100$. We evaluated $N \in \{2, 5, 10, 20\}$ and included standard DPS with $T=500$ and $T=1000$ steps (corresponding to 500 and 1000 NFEs respectively) as references.

Results are presented in \cref{tab:ablation_computation}. We observe that increasing $N$ from 2 to 10 generally improves performance, particularly benefiting the perceptual LPIPS metric. Optimal performance appears around $N=5$ or $N=10$, depending on the specific metric. However, increasing $N$ further to 20 leads to a noticeable decline in quality, potentially due to overfitting within the warm-up phase. Considering the trade-off between performance and the computational overhead, we select $N=5$ as the default value for SPGD.

Notably, as shown in \cref{tab:ablation_computation}, with the same 500 NFEs, SPGD with $T=100, N=5$ significantly outperforms standard DPS using $T=500$ steps, and even surpasses DPS with $T=1000$ steps. 
Further evidence is provided in \cref{fig:psnr_lpips}, which compares the evolution of PSNR and LPIPS during the reverse process for SPGD (T=100, N=5) and DPS (T=500). SPGD consistently achieves superior metric values throughout the generation process for the same NFEs, indicating not only better final results but also faster convergence towards a high-quality solution.

\begin{table}[t]          
\caption{Ablation study on the impact of warm-up steps $N$.}  
        \centering
        \begin{tabular}{llccc}
            \toprule        
            \multicolumn{2}{c}{Allocation}  & {PSNR$\uparrow$} &  {SSIM$\uparrow$} &{LPIPS$\downarrow$}   \\
            \midrule
             $T$= 100  &$N$=1 &30.357 & 0.861 & 0.154  \\
            $T$= 100  &$N$=2 & \textbf{31.655} & 0.896 & 0.109 \\   
            \rowcolor{lightblue}
           $T$= 100  &$N$=5 & \underline{31.631} & \textbf{0.904} & \underline{0.075} \\       
            $T$= 100  & $N$=10 & 30.923 & 0.893 & \textbf{0.070}  \\
             $T$= 100  & $N$=20  & 29.900 & 0.874 & 0.079   \\
            \midrule
            $T$=500  &$N$=1 & 30.403 & 0.863 & 0.096  \\
            $T$=1000  &$N$=1 & 31.608 & \underline{0.902} & 0.080   \\
             \bottomrule
        \end{tabular}
        \vspace{-0.3cm} 
        \label{tab:ablation_computation}
\end{table}

\subsubsection{Influence of \texorpdfstring{$\beta$}{beta}}

To investigate the influence of the parameter $\beta$, we conduct an ablation study on Gaussian deblurring and SR $\times$4 tasks. Quantitative results, summarized in Table~\ref{tab:ablation_beta}, reveal that setting $\beta$ within the range of 0.9 to 0.95 leads to optimal performance.

For the Gaussian Deblurring task, the configuration $\beta=0.9$ yields the highest PSNR (28.32), whereas $\beta=0.95$ achieves the best SSIM (0.788) and LPIPS (0.159) scores. In the case of SR $\times$4, $\beta=0.95$ consistently outperforms other values across all evaluation metrics. Notably, performance degrades as $\beta$ approaches 1.0, with the particularly decline in the deblurring scenario. Based on these observations, we adopt $\beta=0.95$ as the default setting in our method.

\begin{table}[t]          
\caption{Ablation study on the impact of parameter $\beta$.}  
        \centering
        \resizebox{.99\linewidth}{!}{%
        \begin{tabular}{lcccccc}
            \toprule        
            {} &  \multicolumn{3}{c}{Deblur (Gaussian)}& \multicolumn{3}{c}{SR ($\times$4)}\\
            \cmidrule(lr){2-4}
            \cmidrule(lr){5-7}
            {Setting}  & {PSNR$\uparrow$} &  {SSIM$\uparrow$} &{LPIPS$\downarrow$}  &{PSNR$\uparrow$}& {SSIM$\uparrow$}  & {LPIPS$\downarrow$}  \\
            \midrule
            $\beta$=0.8 &28.26 & 0.781 & 0.166 & 30.02 & 0.843 & 0.130 \\
		$\beta$=0.9 & \textbf{28.32} & \underline{0.787} & \underline{0.161} & \underline{30.06} & 0.845 & 0.127 \\
        \rowcolor{lightblue}
		$\beta$=0.95 &\underline{28.28} &\textbf {0.788} & \textbf {0.159}  & \textbf {30.08 }& \textbf {0.847} & \textbf {0.124} \\   
		 $\beta$=0.99 &26.93 & 0.756 & 0.189 & 29.99 & \underline{0.846} & \underline{0.124}  \\       
		$\beta$=1.0 &25.15 & 0.669 & 0.249& 29.92 & 0.846 & 0.126	 \\
            \bottomrule
        \end{tabular}
        }
        \label{tab:ablation_beta}
        \vspace{-0.3cm}
\end{table}

\section{Conclusion}

In this paper, we examined  the reverse update from a gradient decomposition perspective, identified and addressed previously under-addressed gradient conflicts and fluctuations that hinder diffusion-based image restoration methods. Our proposed SPGD, incorporating progressive warm-up and adaptive momentum smoothing, directly mitigates these instabilities, leading to a more stable reverse process. Extensive experiments demonstrated SPGD's superiority over state-of-the-art methods across various tasks, achieving significant quantitative and qualitative improvements, often with enhanced efficiency.  

SPGD provides a practical improvement in image restoration and offers valuable insights into stabilizing guided diffusion processes.
Future work could broaden the application of SPGD to diverse conditional diffusion frameworks, including text-to-image synthesis, image-to-image translation, and video generation.


\begin{acks}
This work was supported by the National Major Scientific Instruments and Equipments Development Project of National Natural Science Foundation of China under Grant 62427820, the Fundamental Research Funds for the Central Universities under Grant 1082204112364, and the Science Fund for Creative Research Groups of Sichuan Province Natural Science Foundation under Grant 2024NSFTD0035. 
\end{acks}

\bibliographystyle{ACM-Reference-Format}
\bibliography{bib}

\clearpage
\appendix

\section{Derivation for Reverse Update Decomposition}
\label{app:derivation}
In this part, we formally derive the DPS algorithm with DDIM sampling to our gradient perspective decomposition \cref{eq:grad_view}.

The standard DDIM sampling with $\sigma_t=0$ is given by:
\begin{equation}
    \mathbf{x}_{t-1}^{\prime} = \sqrt{\bar{\alpha}_{t-1}} \left( \frac{\mathbf{x}_t - \sqrt{1-\bar{\alpha}_t} \boldsymbol{\epsilon}_\theta(\mathbf{x}_t, t)}{\sqrt{\bar{\alpha}_t}} \right) + \sqrt{1-\bar{\alpha}_{t-1}} \boldsymbol{\epsilon}_\theta(\mathbf{x}_t, t),
\end{equation}
where $\bar{\alpha}_t = \prod_{s=1}^{t} \alpha_s$ is the cumulative product of the noise schedule, and $\boldsymbol{\epsilon}_\theta(\mathbf{x}_t, t)$ denotes the estimated noise obtained from the learned network.

Rewriting the scaling term,
$$
    \frac{\sqrt{\bar{\alpha}_{t-1}}}{\sqrt{\bar{\alpha}_t}} = \frac{1}{\sqrt{\alpha_t}},
$$
we express $\mathbf{x}_{t-1}^{\prime}$ as:
\begin{equation}
\begin{aligned}
    \mathbf{x}_{t-1}^{\prime} &=\frac{\sqrt{\bar{\alpha}_{t-1}}}{\sqrt{\bar{\alpha}_t}} \mathbf{x}_t - \frac{\sqrt{\bar{\alpha}_{t-1}}}{\sqrt{\bar{\alpha}_t}}  \sqrt{1-\bar{\alpha}_t} \boldsymbol{\epsilon}_\theta(\mathbf{x}_t, t) + \sqrt{1-\bar{\alpha}_{t-1}} \boldsymbol{\epsilon}_\theta(\mathbf{x}_{t}, t)
    \\
    &= \frac{\sqrt{\bar{\alpha}_{t-1}}}{\sqrt{\bar{\alpha}_t}} \mathbf{x}_t - \left(\frac{\sqrt{\bar{\alpha}_{t-1}}}{\sqrt{\bar{\alpha}_t}} \sqrt{1-\bar{\alpha}_t}-
    \sqrt{1-\bar{\alpha}_{t-1}}     \right) \boldsymbol{\epsilon}_\theta(\mathbf{x}_{t}, t)
    \\
    &=\frac{1}{\sqrt{\alpha_t}} \mathbf{x}_t - \left( \frac{\sqrt{1-\bar{\alpha}_t}}{\sqrt{\alpha_t}} - \sqrt{1-\bar{\alpha}_{t-1}} \right) \boldsymbol{\epsilon}_\theta(\mathbf{x}_t, t).
\end{aligned}
\end{equation}
Here, the second term represents the denoising gradient, which we denote as:
\begin{equation}
    \mathbf{g}_d(\mathbf{x}_t) = \left( \frac{\sqrt{1-\bar{\alpha}_t}}{\sqrt{\alpha_t}} - \sqrt{1-\bar{\alpha}_{t-1}} \right) \boldsymbol{\epsilon}_\theta(\mathbf{x}_t, t).
\end{equation}

For conditional generation based on measurement $\y$, DPS introduce an additional likelihood update:
\begin{equation}
    \mathbf{x}_{t-1} = \mathbf{x}_{t-1}^{\prime} - \zeta \nabla_{\mathbf{x}_t} \| \mathbf{y} - \mathcal{A}(\hat{\mathbf{x}}_{0}(\mathbf{x}_t)) \|_2^2,
\end{equation}
where $\mathbf{y}$ represents the observed data, $\mathcal{A}$ is the observation operator, and $\hat{\mathbf{x}}_{0}(\mathbf{x}_t)$ is the estimated clean image given by:
\begin{equation}
    \hat{\mathbf{x}}_{0}(\mathbf{x}_t) = \frac{\mathbf{x}_t - \sqrt{1-\bar{\alpha}_t} \boldsymbol{\epsilon}_\theta(\mathbf{x}_t, t)}{\sqrt{\bar{\alpha}_t}}.
\end{equation}
The likelihood gradient term is thus defined as:
\begin{equation}
    \mathbf{g}_l(\mathbf{x}_t) = \nabla_{\mathbf{x}_t} \| \mathbf{y} - \mathcal{A}(\hat{\mathbf{x}}_{0}(\mathbf{x}_t)) \|_2^2.
\end{equation}

Finally, the complete algorithm in an gradient perspective can be expressed as:
\begin{equation}
\begin{aligned}
    \mathbf{x}_{t-1} =&\underbrace{\frac{1}{\sqrt{\alpha_t}} \mathbf{x}_t}_{\text{fixed scaling}} - \underbrace{\left(\frac{\sqrt{1-\bar{\alpha}_t}}{\sqrt{{\alpha}_t}}-\sqrt{1-\bar{\alpha}_{t-1}} \right) \boldsymbol{\epsilon}_{\boldsymbol\theta}(\mathbf{x}_t, t)}_{\text{denoising gradient}~\mathbf{g}_d(\mathbf{x}_t)}  \\
& - {\zeta}\underbrace{ \nabla_{\mathbf{x}_t} \|\mathbf{y} - \mathcal{A}(\hat{\mathbf{x}}_{0}(\mathbf{x}_t))\|_2^2}_{\text{likelihood gradient}~\mathbf{g}_l(\mathbf{x}_t)} \\
    = & \underbrace{\frac{1}{\sqrt{\alpha_t}} \mathbf{x}_t}_{\text{fixed scaling}} - \underbrace{\mathbf{g}_d(\mathbf{x}_t)}_{\text{denoising gradient}} - \zeta \underbrace{\mathbf{g}_l(\mathbf{x}_t)}_{\text{likelihood gradient}},
\end{aligned}
\end{equation}
which matches the \cref{eq:grad_view} in the main paper.

\section{Proofs}
\label{sec:proof}
\subsection{Descent Lemma for L-Smooth Functions}
To support the proof of Proposition~\ref{prop:warmup_optimality}, we first state and prove the following lemma concerning functions with $L$-smooth gradients.

\begin{lemma} \label{lem:descent}
Let $f: \Real^d \to \Real$ be a continuously differentiable function whose gradient $\nabla f$ is $L$-Lipschitz continuous for some $L > 0$, i.e., $\norm{\nabla f(\wprime) - \nabla f(\w)} \le L \norm{\wprime - \w}$ for all $\w, \wprime \in \Real^d$. Then, for any $\w, \wprime \in \Real^d$:
\begin{align}
f(\wprime) \le f(\w) + \langle \nabla f(\w), \wprime-\w \rangle + \frac{L}{2} \norm{\wprime-\w}^2.
\end{align}
\end{lemma}

\begin{proof}
By the Fundamental Theorem of Calculus for vector functions, we have:
\begin{align}
f(\wprime) - f(\w) = \int_0^1 \langle \nabla f(\w + t(\wprime - \w)), \wprime - \w \rangle \, dt.
\end{align}
Adding and subtracting $\langle \nabla f(\w), \wprime - \w \rangle = \int_0^1 \langle \nabla f(\w), \wprime - \w \rangle \, dt$:
\begin{align}
f(\wprime) - f(\w) =& \langle \nabla f(\w), \wprime - \w \rangle+ \int_0^1 \langle \nabla f(\w + t(\wprime - \w)) \notag \\ & - \nabla f(\w), \wprime - \w \rangle \, dt.
\end{align}
Applying the Cauchy-Schwarz inequality to the integrand:
\begin{align}
&\langle \nabla f(\w + t(\wprime - \w)) - \nabla f(\w), \wprime - \w \rangle \notag \\
\le& \norm{\nabla f(\w + t(\wprime - \w)) - \nabla f(\w)} \norm{\wprime - \w}.
\end{align}
Using the $L$-Lipschitz property of the gradient $\nabla f$:
\begin{align}
&\norm{\nabla f(\w + t(\wprime - \w)) - \nabla f(\w)} \notag\\
\le& L \norm{(\w + t(\wprime - \w)) - \w} \notag\\
=& L \norm{t(\wprime - \w)} \notag\\
=& L t \norm{\wprime - \w}.
\end{align}
Substituting this back into the integral bound:
\begin{align}
& f(\wprime) - f(\w) \notag \\
\le& \langle \nabla f(\w), \wprime - \w \rangle + \int_0^1 (L t \norm{\wprime - \w}) \norm{\wprime - \w} \, dt \notag \\
=& \langle \nabla f(\w), \wprime - \w \rangle + L \norm{\wprime - \w}^2 \int_0^1 t \, dt.
\end{align}
Evaluating the integral $\int_0^1 t \, dt = 1/2$, we obtain:
\begin{align}
f(\wprime) - f(\w) \le \langle \nabla f(\w), \wprime - \w \rangle + \frac{L}{2} \norm{\wprime - \w}^2,
\end{align}
which gives the desired inequality.
\end{proof}

\subsection{Proof of Proposition \ref{prop:warmup_optimality}}

\begin{proof}[Proof of \Cref{prop:warmup_optimality}] 
We apply the Lemma~\ref{lem:descent} to each iteration $j$ of the SPGD warm-up loop (Line 9 in Algorithm~\ref{algo:algo}, with $\beta=0$). The update is $\x_t^{(j+1)} = \x_t^{(j)} - \eta \g_l(\x_t^{(j)})$.
Let $f = L_t$, $\w = \x_t^{(j)}$, and $\wprime = \x_t^{(j+1)}$. The displacement is $\wprime-\w = -\eta \g_l(\w)$. Substituting into Lemma~\ref{lem:descent}, we get:
\begin{align}
L_t(\x_t^{(j+1)}) &\le L_t(\x_t^{(j)}) + \langle \g_l(\x_t^{(j)}), -\eta \g_l(\x_t^{(j)}) \rangle + \frac{L}{2} \norm{-\eta \g_l(\x_t^{(j)})}^2  \notag \\
&= L_t(\x_t^{(j)}) - \eta \langle \g_l(\x_t^{(j)}), \g_l(\x_t^{(j)}) \rangle + \frac{L}{2} \eta^2 \norm{\g_l(\x_t^{(j)})}^2 \notag \\
&= L_t(\x_t^{(j)}) - \eta \norm{\g_l(\x_t^{(j)})}^2 + \frac{L\eta^2}{2} \norm{\g_l(\x_t^{(j)})}^2 \notag \\
&= L_t(\x_t^{(j)}) - \eta \left(1 - \frac{L\eta}{2}\right) \norm{\g_l(\x_t^{(j)})}^2.
\end{align}
The condition $\eta < 1/L$ ensures that $(1 - L\eta/2) > 1/2 > 0$. Thus, each warm-up step guarantees a non-increasing likelihood objective, $L_t(\x_t^{(j+1)}) \le L_t(\x_t^{(j)})$, with strict decrease if $\g_l(\x_t^{(j)}) \neq \mathbf{0}$.

Summing this per-step inequality over $j=0, \dots, N-1$ yields:
\begin{equation} \label{eq:proof_sum}
\sum_{j=0}^{N-1} [L_t(\x_t^{(j+1)}) - L_t(\x_t^{(j)})] \le \sum_{j=0}^{N-1} \left[ - \eta \left(1 - \frac{L\eta}{2}\right) \norm{\g_l(\x_t^{(j)})}^2 \right]
\end{equation}
The right-hand side can be simplified by moving the constant factors outside the summation:
$$
\sum_{j=0}^{N-1} [L_t(\x_t^{(j+1)}) - L_t(\x_t^{(j)})] \le - \eta \left(1 - \frac{L\eta}{2}\right) \sum_{j=0}^{N-1} \norm{\g_l(\x_t^{(j)})}^2.
$$
The left-hand side of Eq.~\eqref{eq:proof_sum} is a telescoping sum:
\begin{align*}
&\sum_{j=0}^{N-1} [L_t(\x_t^{(j+1)}) - L_t(\x_t^{(j)})] \\
= &[L_t(\x_t^{(1)}) - L_t(\x_t^{(0)})] + [L_t(\x_t^{(2)}) - L_t(\x_t^{(1)})] + \dots \\
\quad& \dots + [L_t(\x_t^{(N)}) - L_t(\x_t^{(N-1)})] \\
=& L_t(\x_t^{(N)}) - L_t(\x_t^{(0)}).
\end{align*}
Substituting this back, we have:
$$
L_t(\x_t^{(N)}) - L_t(\x_t^{(0)}) \le - \eta \left(1 - \frac{L\eta}{2}\right) \sum_{j=0}^{N-1} \norm{\g_l(\x_t^{(j)})}^2.
$$
Rearranging this inequality to isolate $L_t(\x_t^{(N)})$ yields the final result stated in the proposition:
$$
L_t(\x_t^{(N)}) \le L_t(\x_t^{(0)}) - \eta \sum_{j=0}^{N-1} \left(1 - \frac{L\eta}{2}\right) \norm{\g_l(\x_t^{(j)})}^2.
$$
This completes the proof.
\end{proof}

\section{More Discussions}


\subsection{Novel Aspects of SPGD}

While prior works have explored momentum-based techniques~\cite{he2024faststablediffusioninverse, chung2024prompttuning} and multi-step likelihood updates~\cite{xu2025rethinking, ye2024tfg, huang2024constrained}, SPGD's novelty lies in three key areas:
\begin{enumerate}
    \item \textbf{Analysis-Driven Motivation:} Our core contribution is the novel gradient-centric analysis identifying the dual problems of \textit{gradient conflict} and \textit{gradient fluctuation}. SPGD's components are not an ad-hoc combination but are directly motivated by this analysis to solve these specific, identified issues, providing a principled foundation.
    \item \textbf{Technical Novelty of ADM:} Unlike prior works that use standard, constant momentum, our ADM is technically more advanced. It dynamically adjusts the momentum strength based on directional consistency (Eq.~\eqref{eq:alpha_def}), making it more robust than a fixed momentum coefficient, as validated in our ablation studies.
    \item \textbf{Synergistic Design:} Our warm-up phase is specifically designed to mitigate gradient conflict before the denoising step. This creates a stable context that enables ADM to effectively dampen fluctuations. The synergistic combination of a conflict-mitigating warm-up with a fluctuation-dampening ADM is a unique design absent in prior art.
\end{enumerate}




\begin{figure*}[t]
    \centering
  \begin{tikzpicture}
        \node[anchor=south west,inner sep=0] (image) at (0,0) {\includegraphics[width=.9\linewidth]{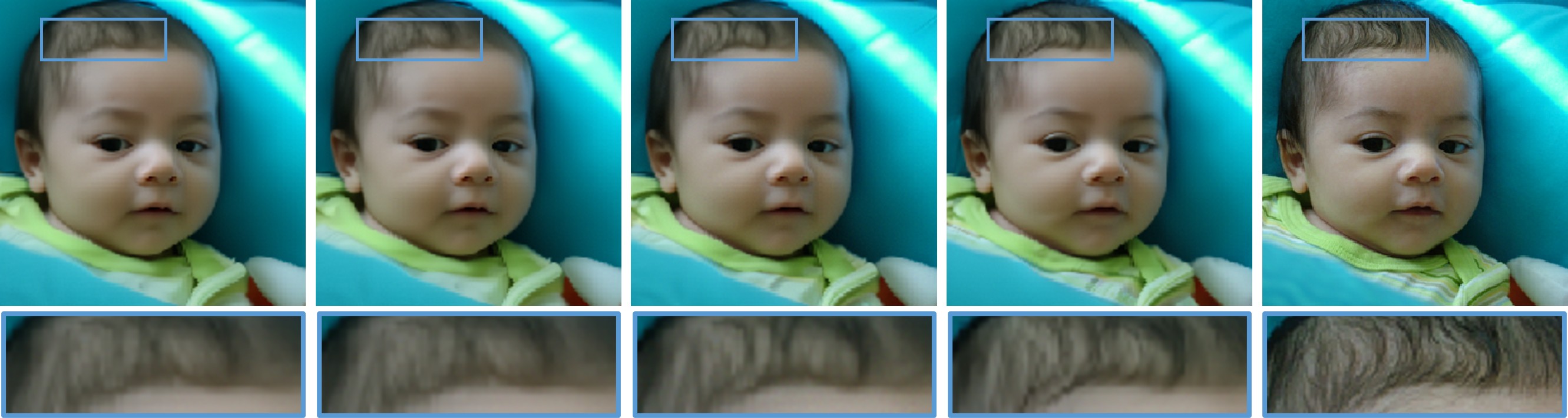}};
        \begin{scope}[x={(image.south east)},y={(image.north west)}]
            \node at (0.092,1.05) {{ baseline}};
            \node at (0.30,1.05) {{w/ ADM}};
            \node at (0.5,1.05) {{w/ warm-up}};
            \node at (0.70,1.05) {\textcolor{myblue}{\textbf{SPGD}}};
            \node at (0.90,1.05) {{reference}};
        \end{scope}
    \end{tikzpicture}
        \caption{Visual results of ablation study on design components. The complete SPGD shows clearer hair textures.}
    \label{fig:component}
    \vspace{-0.2cm}
\end{figure*}
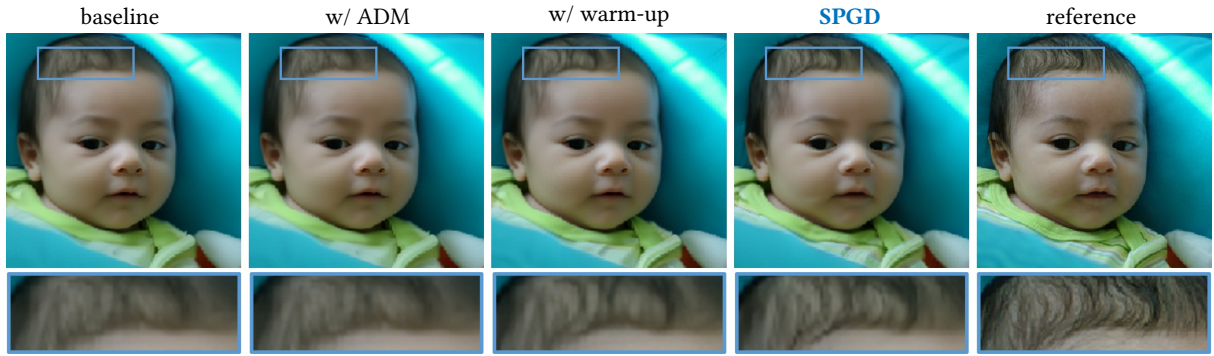

\subsection{Relation to Pixel-as-Parameter (PxP)}
\label{subsec:relation_pxp}

Our SPGD method shares conceptual ground with Pixel-as-Parameter (PxP)~\citep{dinh2023pixelasparam} in viewing the reverse process through a gradient lens and acknowledging potential gradient conflicts. However, our solutions are fundamentally different.

PxP's approach is to explicitly \textbf{resolve} conflicts by projecting gradients to enforce orthogonality before the update. In contrast, SPGD aims to \textbf{mitigate} conflicts by using a progressive warm-up phase. This phase allows the state $\x_t$ to gradually adapt to the likelihood guidance \textit{before} the denoising step is applied. We adopt this strategy because we found that explicit gradient projection can lead to information loss and degrade performance.

Furthermore, PxP computes all guidance terms from the initial $\x_t$ and applies them in a single, combined update. SPGD's architecture is different, introducing an inner loop that iteratively refines $\x_t$ with the likelihood gradient prior to denoising. In essence, SPGD employs sequential refinement for conflict mitigation, whereas PxP uses simultaneous projection for conflict resolution.




\subsection{ADM Stability}
\label{app:adm_theory}

While a full theoretical convergence proof for SPGD with ADM is a significant undertaking beyond the scope of this paper, we can provide further intuition on its stability. The ADM update rule, $\widetilde{\mathbf{g}}_l^{(j)} = \alpha_j \beta \widetilde{\mathbf{g}}_l^{(j-1)} + (1 - \alpha_j \beta) \mathbf{g}_l^{(j)}$, can be interpreted as a standard momentum update with a dynamically adjusted momentum coefficient $\beta' = \alpha_j \beta$.

The stability of standard momentum methods is well-understood and often relies on the step size and momentum coefficient being within certain bounds related to the properties of the objective function (e.g., its smoothness $L$ and strong convexity $\mu$). Our adaptive weight $\alpha_j \in [0, 1]$ ensures that the effective momentum $\beta'$ is always less than or equal to the base coefficient $\beta$ ($\beta' \le \beta$). When the gradient direction changes abruptly, $\alpha_j$ decreases, reducing the effective momentum $\beta'$. This reduction pushes the update behavior closer to that of standard gradient descent, a method with well-known and robust convergence guarantees under standard assumptions (e.g., sufficiently small step size).

In essence, ADM acts as a self-stabilizing mechanism. It leverages the acceleration benefits of momentum when the optimization landscape is smooth and consistent, but reduces its reliance on momentum precisely when the gradient signal becomes noisy or changes direction sharply. This connection to the more conservative gradient descent method provides strong theoretical intuition for why ADM contributes to the overall stability of the SPGD framework. 


\section{Implementation Details} \label{sec:Implementation_Details}
\subsection{Comparison Methods}
To ensure a fair comparison, all pixel-based diffusion methods (Score-SDE, MCG, DDRM, DPS, DiffPIR, DPPS, and RED-Diff) share the same pre-trained checkpoint. Below, we describe our implementation details for each method.

\textbf{PnP-ADMM:} We use the pre-trained DnCNN model from \cite{zhang2017beyond}. The regularization parameter is set to $\tau=0.2$, and the algorithm is run for 10 iterations on all tasks.

\textbf{Score-SDE:} In Score-SDE, data consistency is enforced via a projection step following unconditional diffusion denoising at every iteration. The projection configuration follows the recommendations in \cite{chung2022improving}.

\textbf{MCG:} Our implementation of MCG is based on the source code provided in \cite{chung2022improving}. We replicate the evaluation settings reported in the paper.

\textbf{DPS:} Experimental outcomes for DPS are obtained from the implementation in \cite{chung2023diffusion} using 1,000 DDPM sampling steps. All hyper-parameters are set as suggested in D.1 section of the corresponding paper.

\textbf{DDRM:} For DDRM, we utilize 20 NFEs DDIM sampling~\cite{song2021denoising} with parameters $\eta = 0.85$ and $\eta_B = 1.0$, as recommended in the literature.

\textbf{DiffPIR:} We adopt the original code of DiffPIR from \cite{zhu2023denoising}. We adjust the diffusion steps to 20 for implementation efficiency. Hyper-parameters are selected to reflect the settings for the noiseless case.

\textbf{DPPS:} We follow the default configuration as provided by the open-sourced code in \cite{wu2024diffusion}.

\textbf{RED-Diff:} The implementation for RED-Diff is adopted from \cite{mardani2024red-diff}. We use the Adam optimizer over 1,000 steps, setting $\lambda=0.25$ and $lr=0.25$, consistent with the experimental settings in the paper.

\subsection{Our SPGD:}
For our SPGD implementation detailed in \cref{algo:algo}, we follow the time step discretization and noise schedule from EDM~\cite{karras2022elucidating}. We use $T=100$ outer diffusion steps. Each outer step includes $N=5$ inner warm-up iterations, corresponding to 500 NFEs in total. The base momentum coefficient for ADM is set to $\beta=0.95$.

The learning rate $\zeta$ is tuned separately for each task. The specific values are provided in \cref{tab:Hyperparameters}.

\begin{table}[htbp]
\centering
\caption{Learning rates $\zeta$ for each task.}
\resizebox{.75\linewidth}{!}{%
\begin{tabular}{lcc}
\toprule
Dataset & FFHQ & ImageNet \\
\midrule
Inpainting           & 2.5 & 2.5   \\
Deblurring (Gaussian)& 1.5 & 1     \\
Deblurring (Motion)  & 1   & 1     \\
Super-Resolution ($\times 4$) & 8   & 7     \\
\bottomrule
\end{tabular}
}
\label{tab:Hyperparameters}
\end{table}






\subsection{Computational Cost}

\begin{table}[ht]
\centering
\caption{Computational Cost (\(T=100\)) under different \(N\).}
\begin{tabular}{lcccccc}
\toprule
\textbf{Metric} & \textbf{$N$=1} & \textbf{$N$=2} & \textbf{$N$=5} & \textbf{$N$=10} & \textbf{$N$=20} \\
\midrule
FLOPS (T) ↓    & 19.38  & 38.75  & 96.89   & 193.77  & 387.54 \\
Time (s) ↓     & 8.69   & 15.11  & 33.83   & 65.19   & 129.14 \\
\bottomrule
\end{tabular}
\label{tab:computational_cost}
\end{table}

To provide a clearer understanding of the computational overhead, we report both the inference time and the forward-pass cost (in FLOPs) of the network $\epsilonb_{\thetab}$ for each image on RTX 4090 GPU, with $T=100$ and varying $N$ on FFHQ inpainting.

\section{More Visual Results} \label{app:more}
In this section, we provide supplementary experimental results to further substantiate the effectiveness of our proposed method. Specifically, \cref{fig:component} presents high-resolution visualizations for the ablation study on the design components.
\cref{fig:more_1} to \cref{fig:more_4} further illustrate the visual advantages of our proposed approach.
\cref{fig:last1} to \cref{fig:last4} provide additional examples of angular relationships of gradients, both before and after ADM smoothing.

\begin{figure*}[ht]
    \centering
    \includegraphics[width=.99\linewidth]{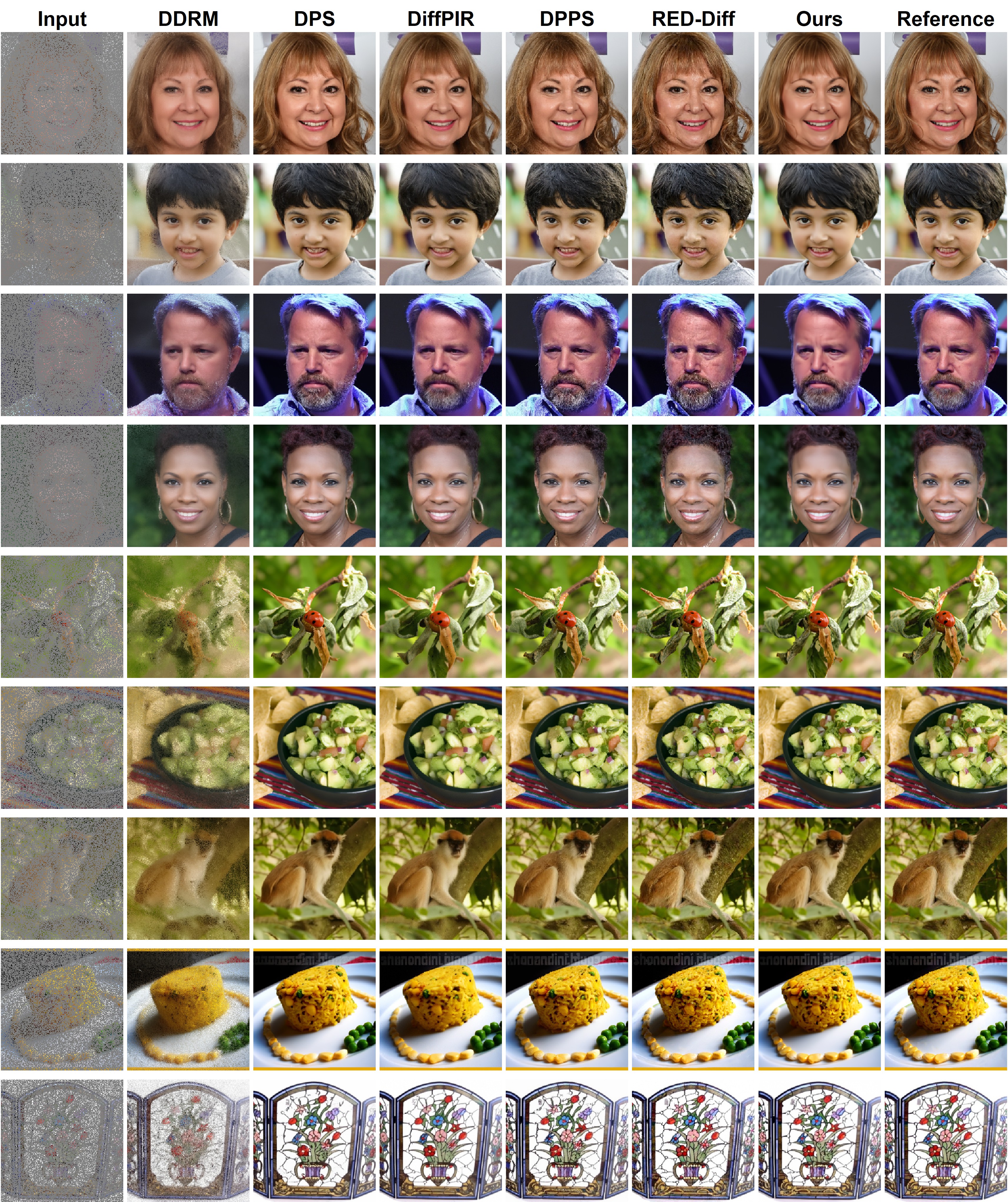}
    \caption{Qualitative visual comparison of \textbf{random inpainting} on FFHQ (top rows) and ImageNet (bottom rows).}\label{fig:more_1}
\end{figure*}

\begin{figure*}[ht]
    \centering
    \includegraphics[width=.99\linewidth]{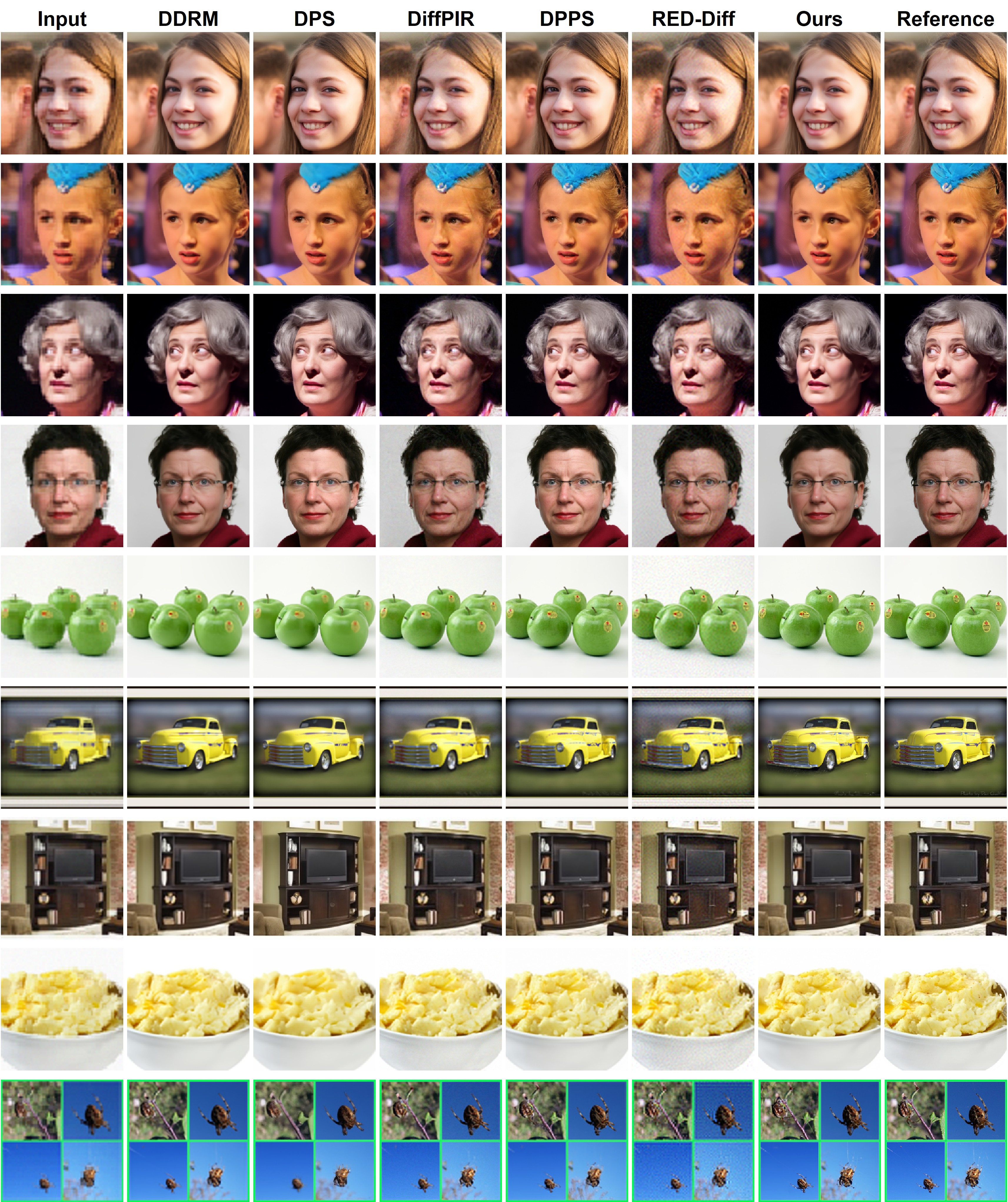}
    \caption{Qualitative visual comparison of \textbf{SR $\times$4} on FFHQ (top rows) and ImageNet (bottom rows).}\label{fig:more_2}
\end{figure*}

\begin{figure*}[ht]
    \centering
    \includegraphics[width=.99\linewidth]{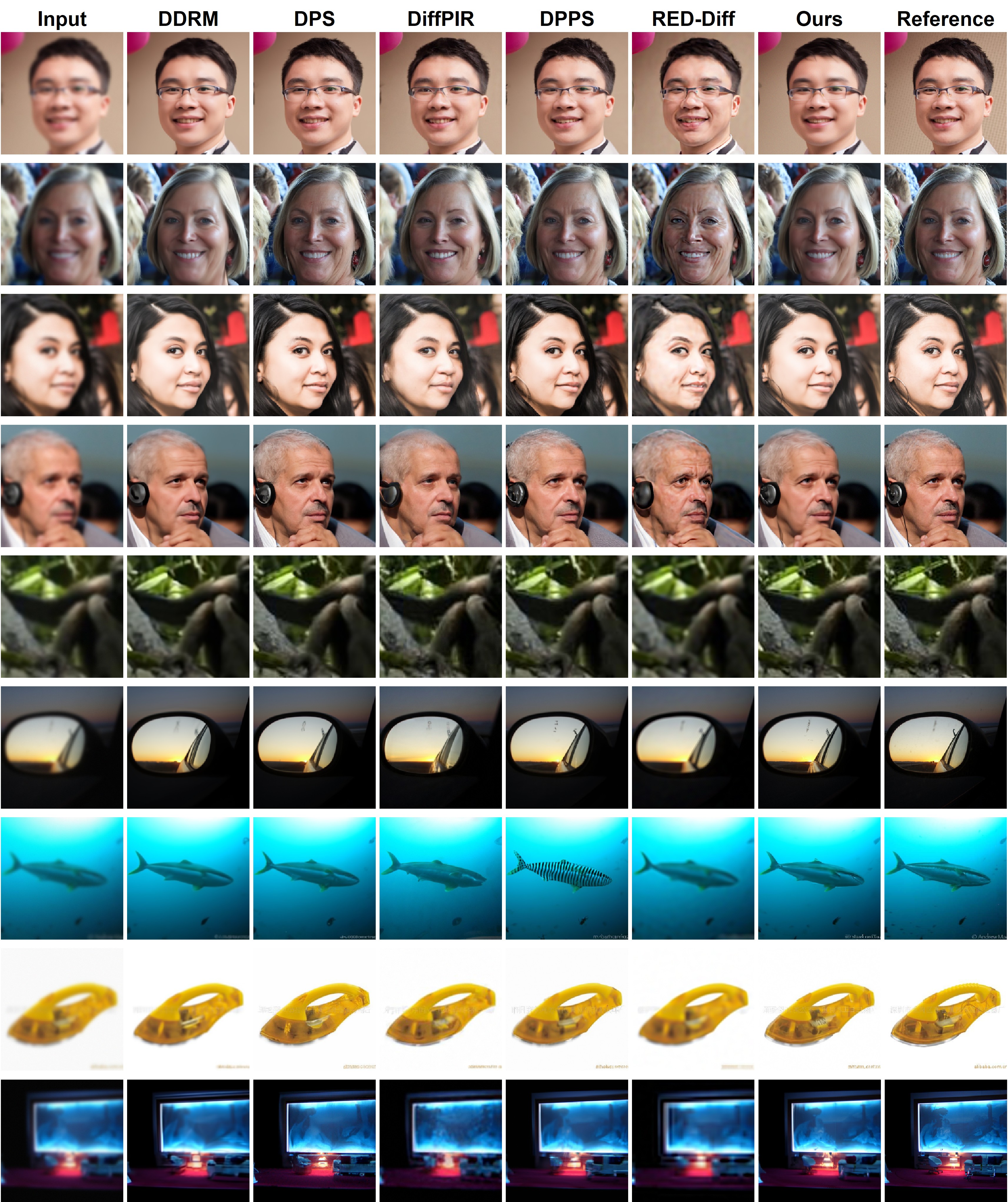}
    \caption{Qualitative visual comparison of \textbf{Gaussian deblurring} on FFHQ (top rows) and ImageNet (bottom rows).}\label{fig:more_3}
\end{figure*}

\begin{figure*}[ht]
    \centering
    \includegraphics[width=0.96\linewidth]{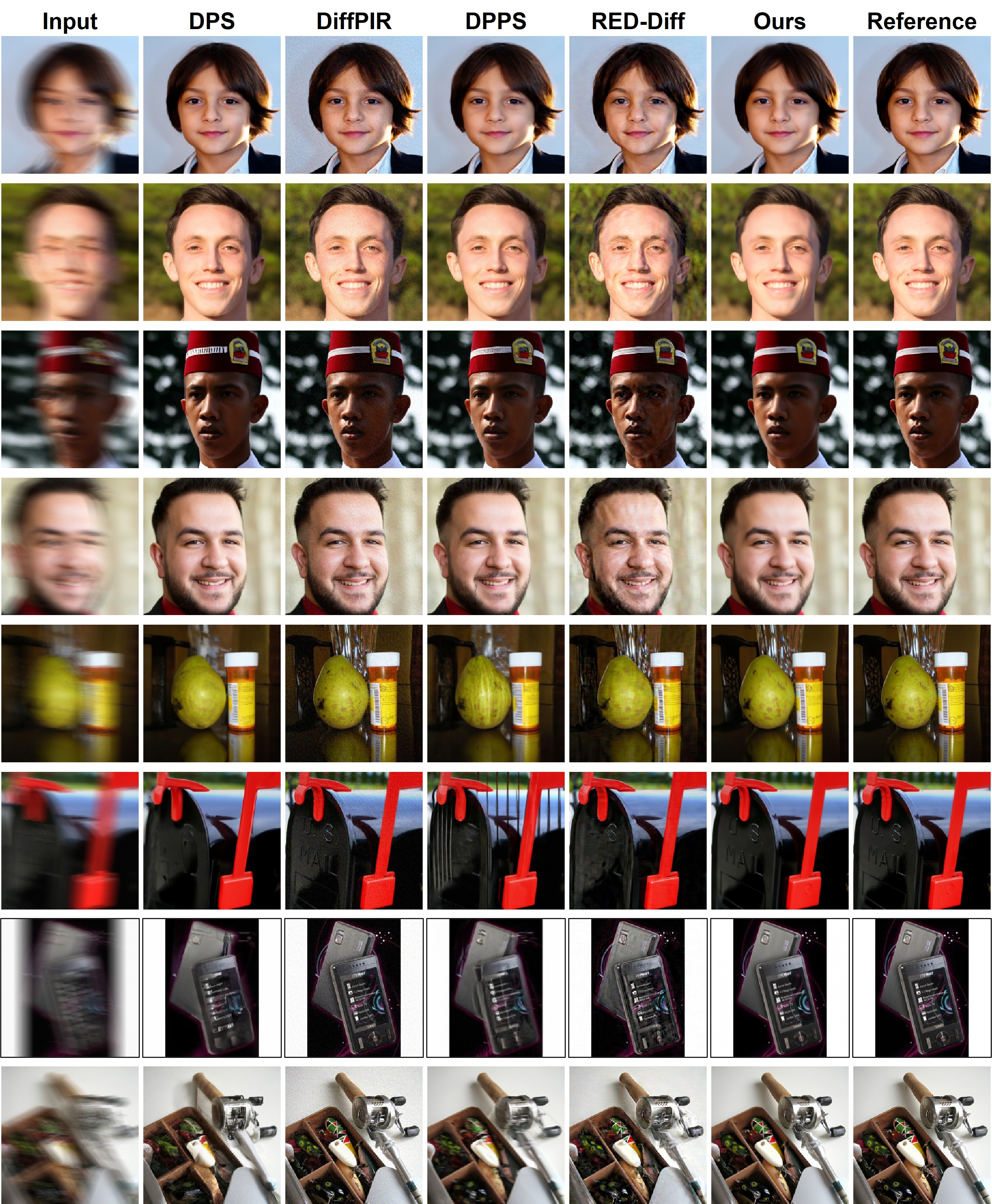}
    \caption{Qualitative visual comparison of \textbf{motion deblurring} on FFHQ (top rows) and ImageNet (bottom rows).}\label{fig:more_4}
\end{figure*}

\begin{figure*}[htbp] 
    \centering
    \begin{subfigure}{0.45\textwidth}
        \centering
        \includegraphics[width=0.49\linewidth]{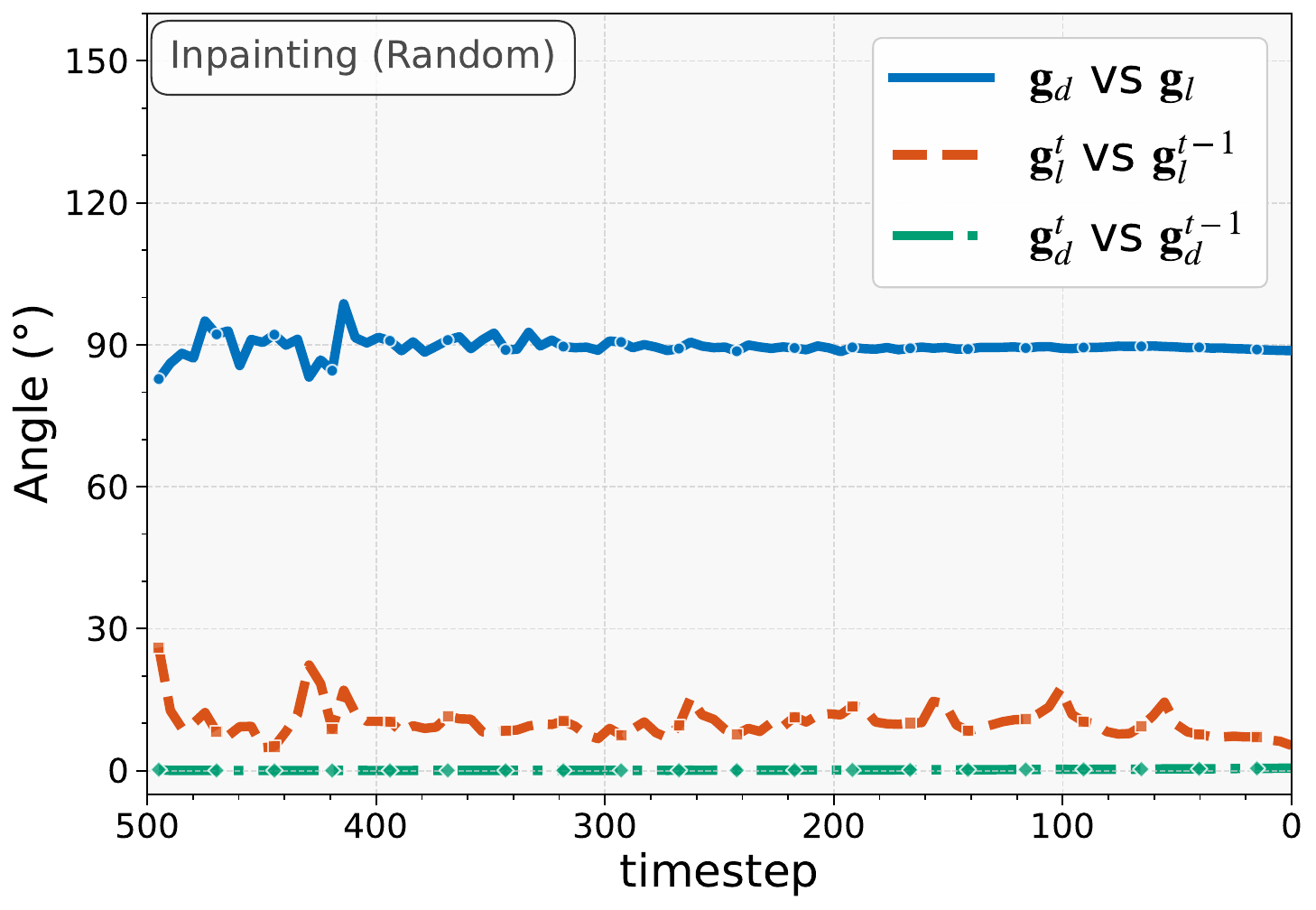}
        \includegraphics[width=0.49\linewidth]{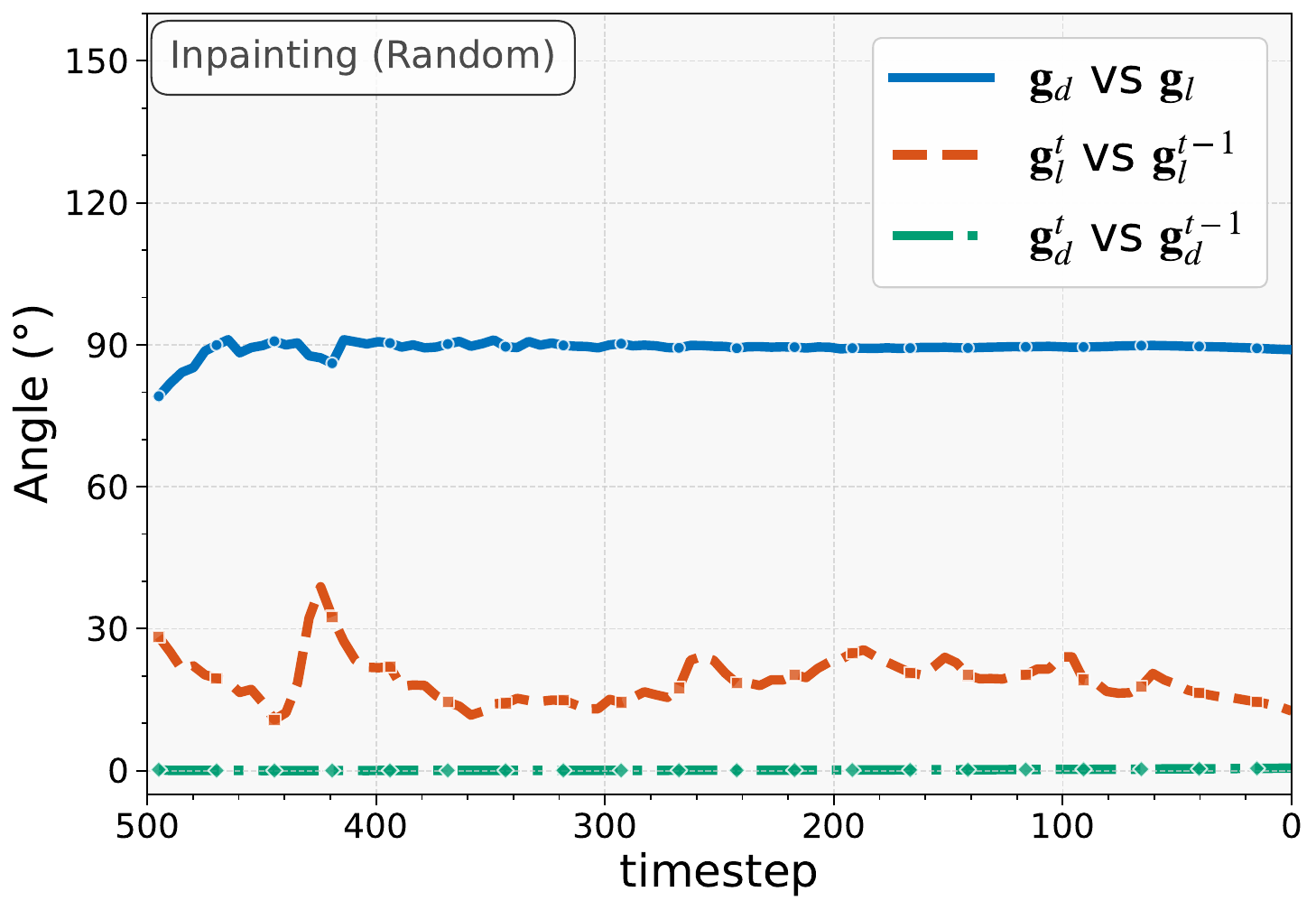}
        \caption{Sample 1}
    \end{subfigure}
    \begin{subfigure}{0.45\textwidth}
        \centering
        \includegraphics[width=0.49\linewidth]{before/Inpainting_Random_Column_2.pdf}
        \includegraphics[width=0.49\linewidth]{after/Inpainting_Random_Column_2.pdf}
        \caption{Sample 2}
    \end{subfigure}
        \begin{subfigure}{0.45\textwidth}
        \centering
        \includegraphics[width=0.49\linewidth]{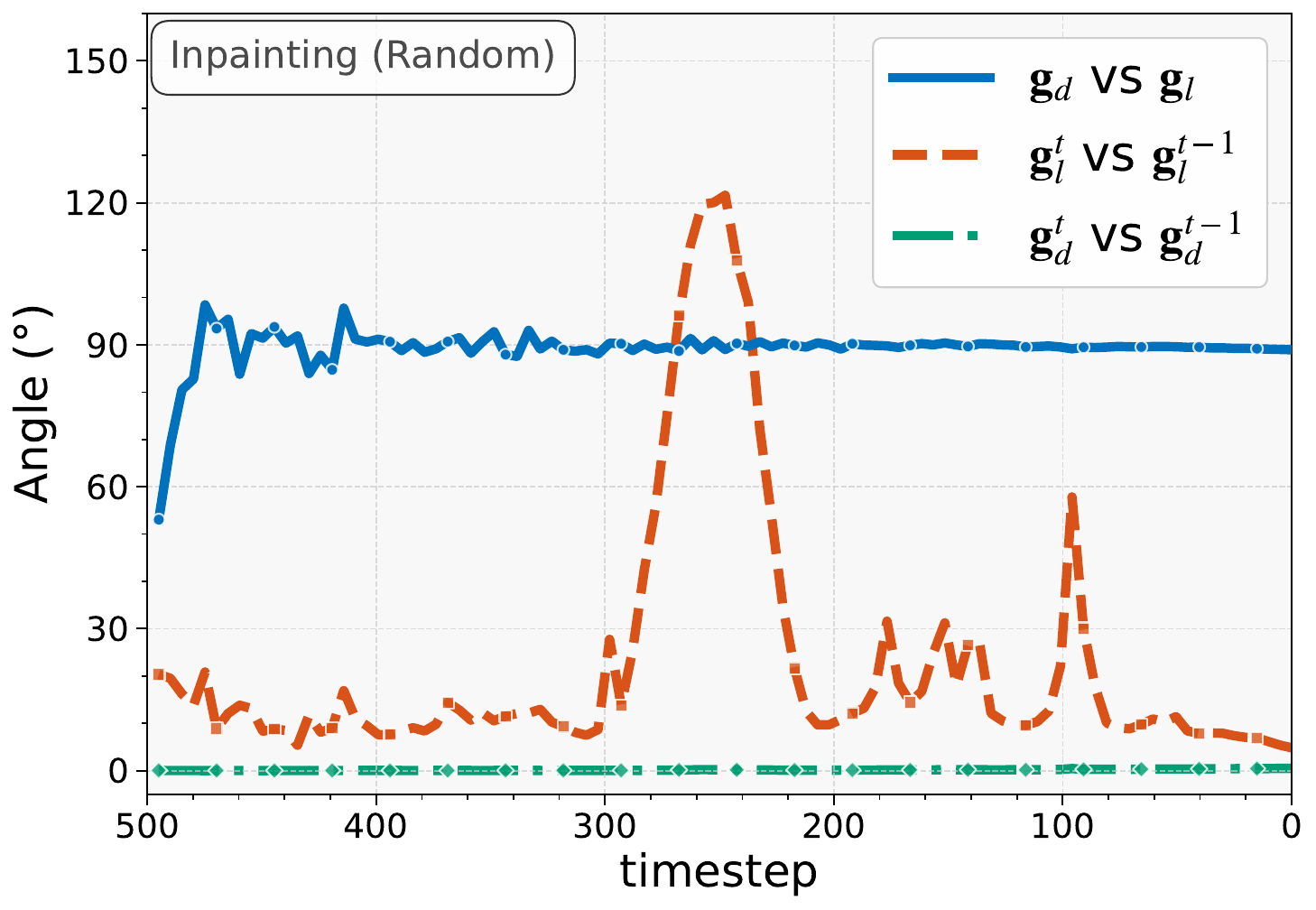}
        \includegraphics[width=0.49\linewidth]{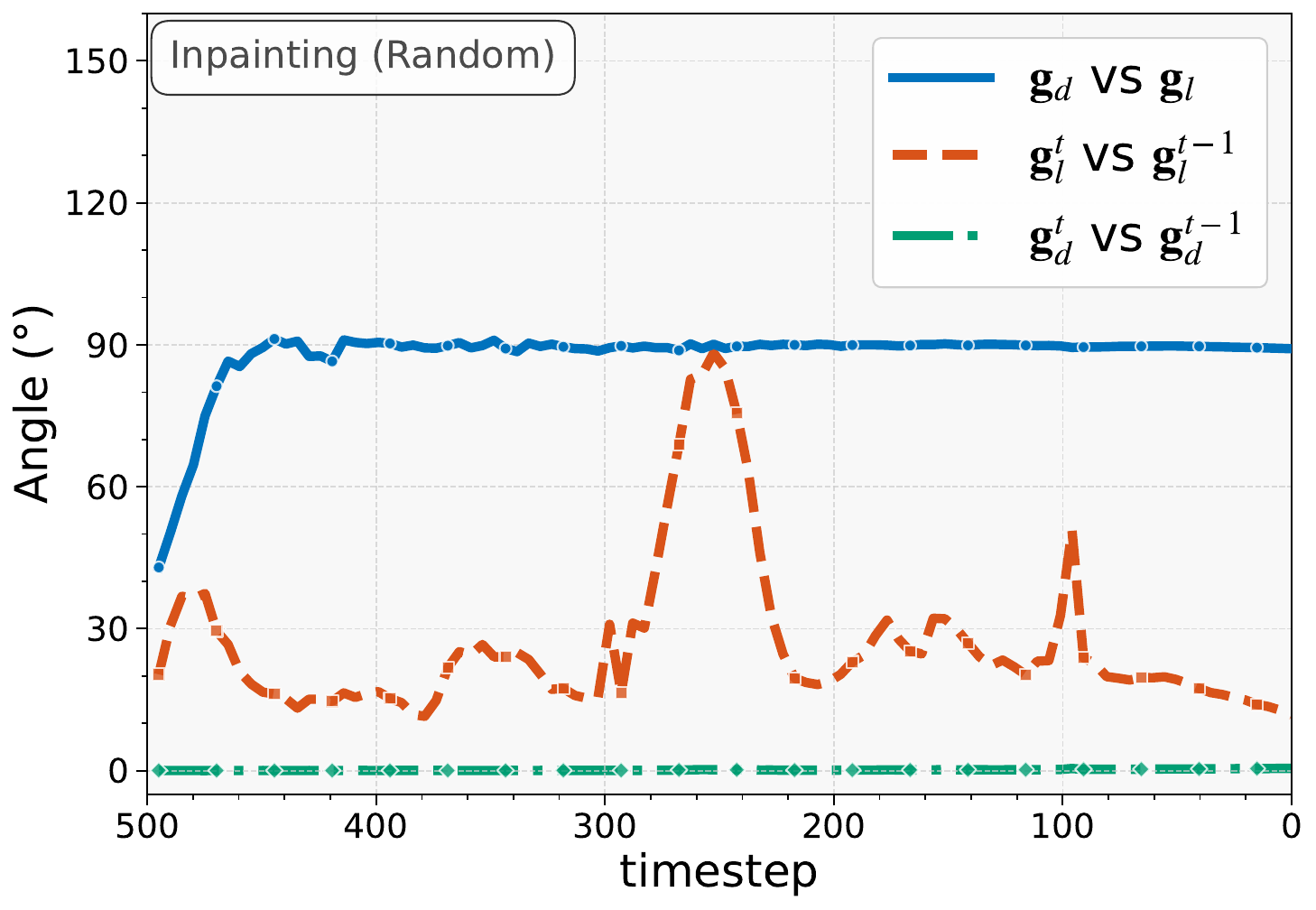}
        \caption{Sample 3}
    \end{subfigure}
    \begin{subfigure}{0.45\textwidth}
        \centering
        \includegraphics[width=0.49\linewidth]{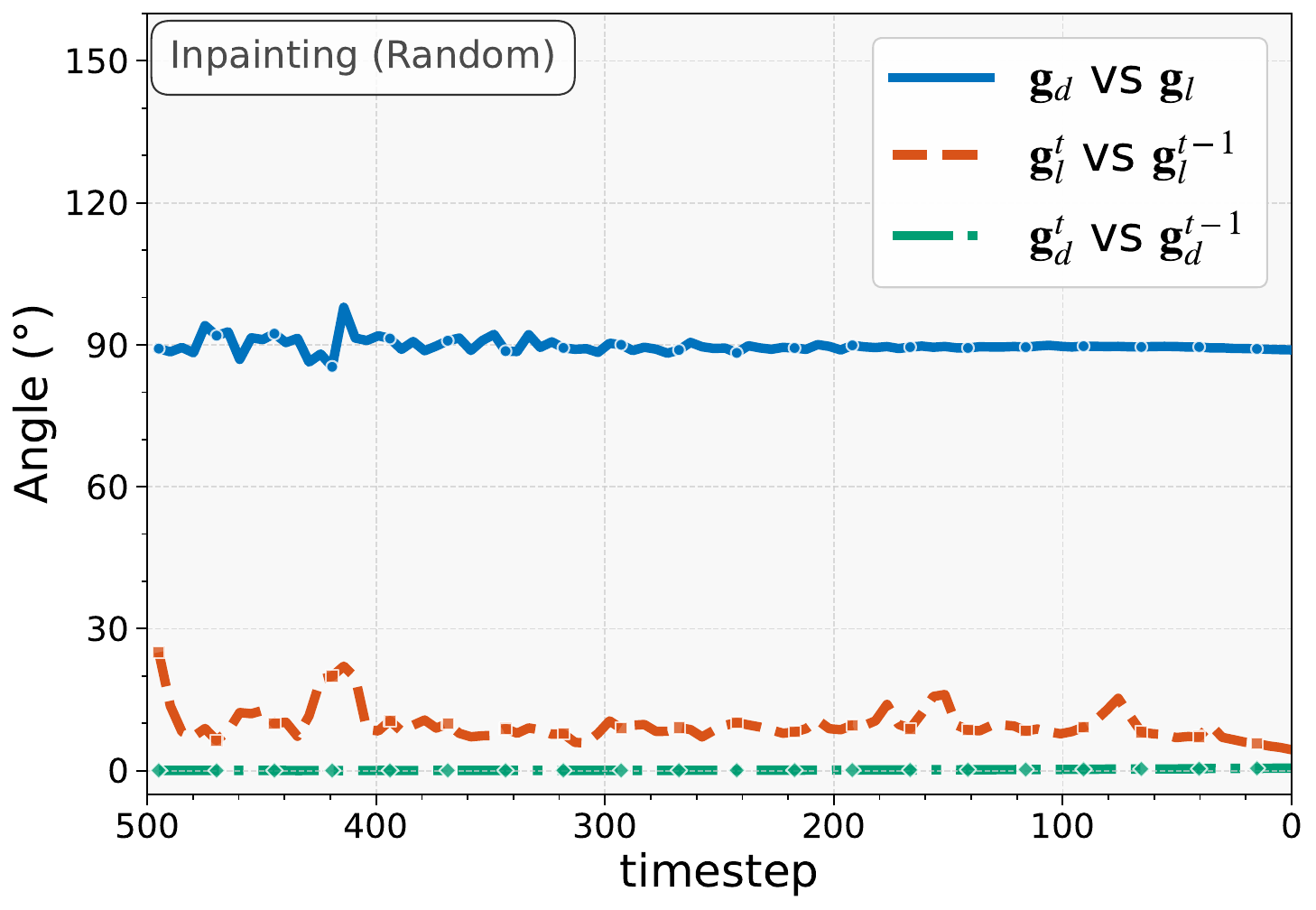}
        \includegraphics[width=0.49\linewidth]{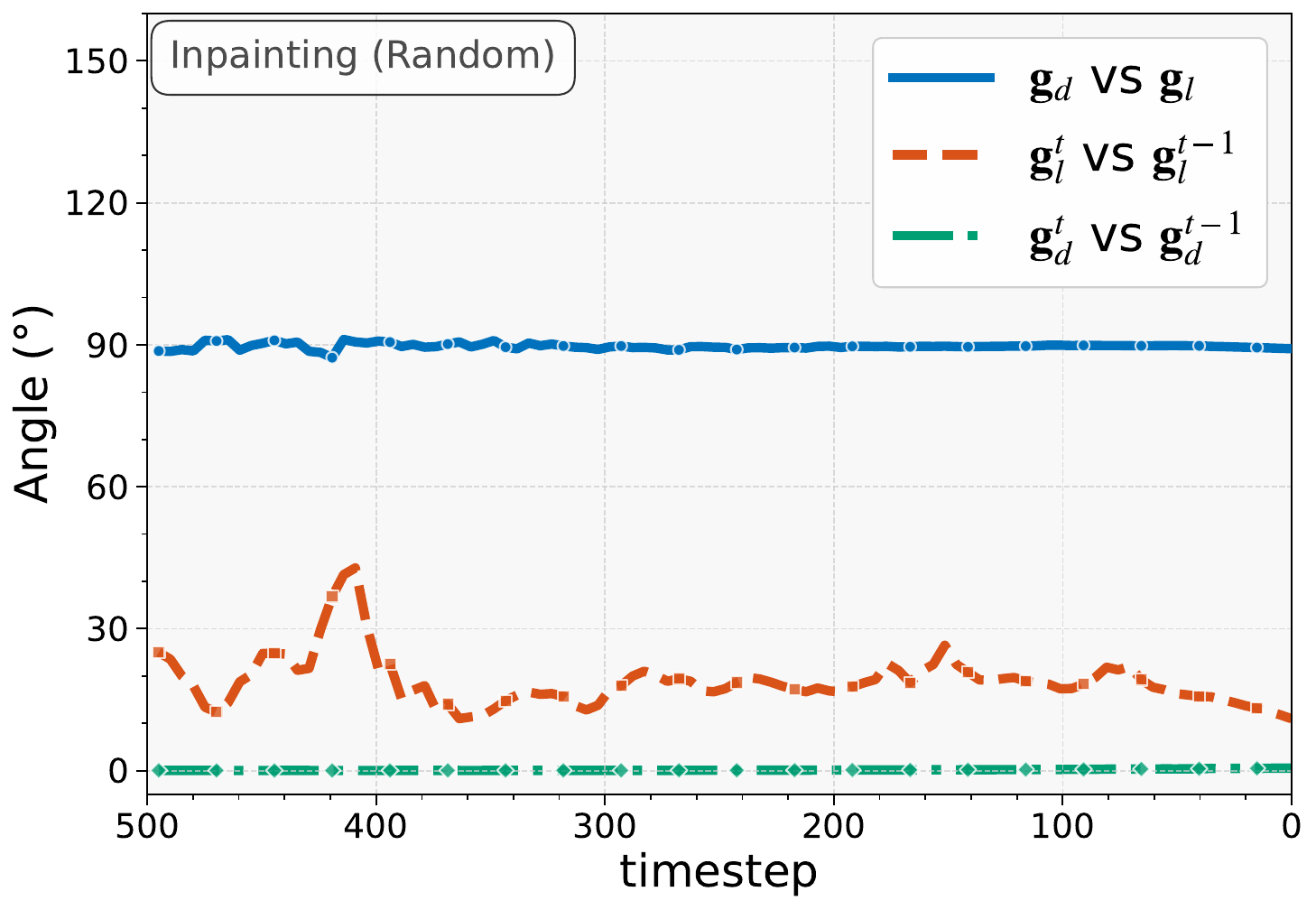}
        \caption{Sample 4}
    \end{subfigure}
        \begin{subfigure}{0.45\textwidth}
        \centering
        \includegraphics[width=0.49\linewidth]{before/Inpainting_Random_Column_5.pdf}
        \includegraphics[width=0.49\linewidth]{after/Inpainting_Random_Column_5.pdf}
        \caption{Sample 5}
    \end{subfigure}
        \begin{subfigure}{0.45\textwidth}
        \centering
        \includegraphics[width=0.49\linewidth]{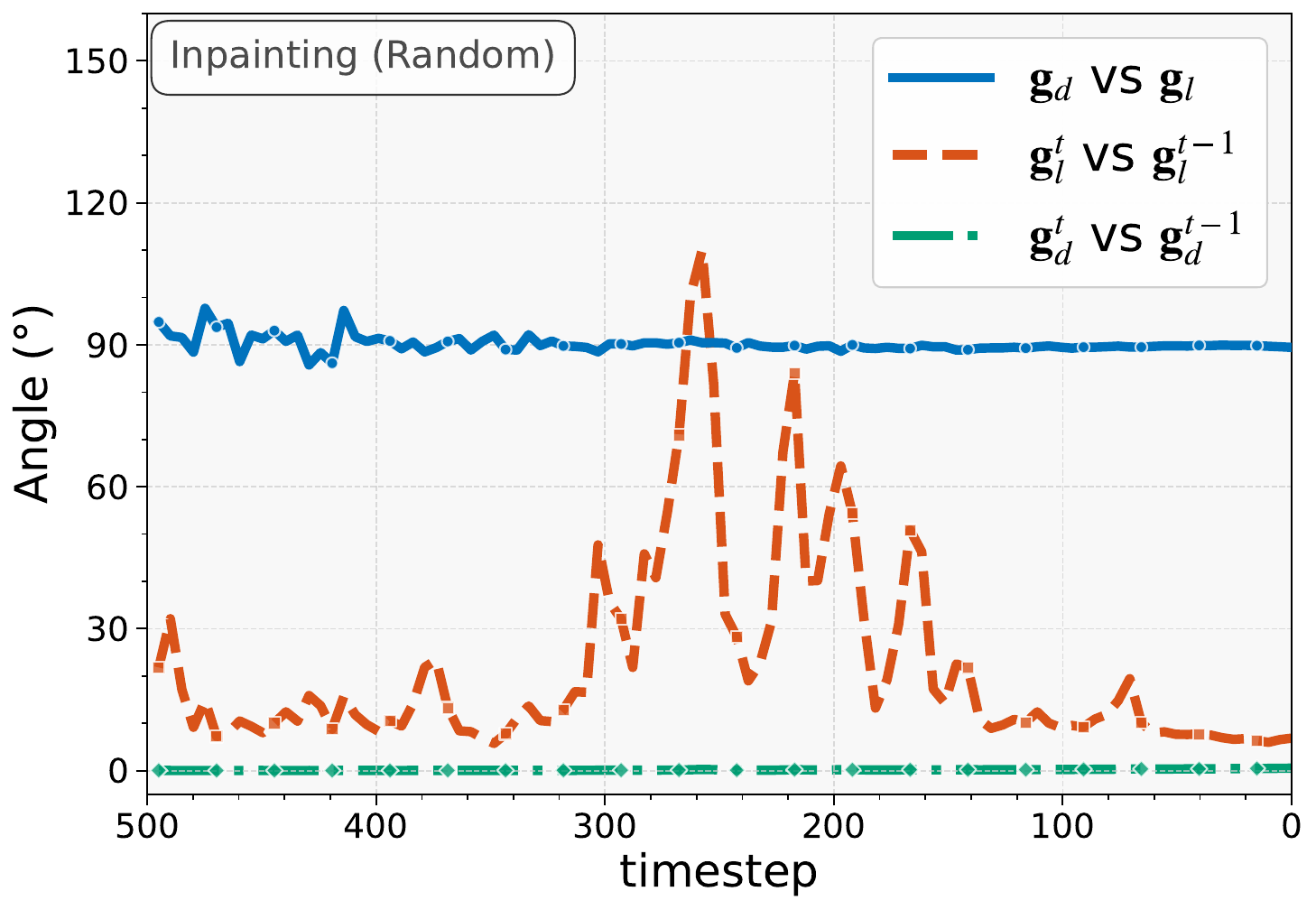}
        \includegraphics[width=0.49\linewidth]{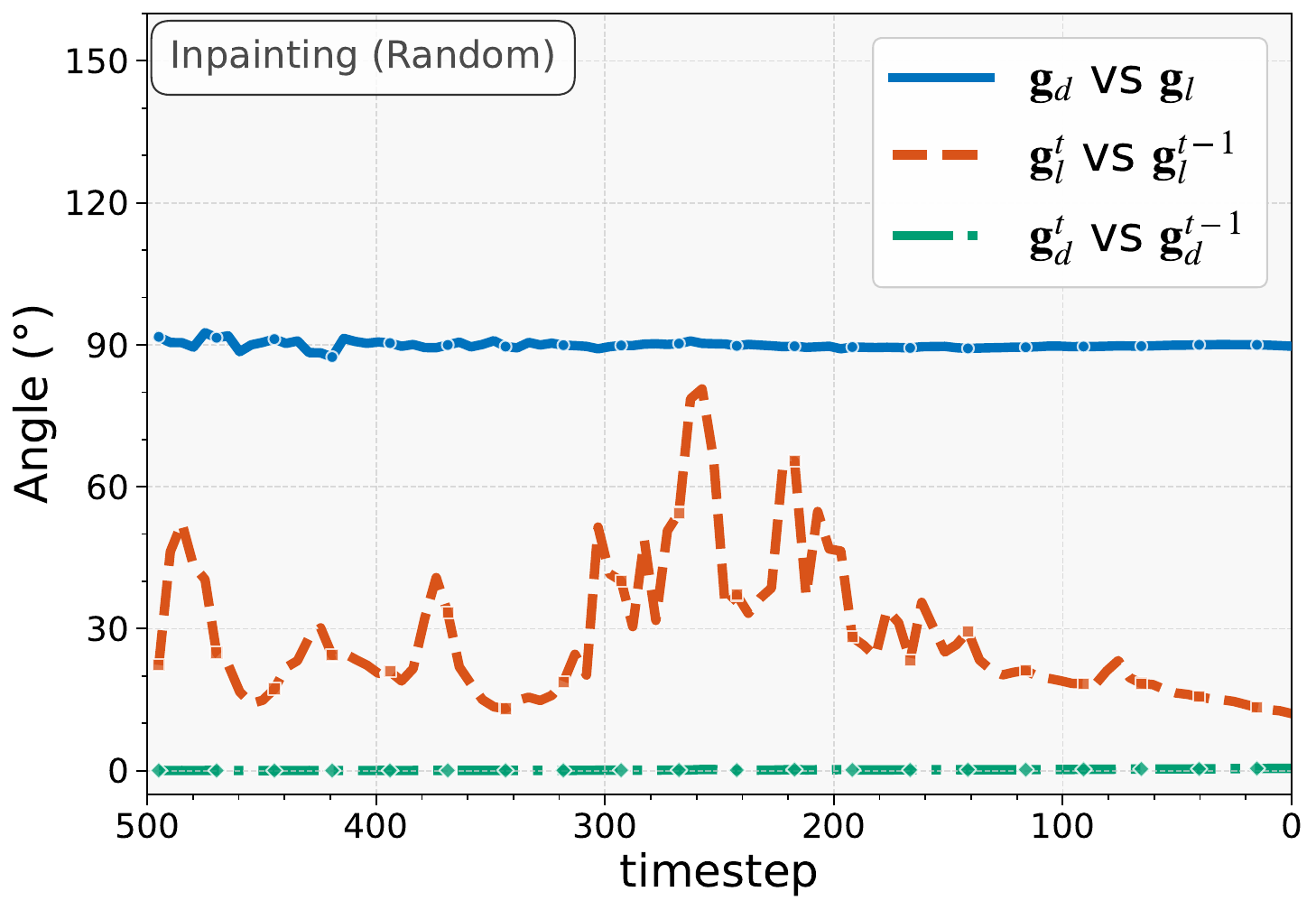}
        \caption{Sample 6}
    \end{subfigure}
    \caption{Gradient angles dynamics before (left) and after ADM smoothing (right) on random inpainting.}
    \label{fig:last1}
\end{figure*}

\begin{figure*}[htbp]
    \centering
    \begin{subfigure}{0.45\textwidth}
        \centering
        \includegraphics[width=0.49\linewidth]{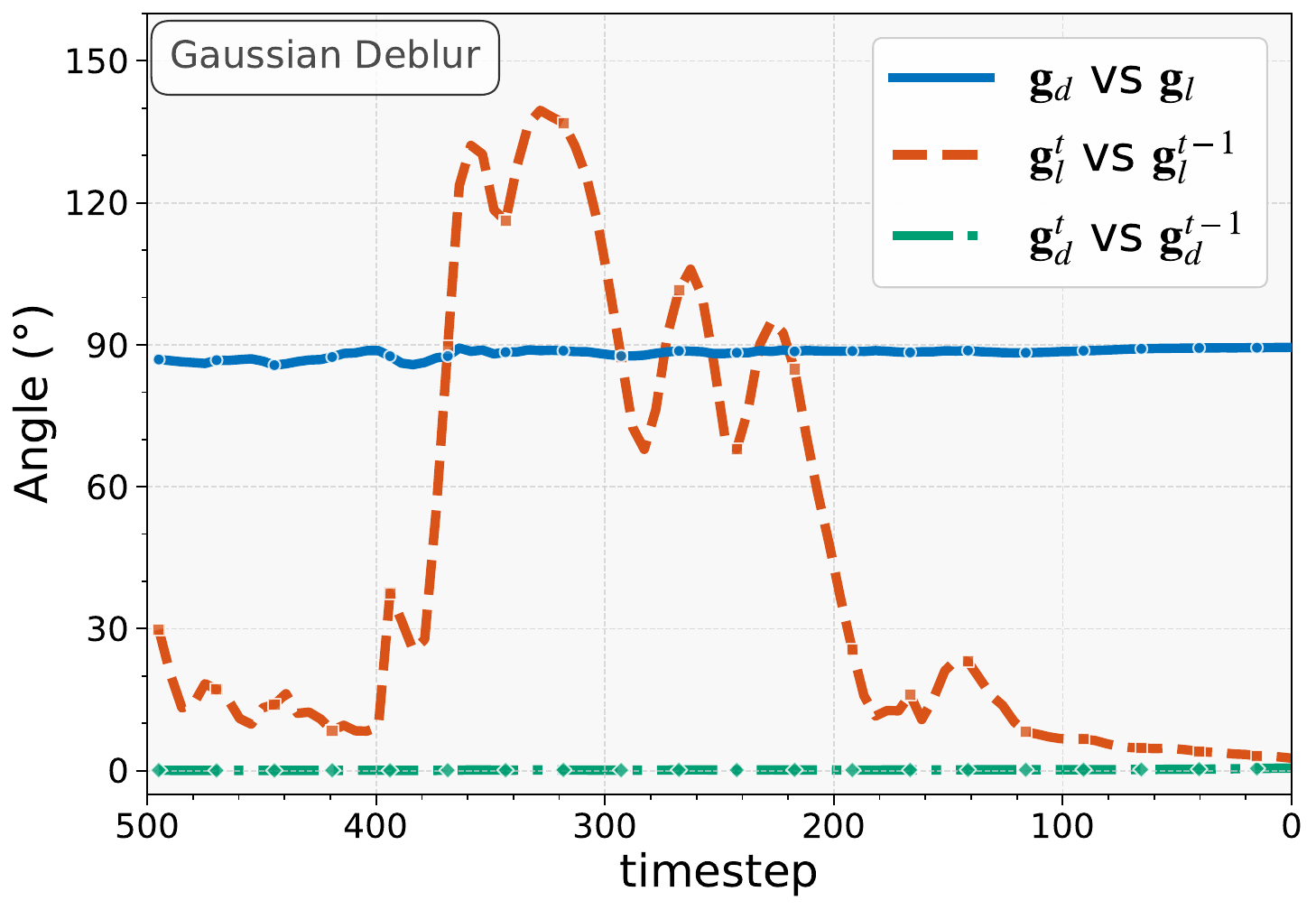}
        \includegraphics[width=0.49\linewidth]{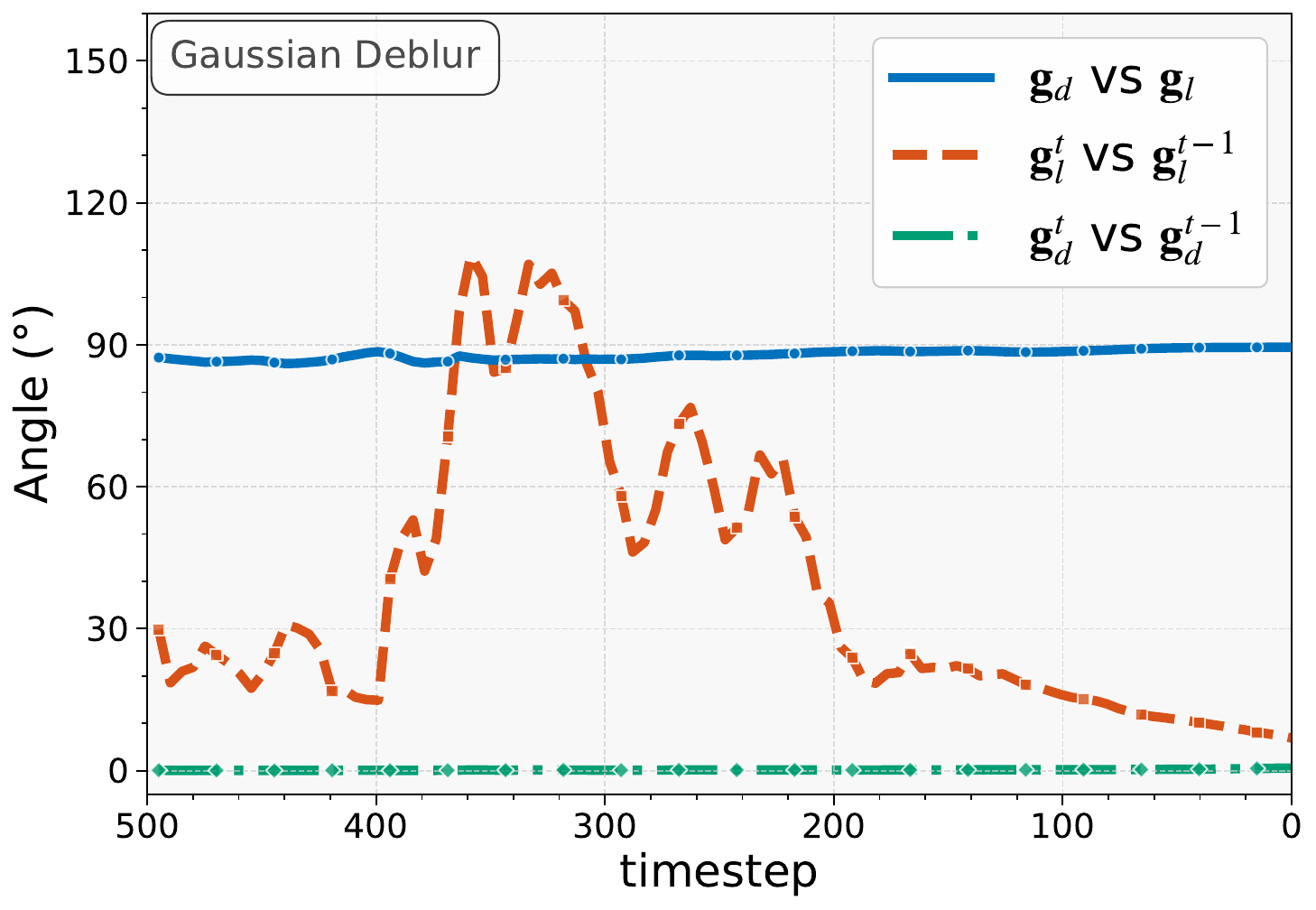}
        \caption{Sample 1}
    \end{subfigure}
    \begin{subfigure}{0.45\textwidth}
        \centering
        \includegraphics[width=0.49\linewidth]{before/Gaussian_Deblur_Column_2.pdf}
        \includegraphics[width=0.49\linewidth]{after/Gaussian_Deblur_Column_2.pdf}
        \caption{Sample 2}
    \end{subfigure}
        \begin{subfigure}{0.45\textwidth}
        \centering
        \includegraphics[width=0.49\linewidth]{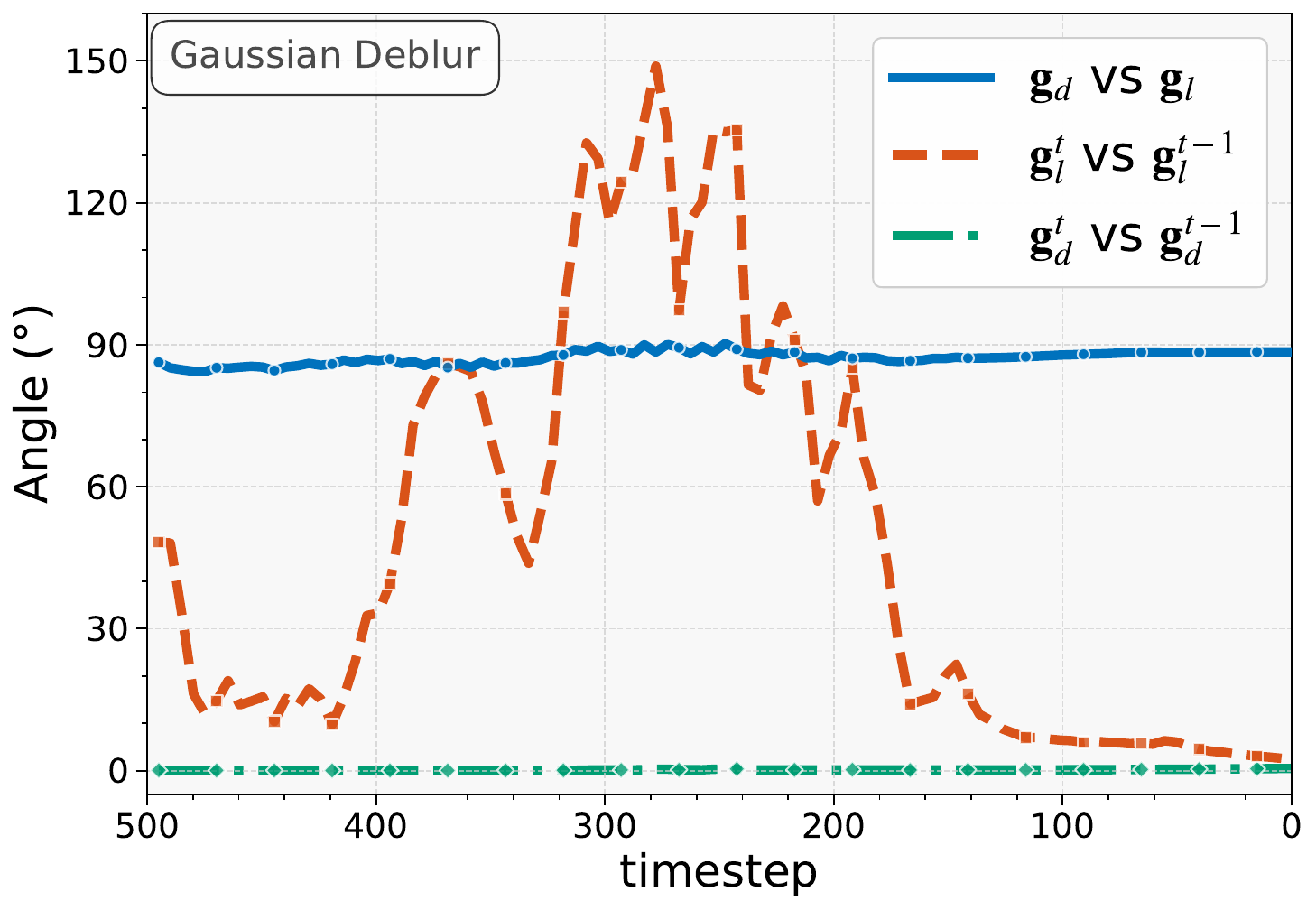}
        \includegraphics[width=0.49\linewidth]{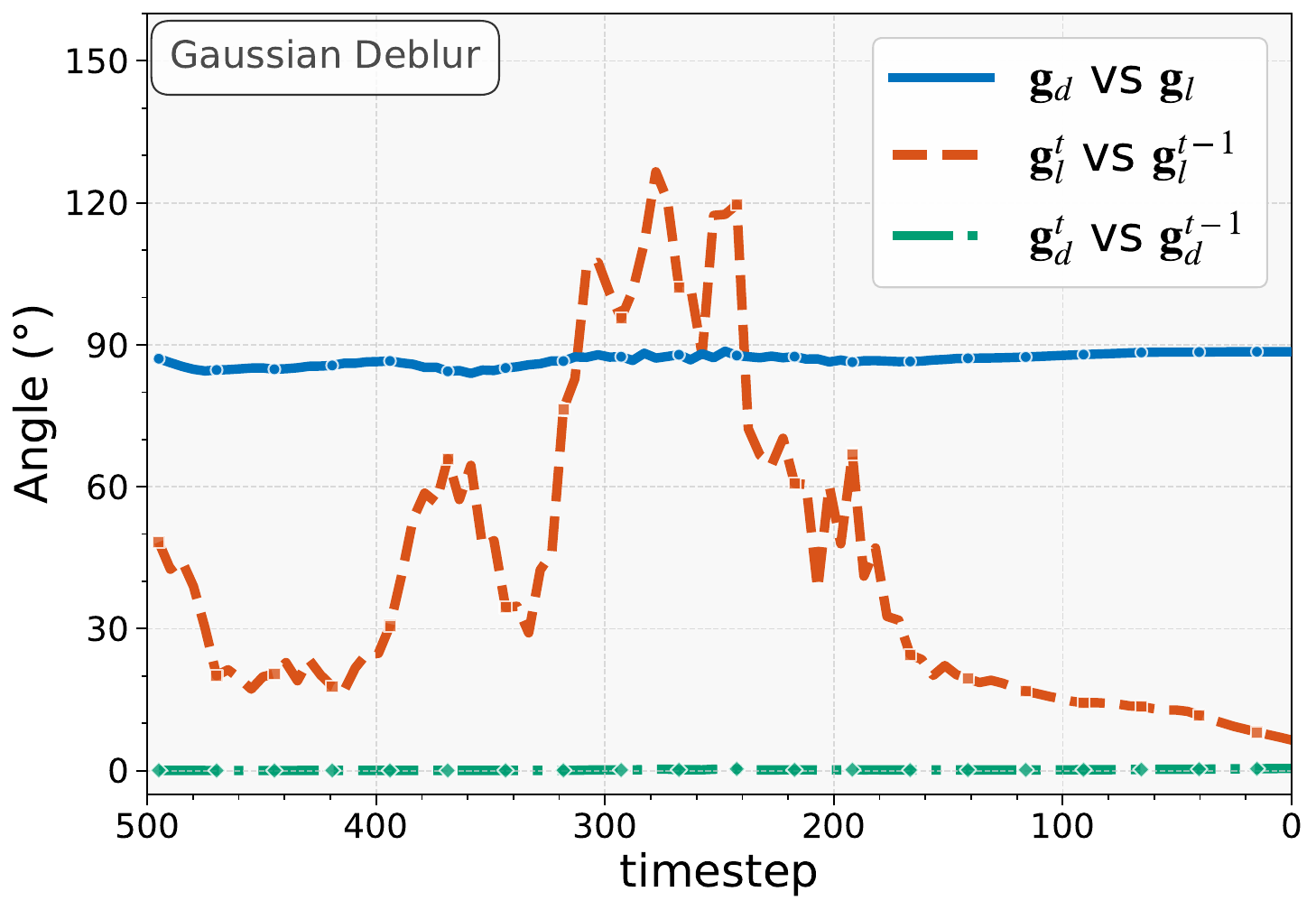}
        \caption{Sample 3}
    \end{subfigure}
    \begin{subfigure}{0.45\textwidth}
        \centering
        \includegraphics[width=0.49\linewidth]{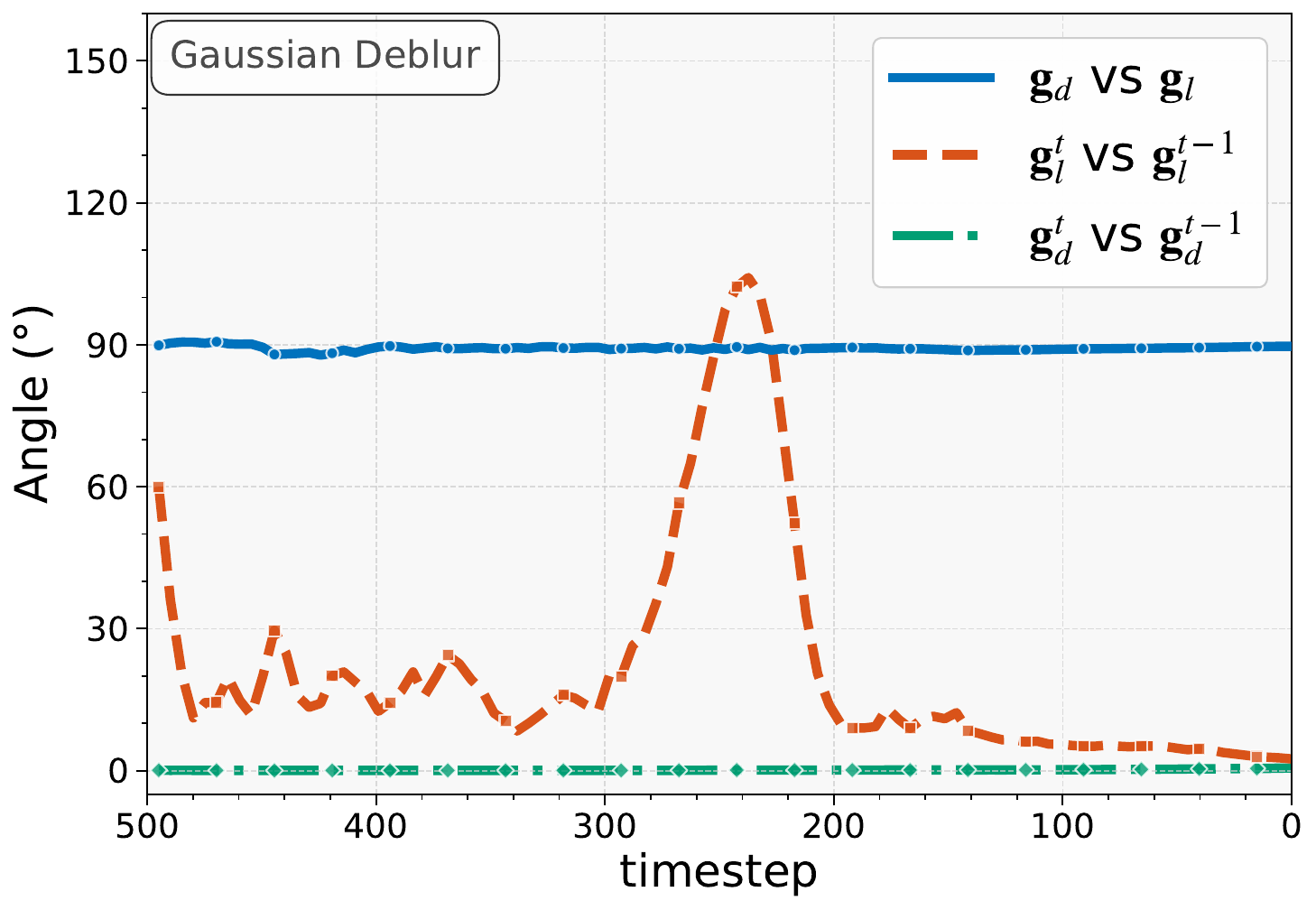}
        \includegraphics[width=0.49\linewidth]{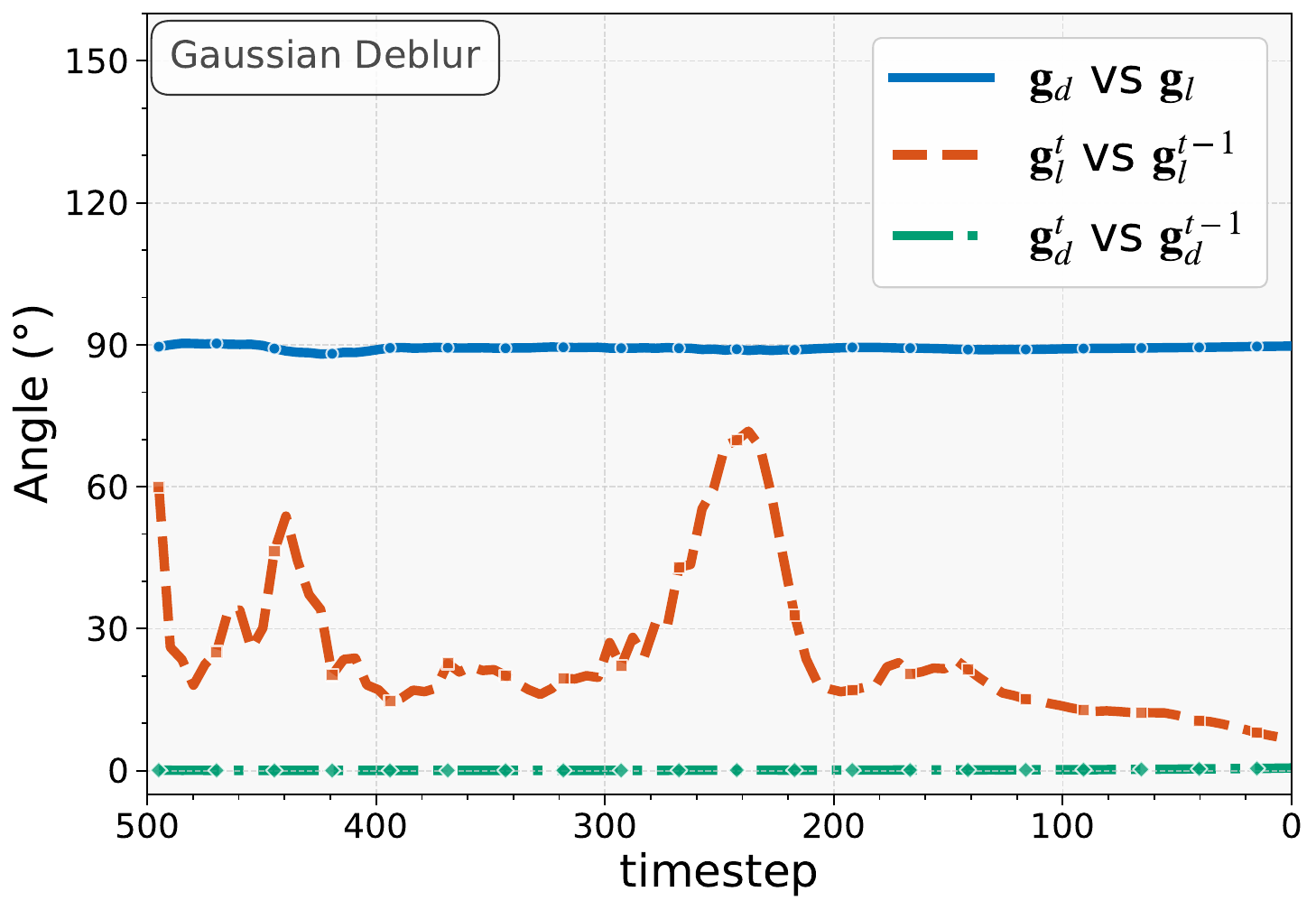}
        \caption{Sample 4}
    \end{subfigure}
        \begin{subfigure}{0.45\textwidth}
        \centering
        \includegraphics[width=0.49\linewidth]{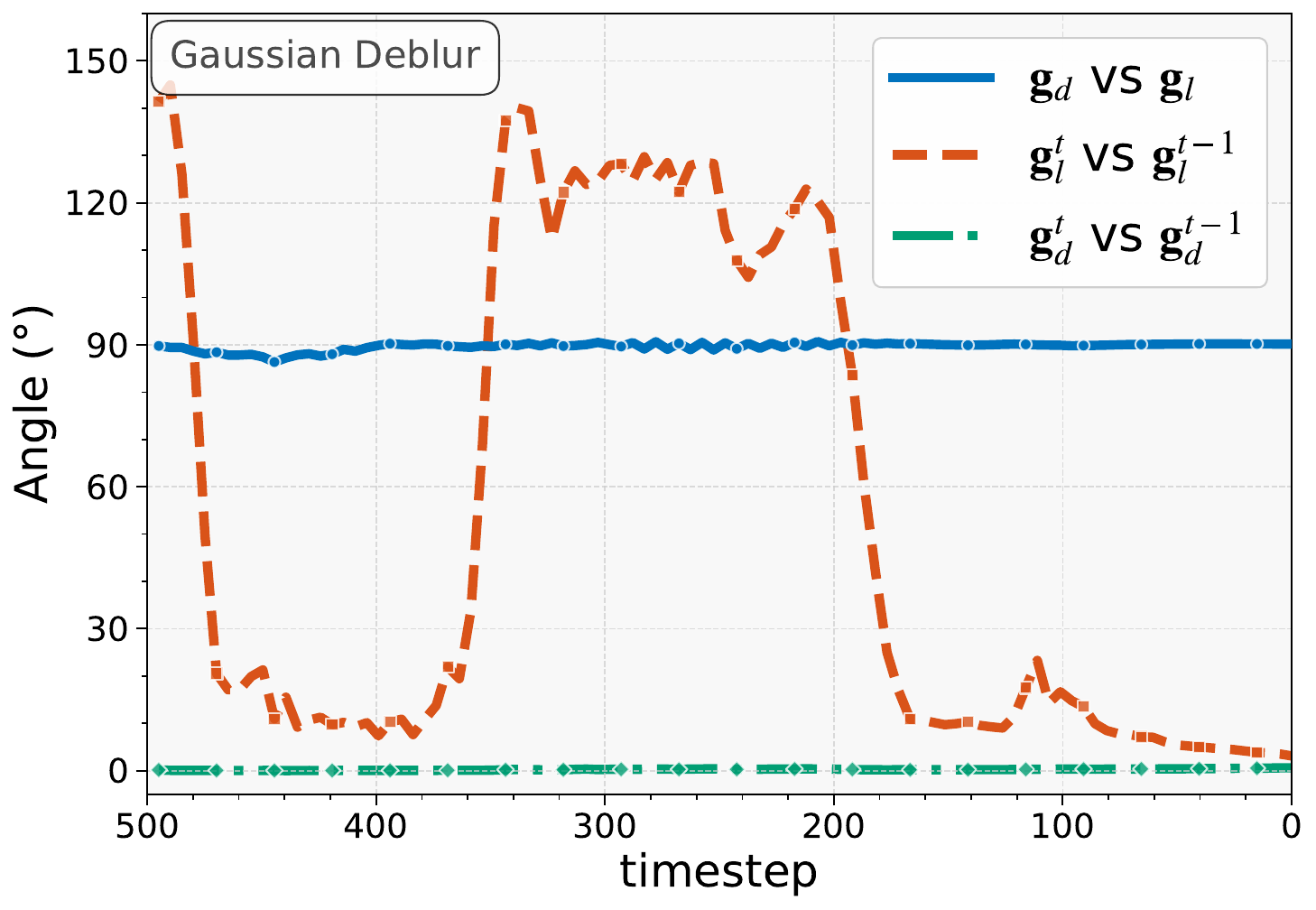}
        \includegraphics[width=0.49\linewidth]{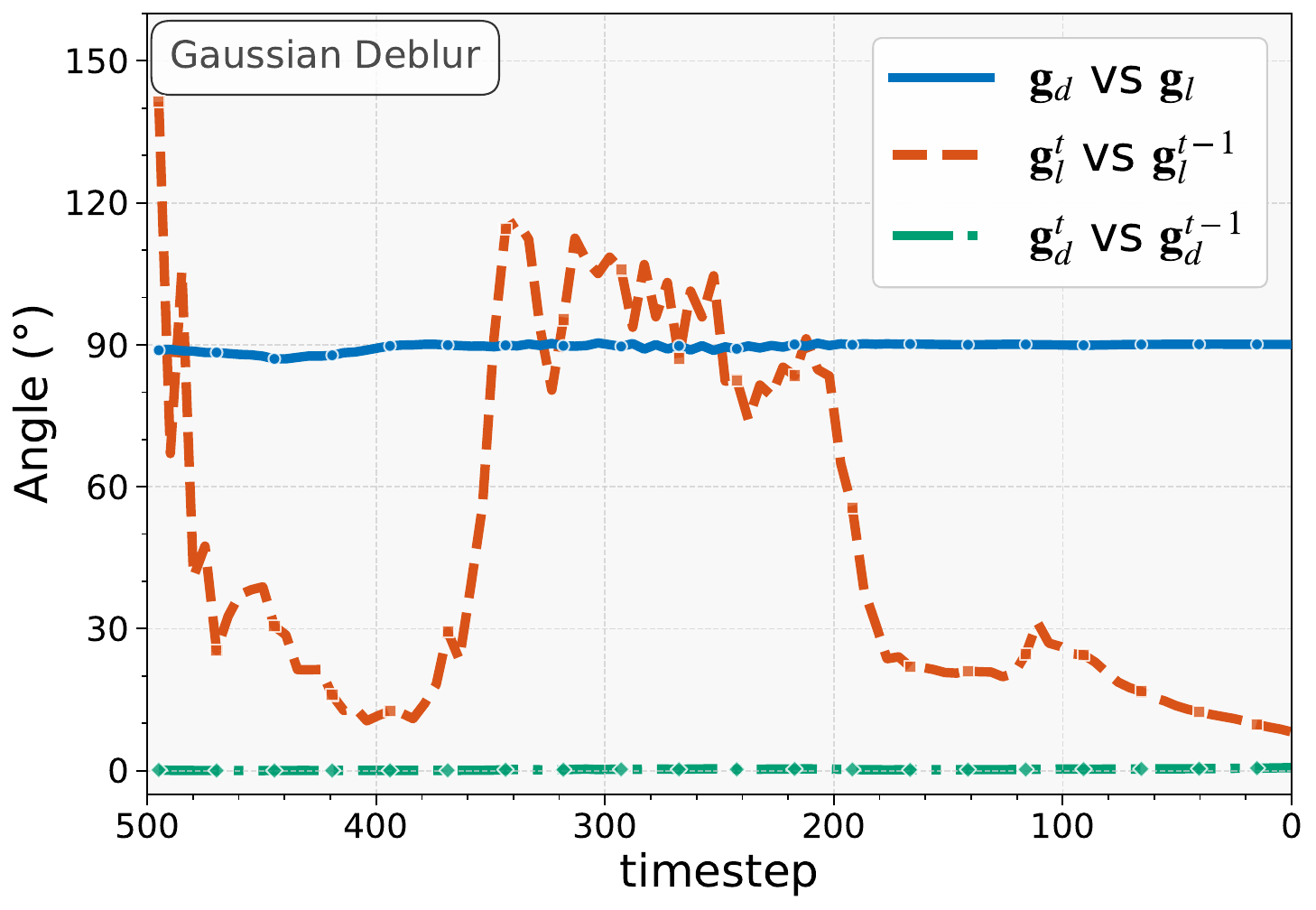}
        \caption{Sample 5}
    \end{subfigure}
        \begin{subfigure}{0.45\textwidth}
        \centering
        \includegraphics[width=0.49\linewidth]{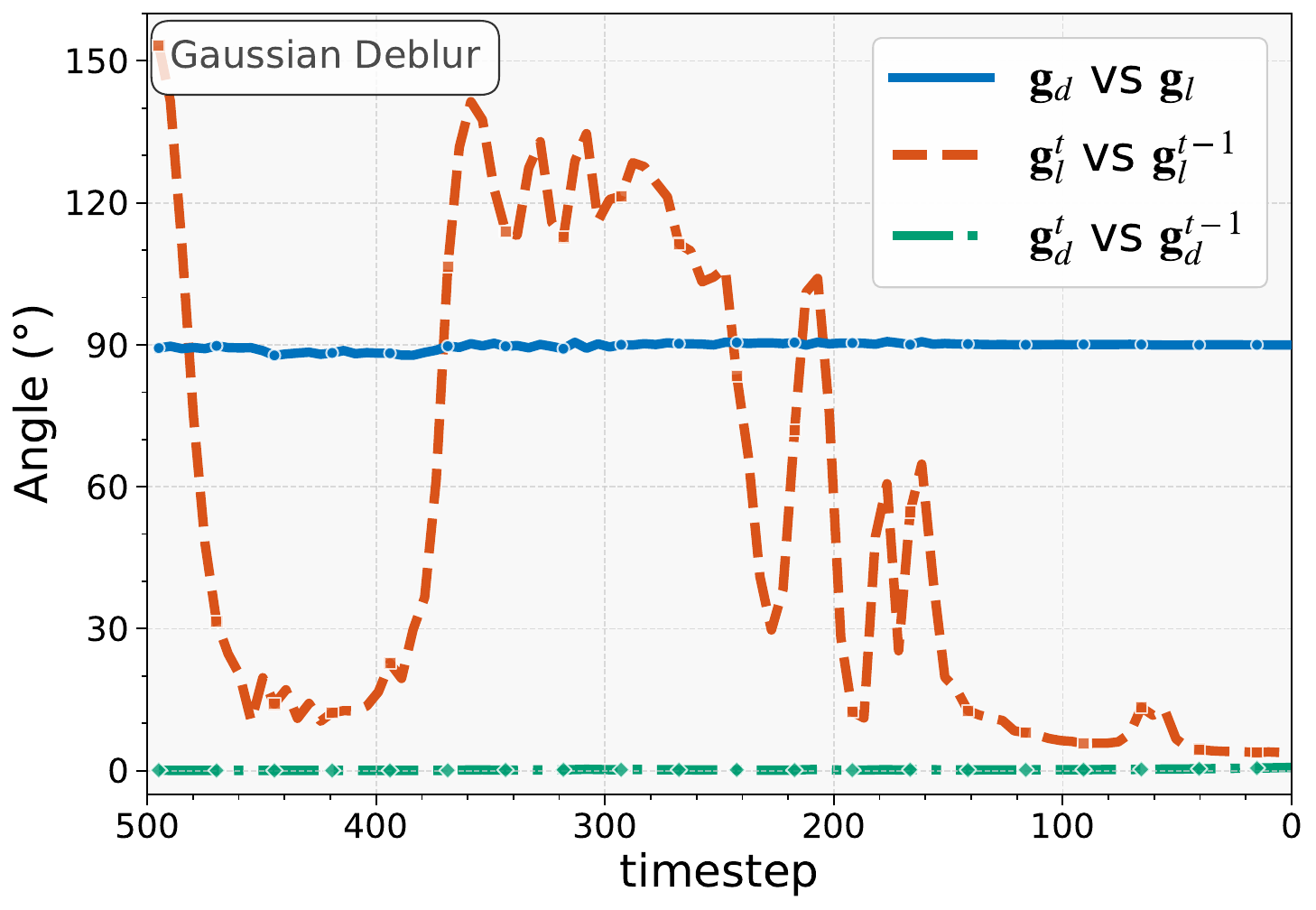}
        \includegraphics[width=0.49\linewidth]{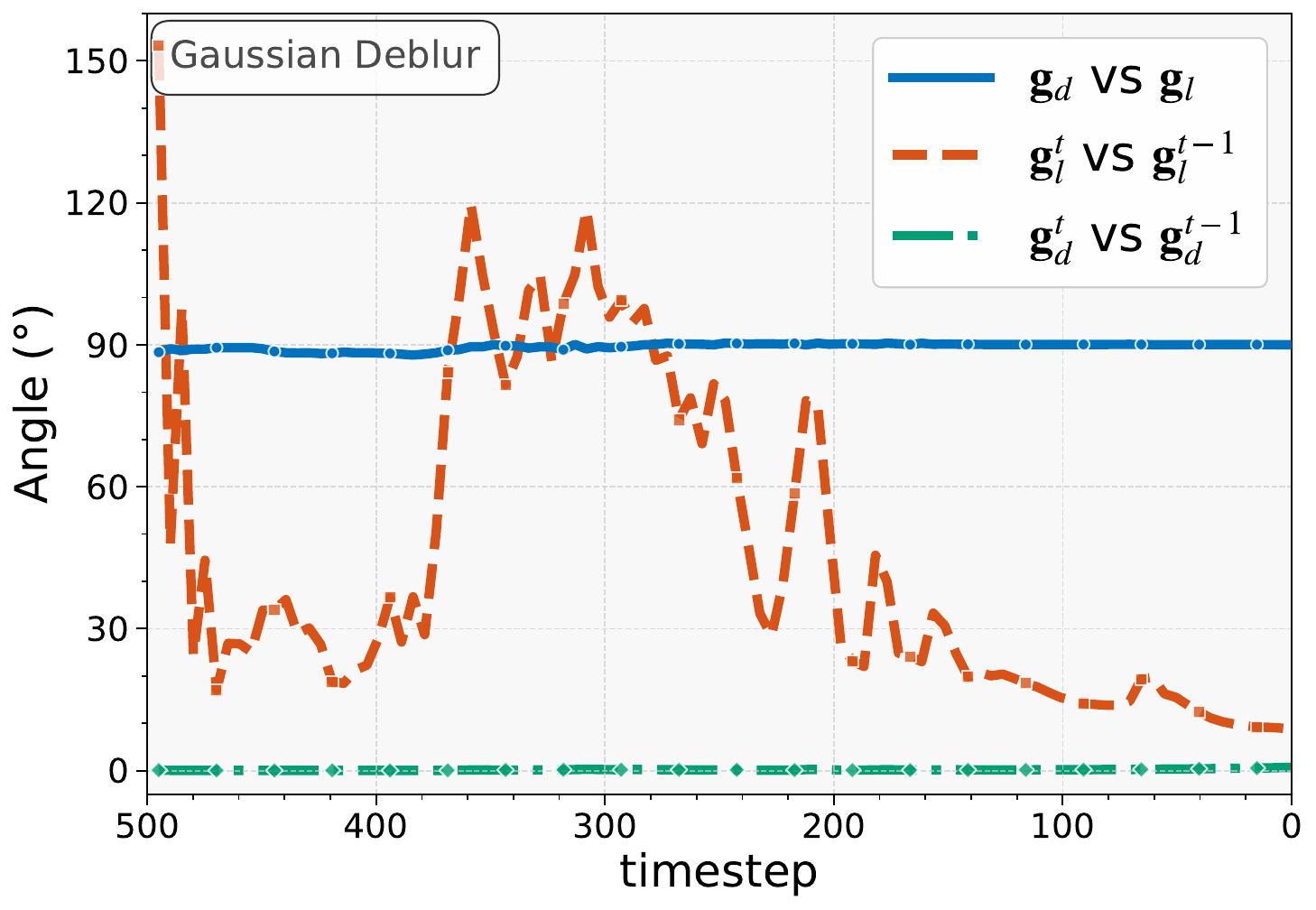}
        \caption{Sample 6}
    \end{subfigure}
    \caption{Gradient angles dynamics before (left) and after ADM smoothing (right) on Gaussian deblurring.}
\end{figure*}

\begin{figure*}[htbp]
    \centering
    \begin{subfigure}{0.45\textwidth}
        \centering
        \includegraphics[width=0.49\linewidth]{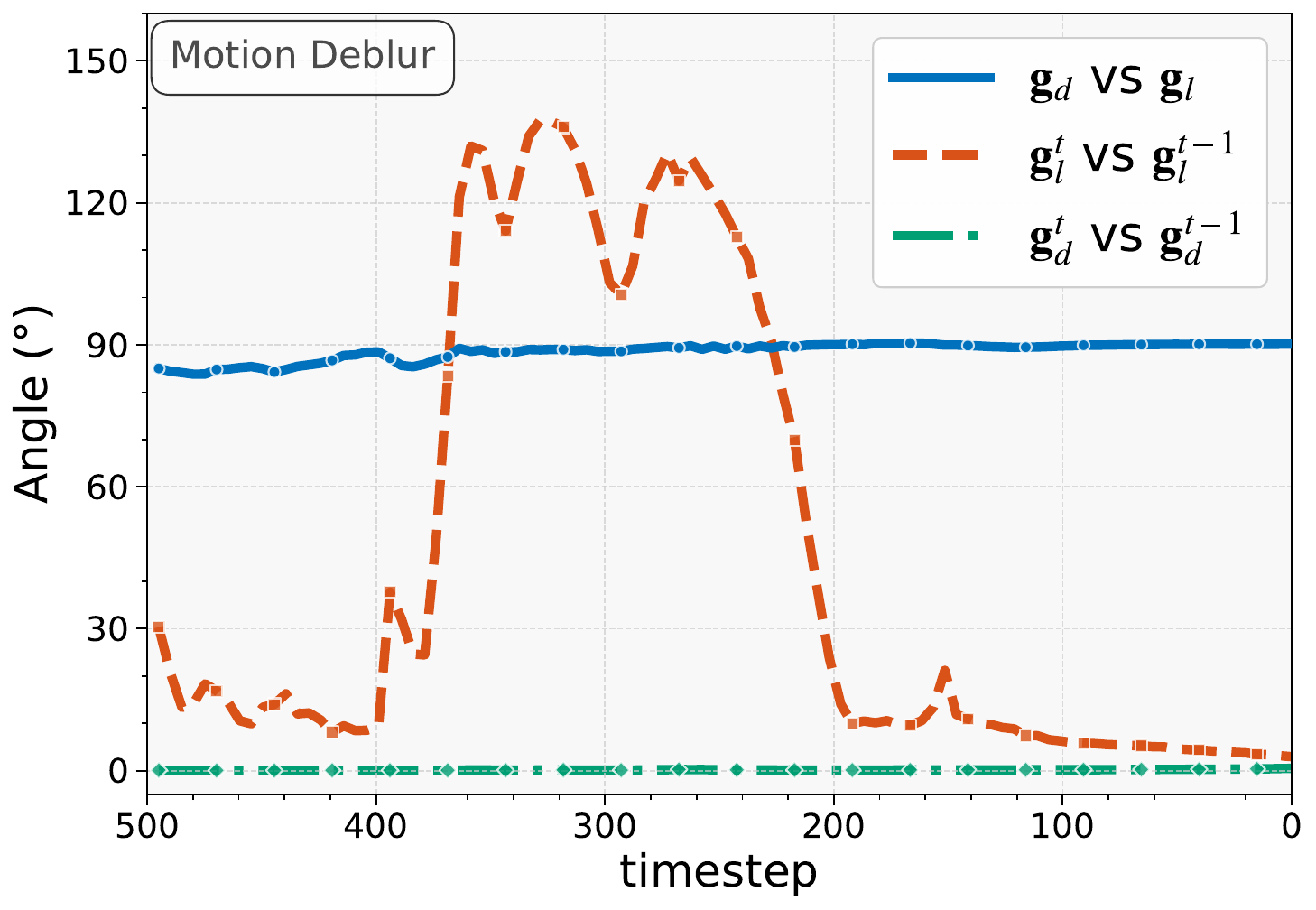}
        \includegraphics[width=0.49\linewidth]{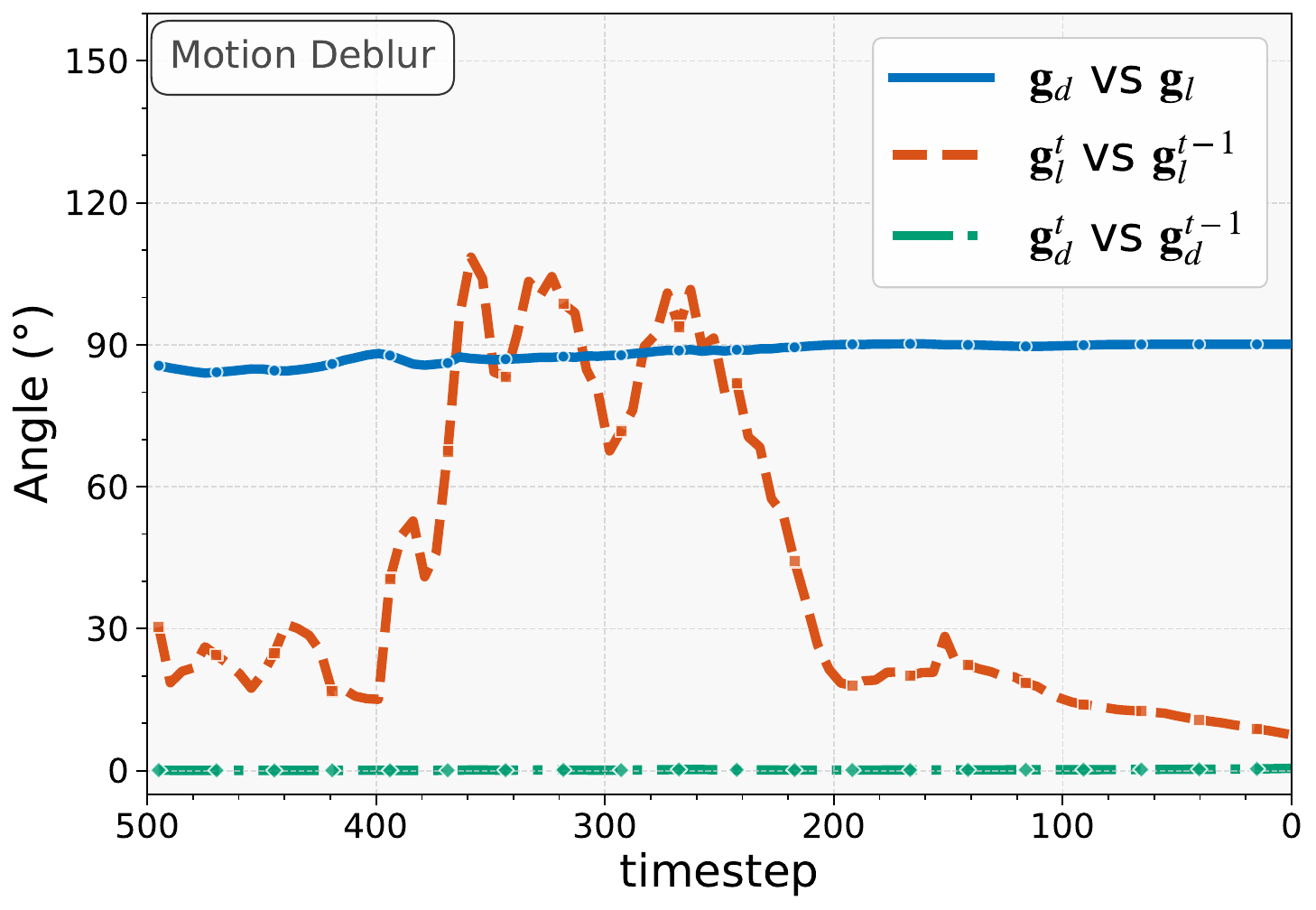}
        \caption{Sample 1}
    \end{subfigure}
    \begin{subfigure}{0.45\textwidth}
        \centering
        \includegraphics[width=0.49\linewidth]{before/Motion_Deblur_Column_2.pdf}
        \includegraphics[width=0.49\linewidth]{after/Motion_Deblur_Column_2.pdf}
        \caption{Sample 2}
    \end{subfigure}
        \begin{subfigure}{0.45\textwidth}
        \centering
        \includegraphics[width=0.49\linewidth]{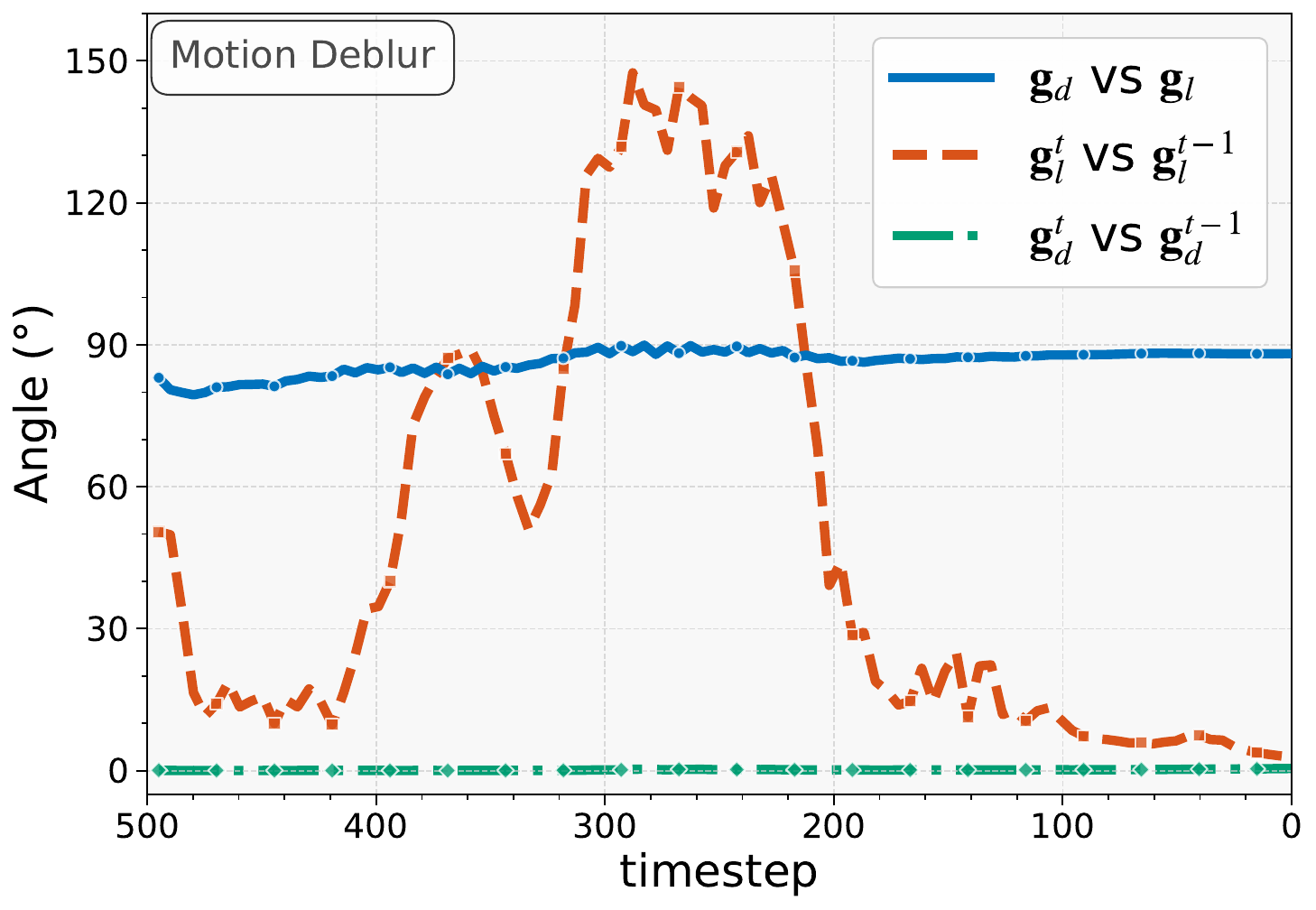}
        \includegraphics[width=0.49\linewidth]{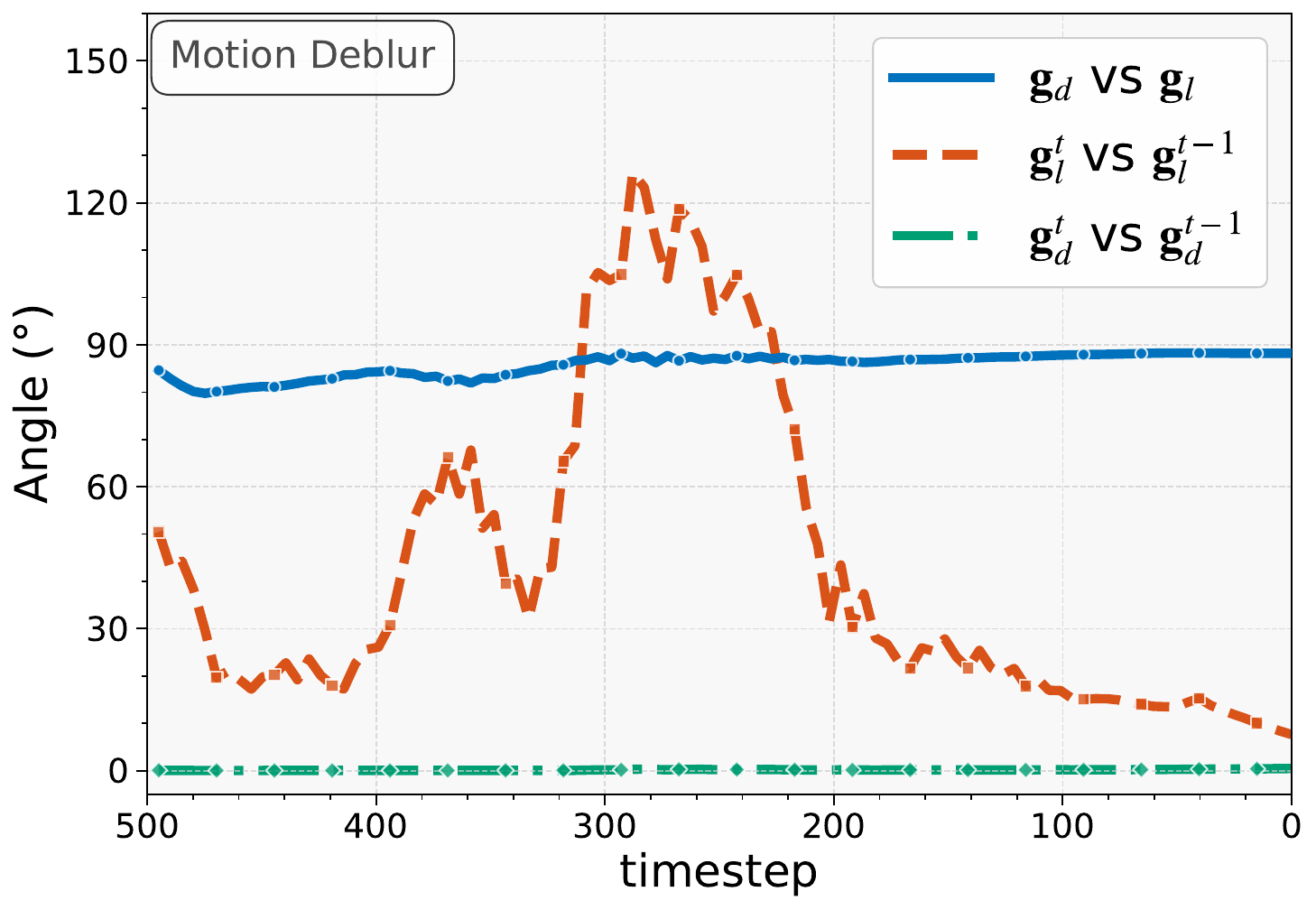}
        \caption{Sample 3}
    \end{subfigure}
    \begin{subfigure}{0.45\textwidth}
        \centering
        \includegraphics[width=0.49\linewidth]{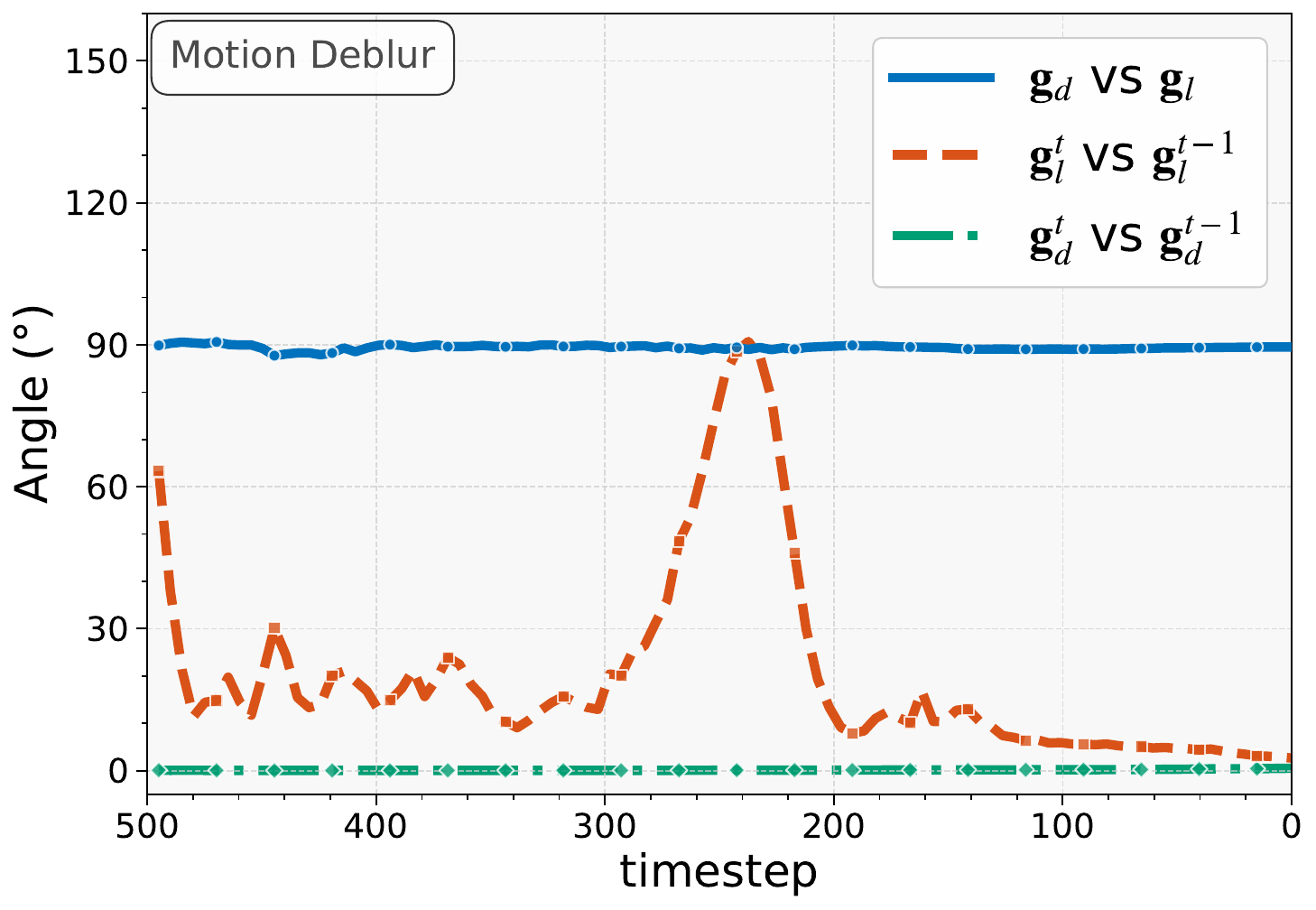}
        \includegraphics[width=0.49\linewidth]{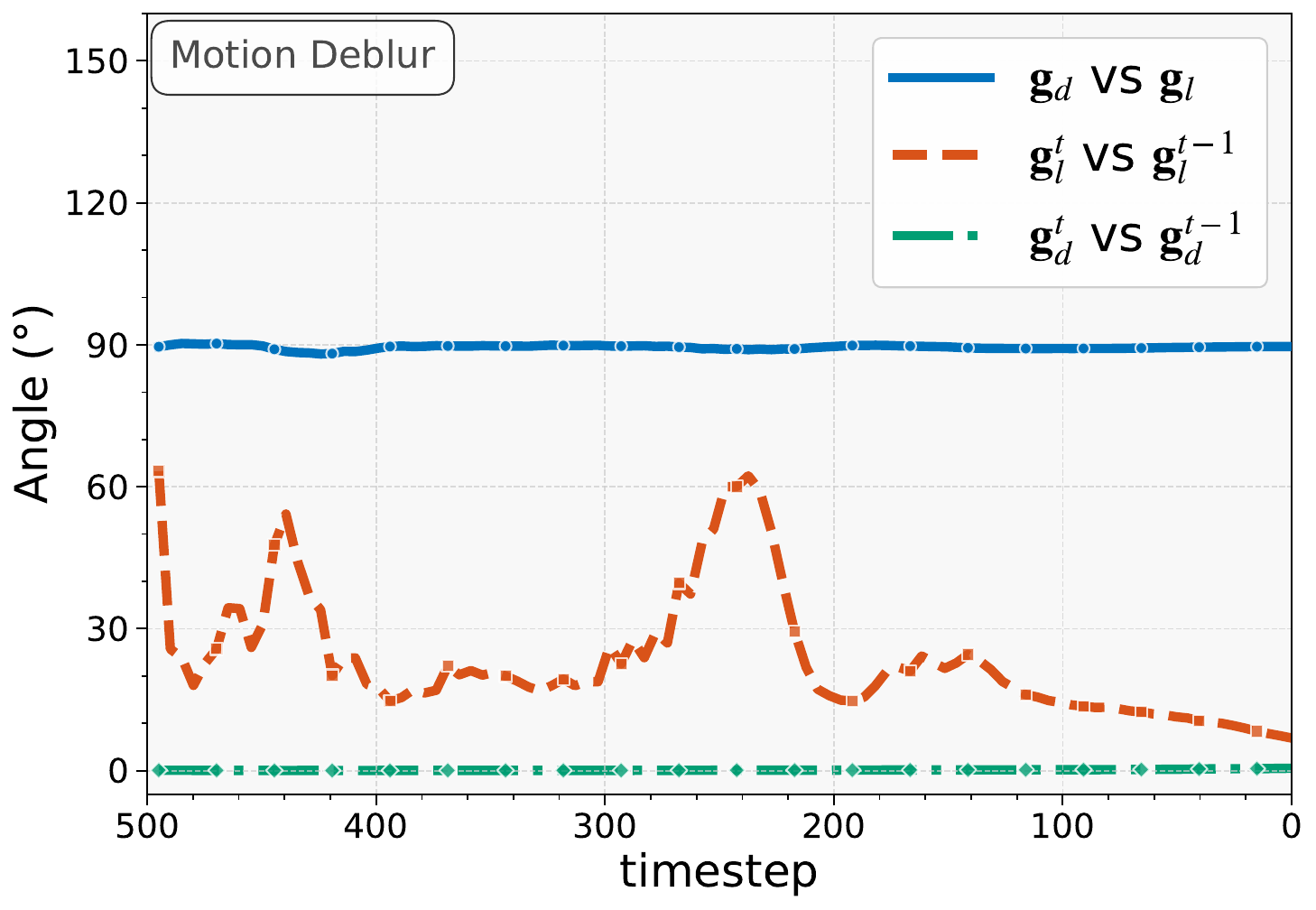}
        \caption{Sample 4}
    \end{subfigure}
        \begin{subfigure}{0.45\textwidth}
        \centering
        \includegraphics[width=0.49\linewidth]{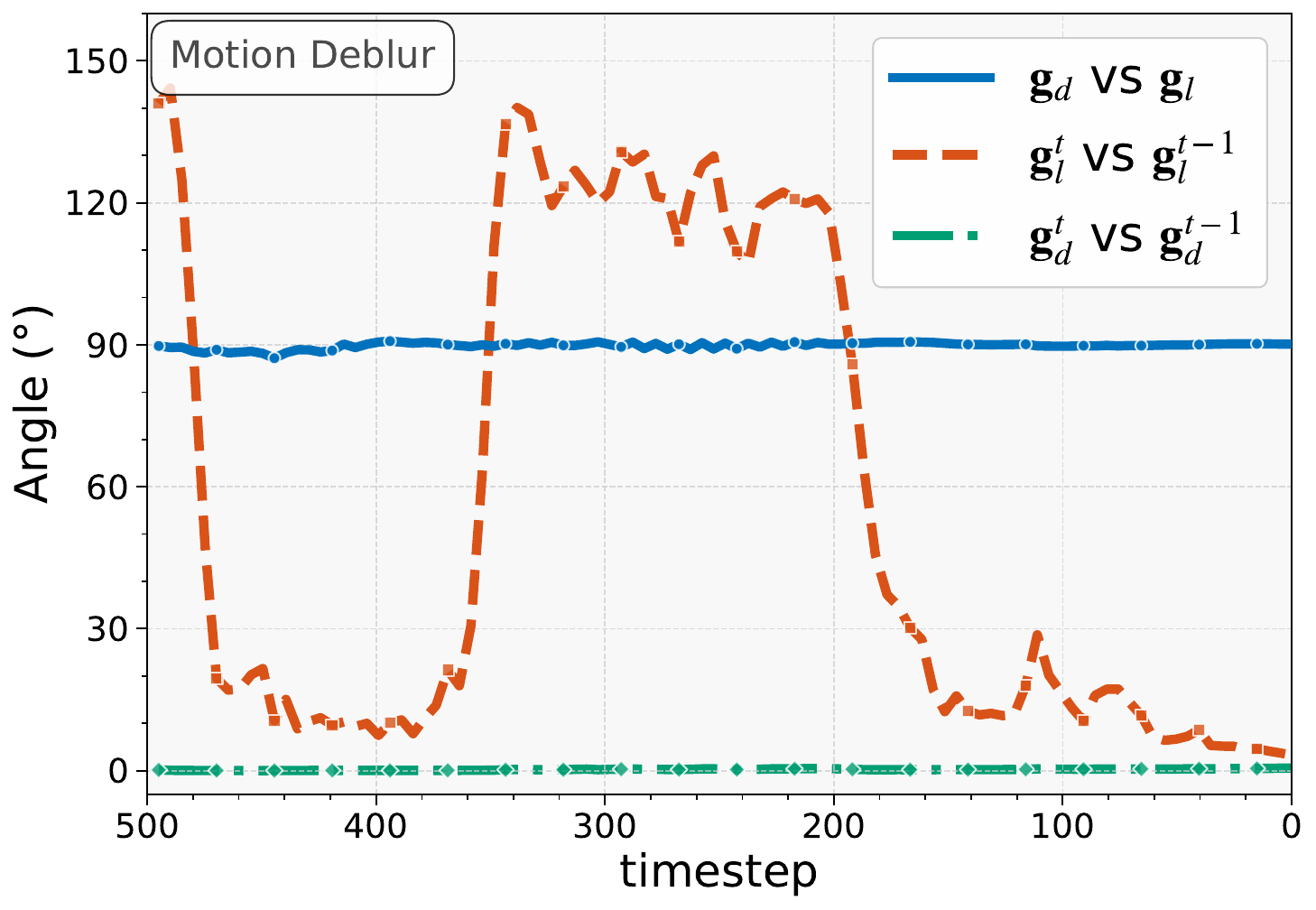}
        \includegraphics[width=0.49\linewidth]{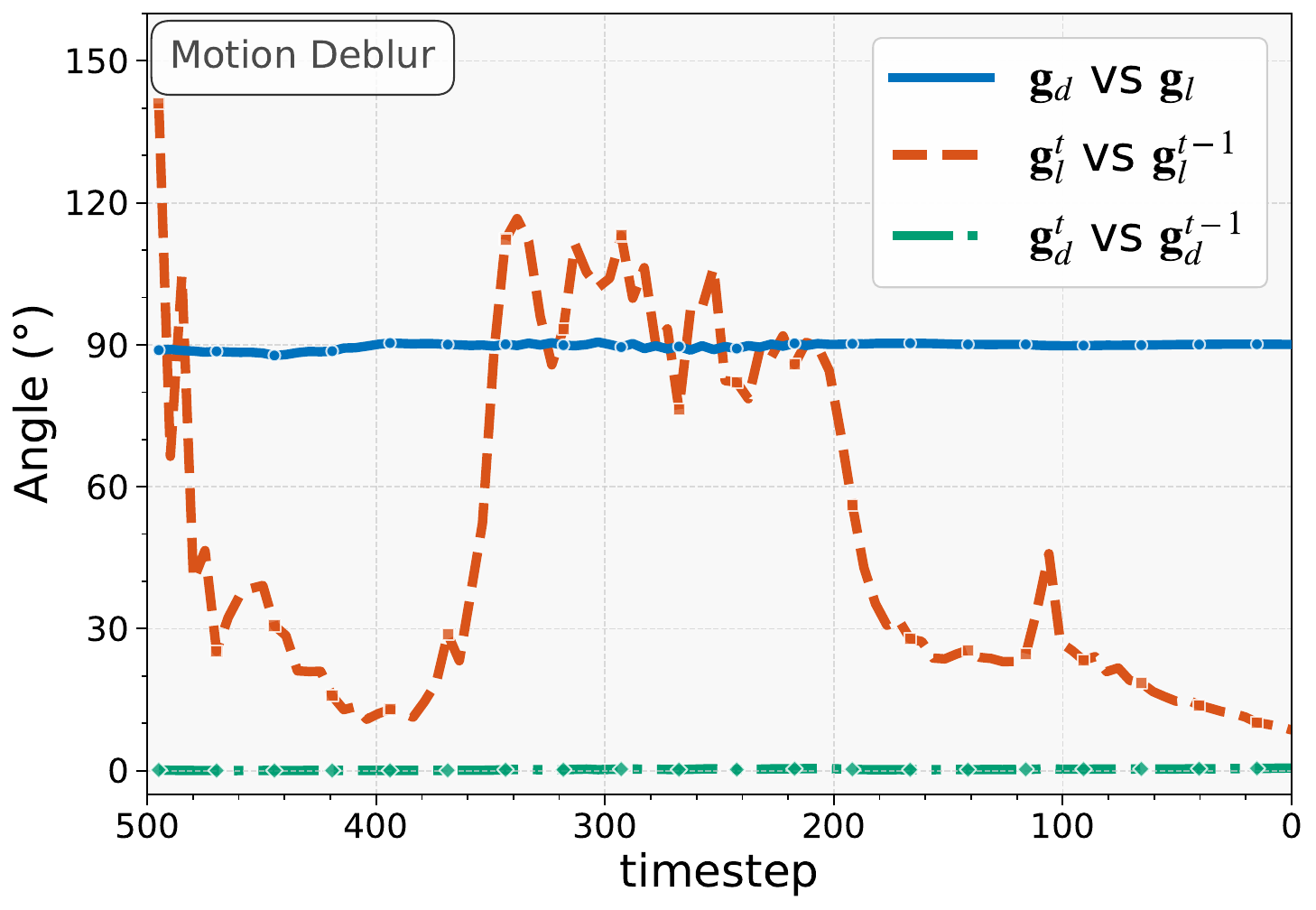}
        \caption{Sample 5}
    \end{subfigure}
        \begin{subfigure}{0.45\textwidth}
        \centering
        \includegraphics[width=0.49\linewidth]{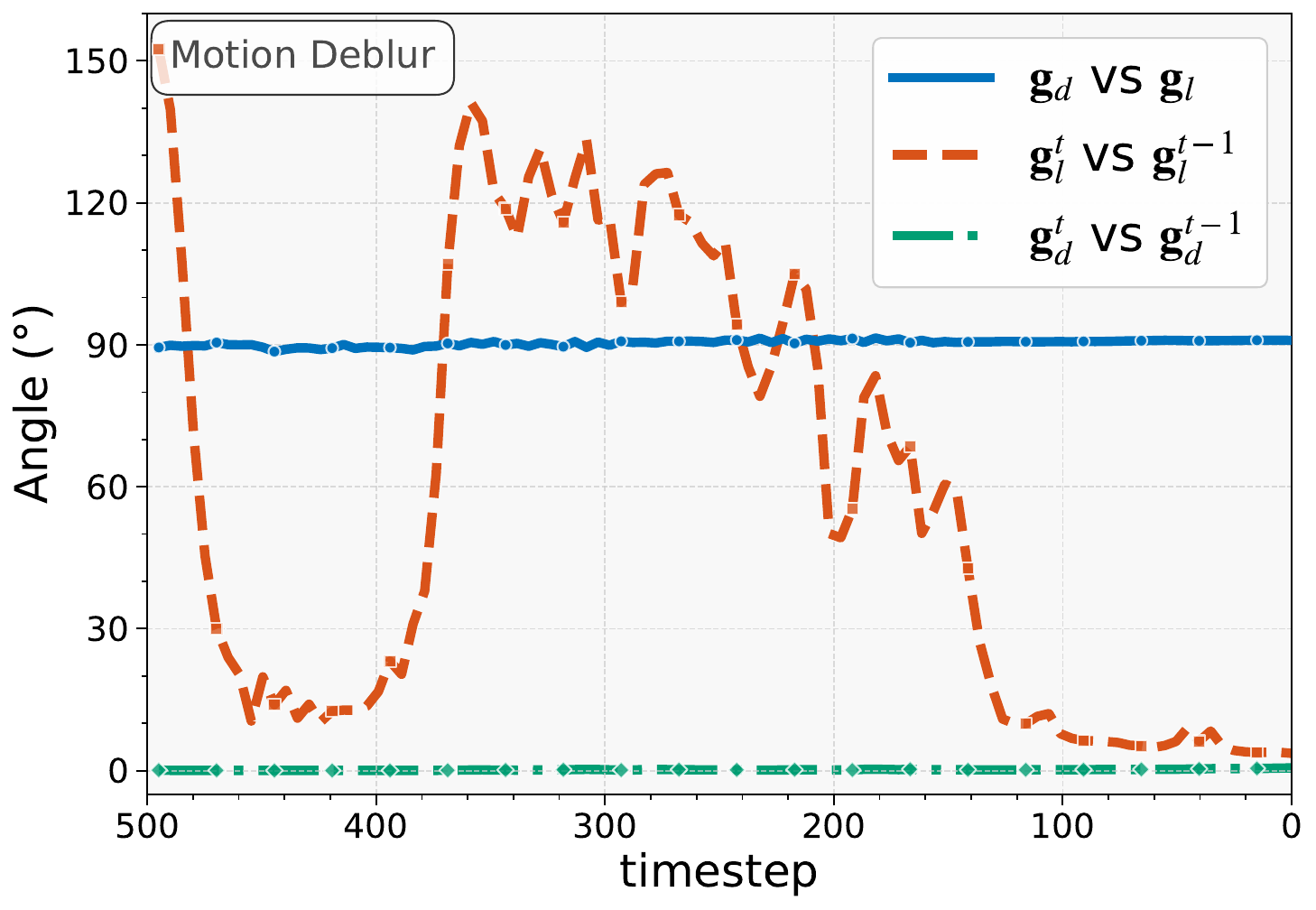}
        \includegraphics[width=0.49\linewidth]{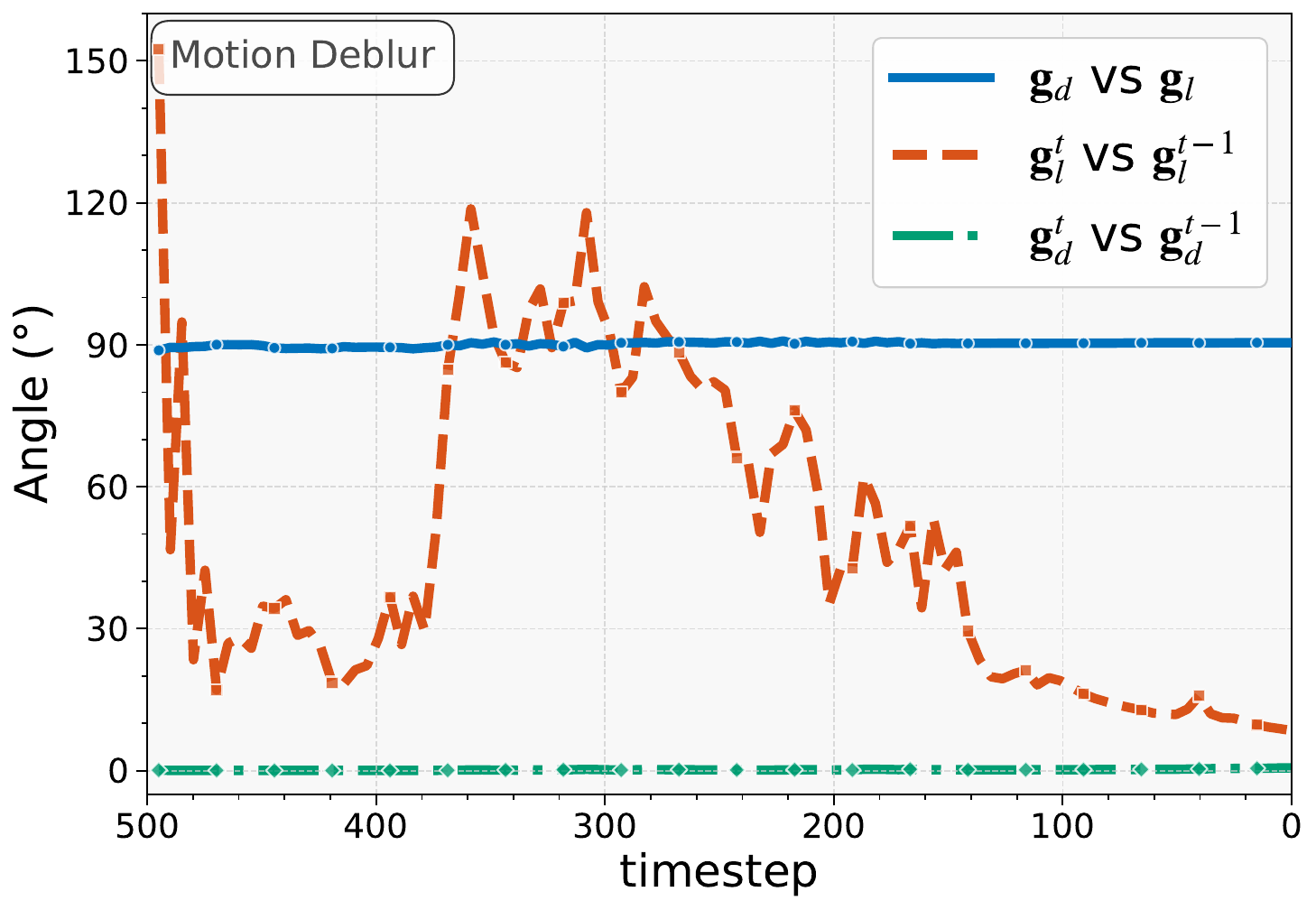}
        \caption{Sample 6}
    \end{subfigure}
    \caption{Gradient angles dynamics before (left) and after ADM smoothing (right) on motion deblurring.}
\end{figure*}

\begin{figure*}[htbp] 
    \centering
    \begin{subfigure}{0.45\textwidth}
        \centering
        \includegraphics[width=0.49\linewidth]{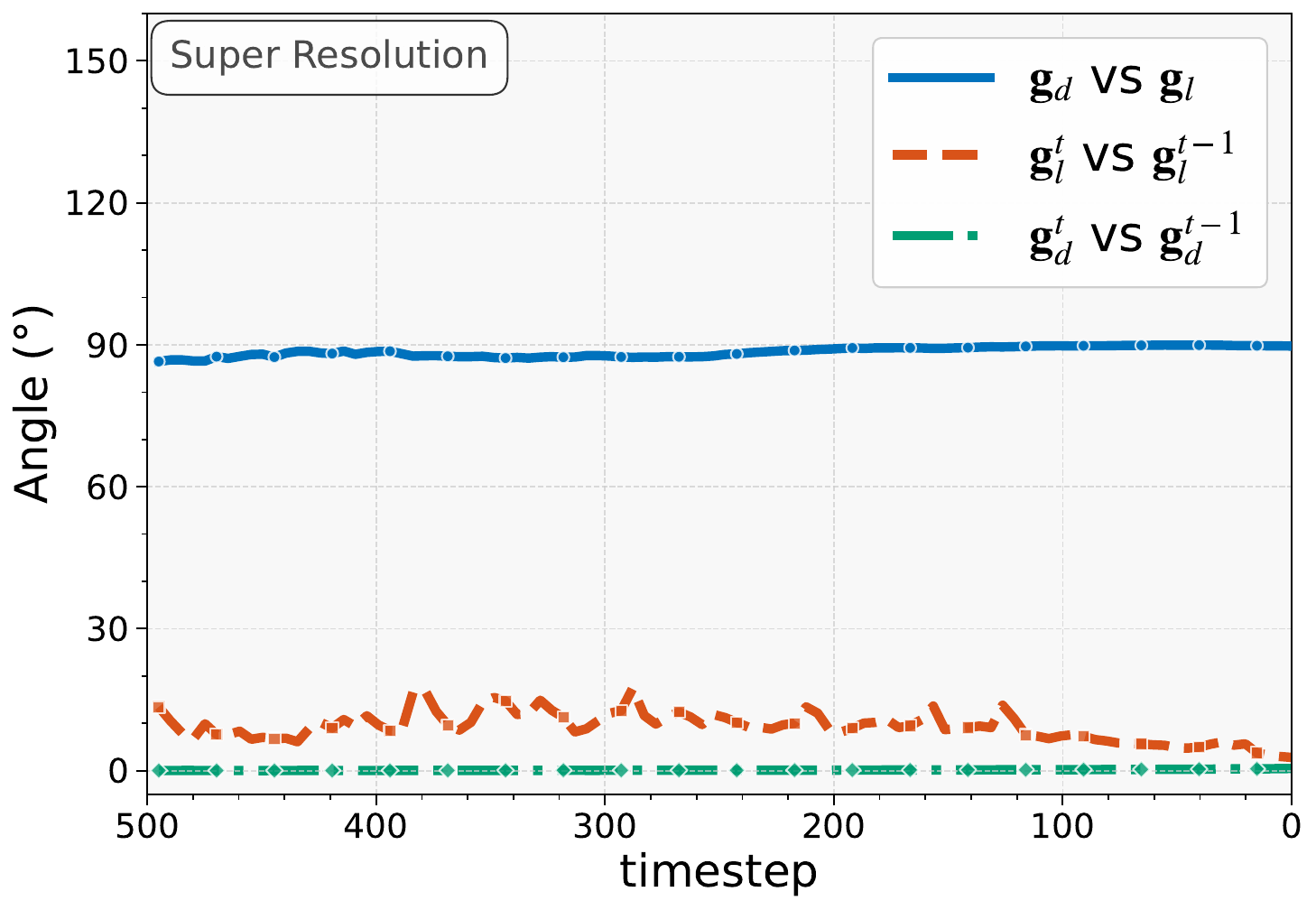}
        \includegraphics[width=0.49\linewidth]{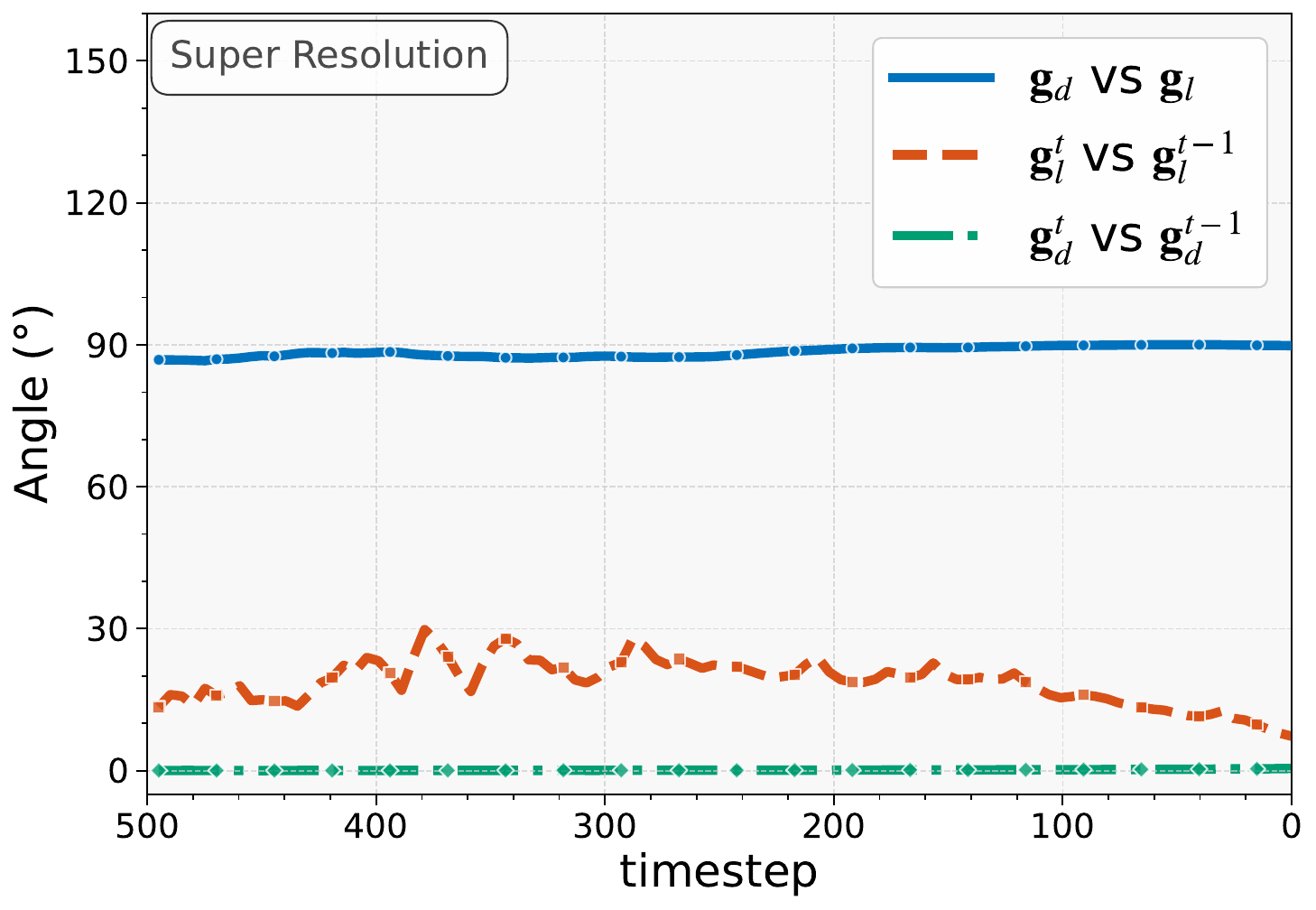}
        \caption{Sample 1}
    \end{subfigure}
    \begin{subfigure}{0.45\textwidth}
        \centering
        \includegraphics[width=0.49\linewidth]{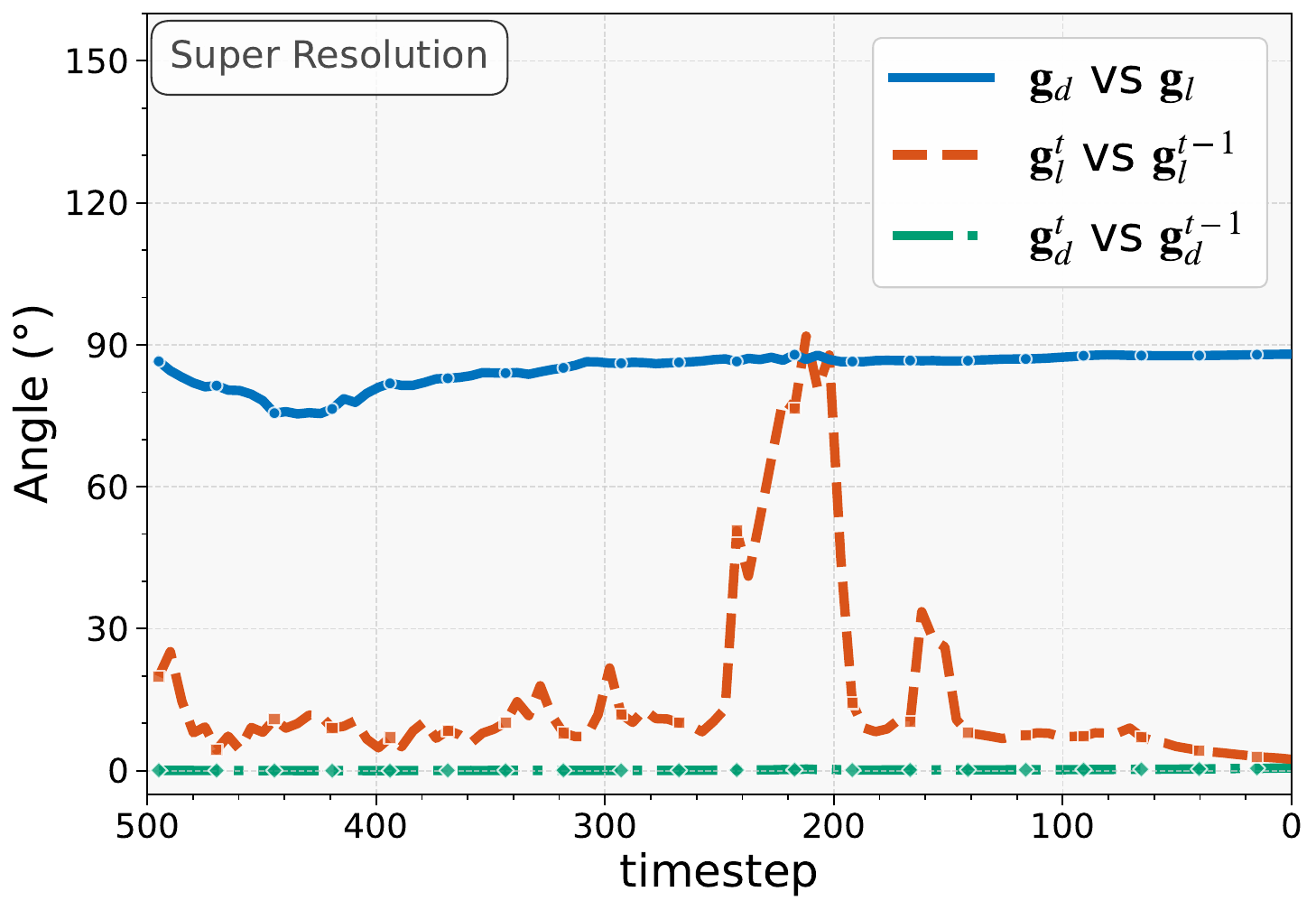}
        \includegraphics[width=0.49\linewidth]{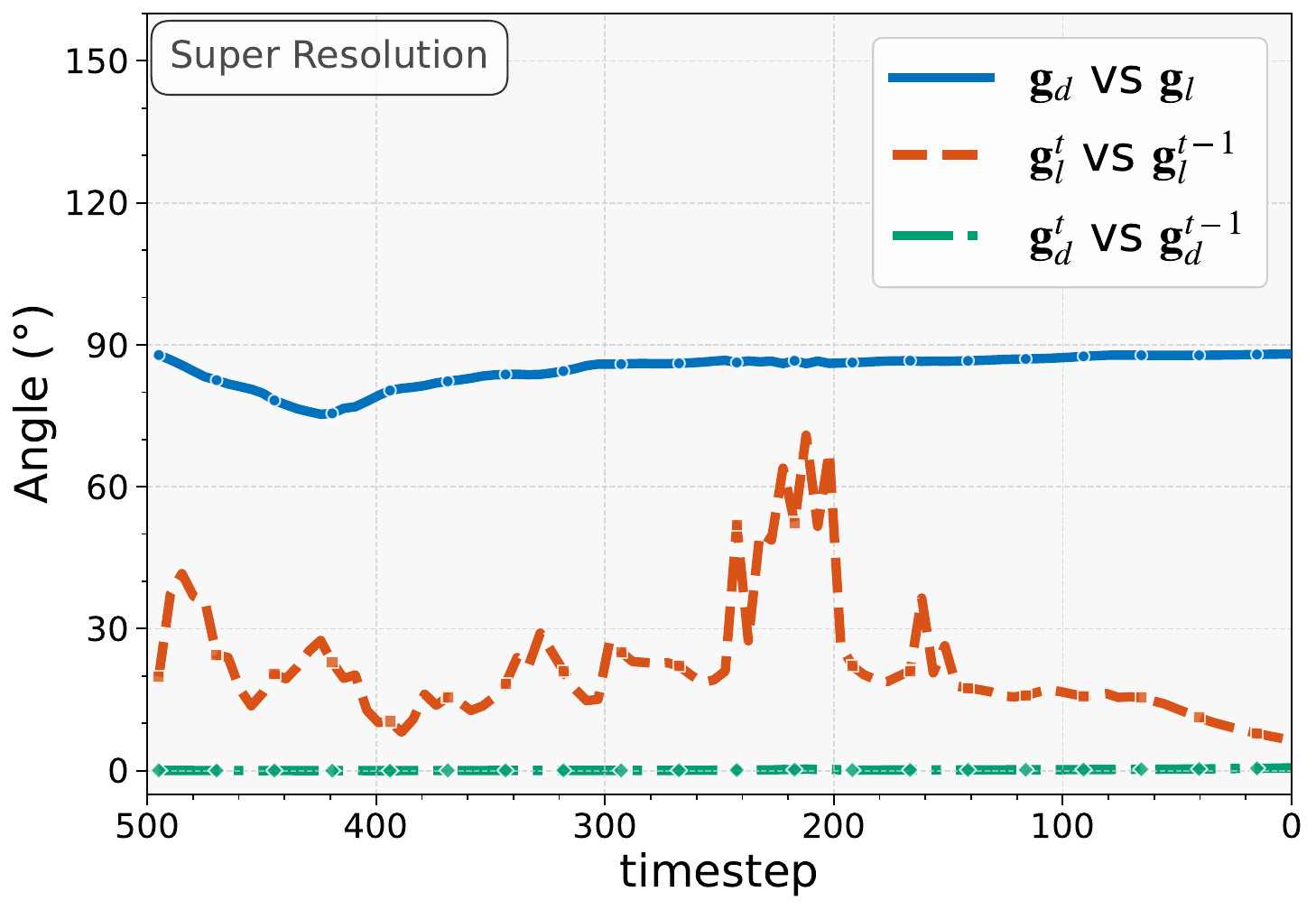}
        \caption{Sample 2}
    \end{subfigure}
        \begin{subfigure}{0.45\textwidth}
        \centering
        \includegraphics[width=0.49\linewidth]{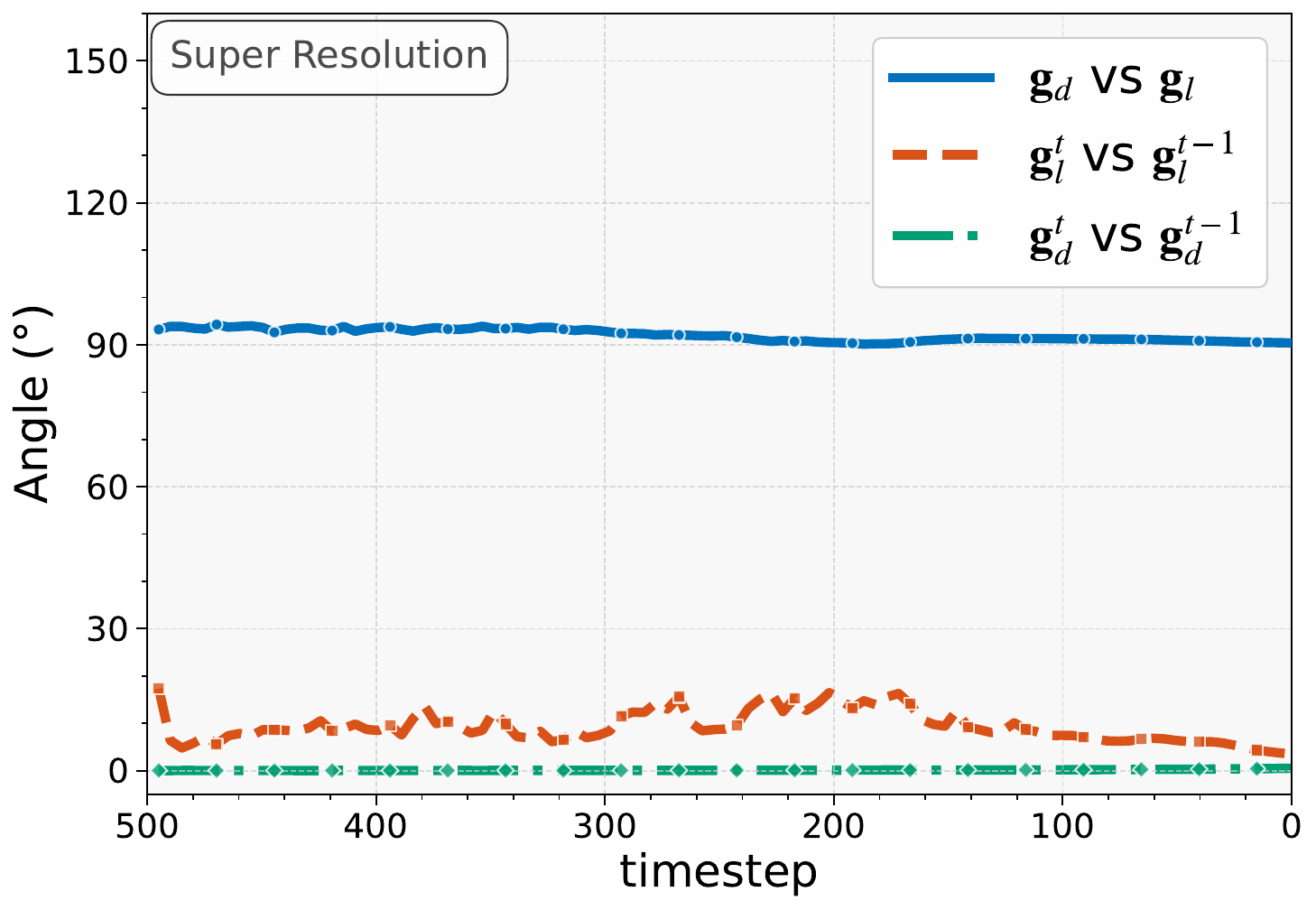}
        \includegraphics[width=0.49\linewidth]{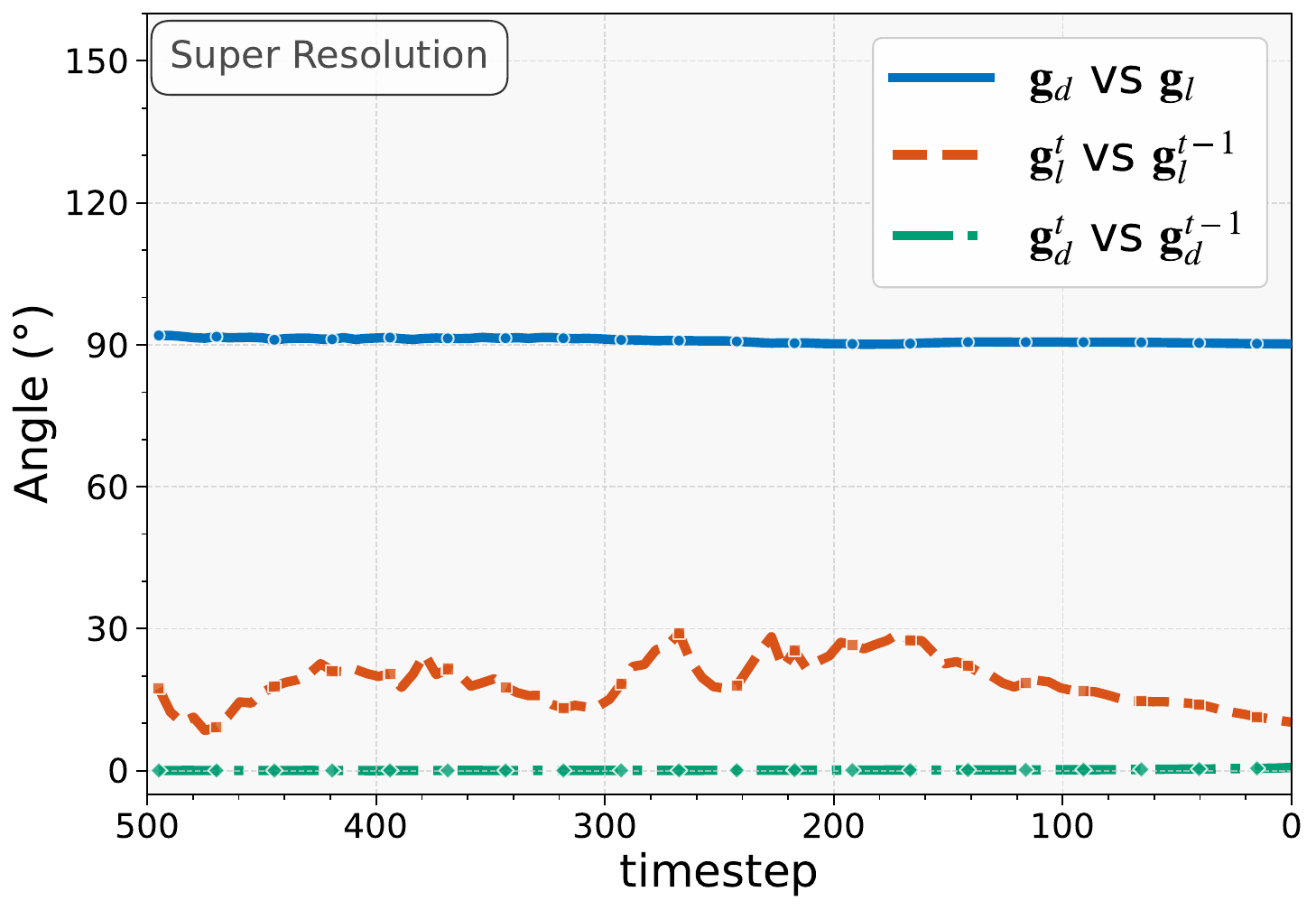}
        \caption{Sample 3}
    \end{subfigure}
    \begin{subfigure}{0.45\textwidth}
        \centering
        \includegraphics[width=0.49\linewidth]{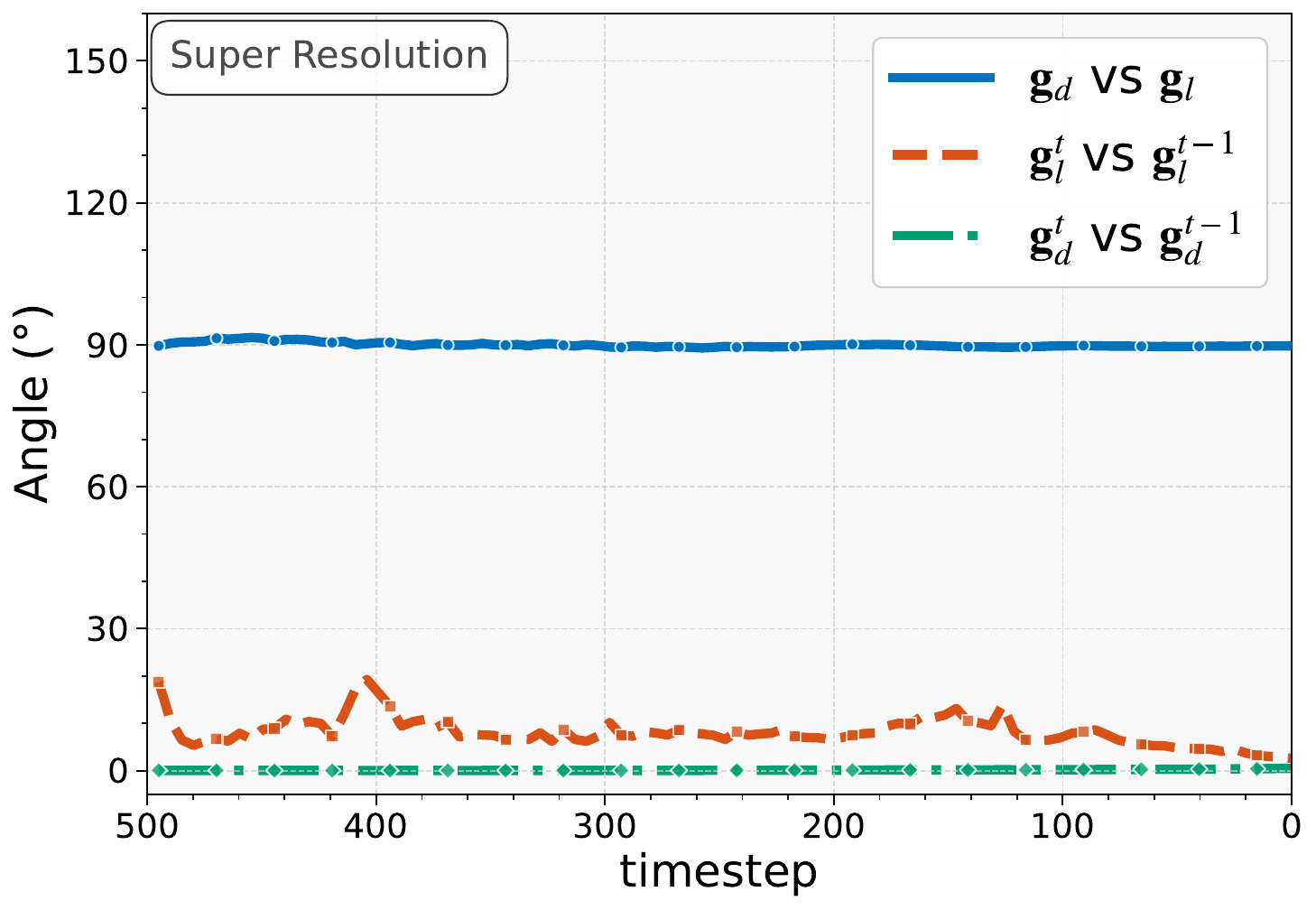}
        \includegraphics[width=0.49\linewidth]{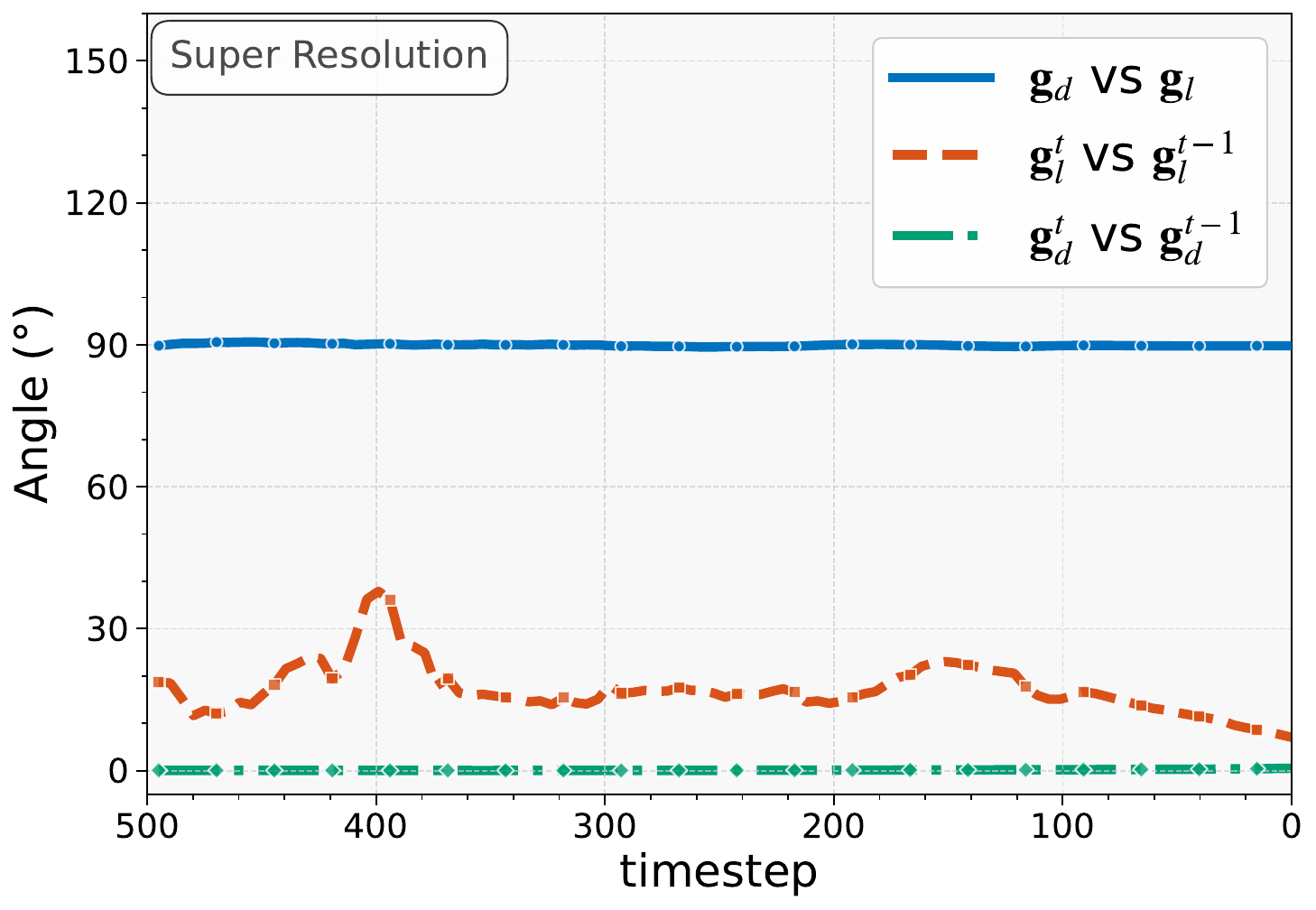}
        \caption{Sample 4}
    \end{subfigure}
        \begin{subfigure}{0.45\textwidth}
        \centering
        \includegraphics[width=0.49\linewidth]{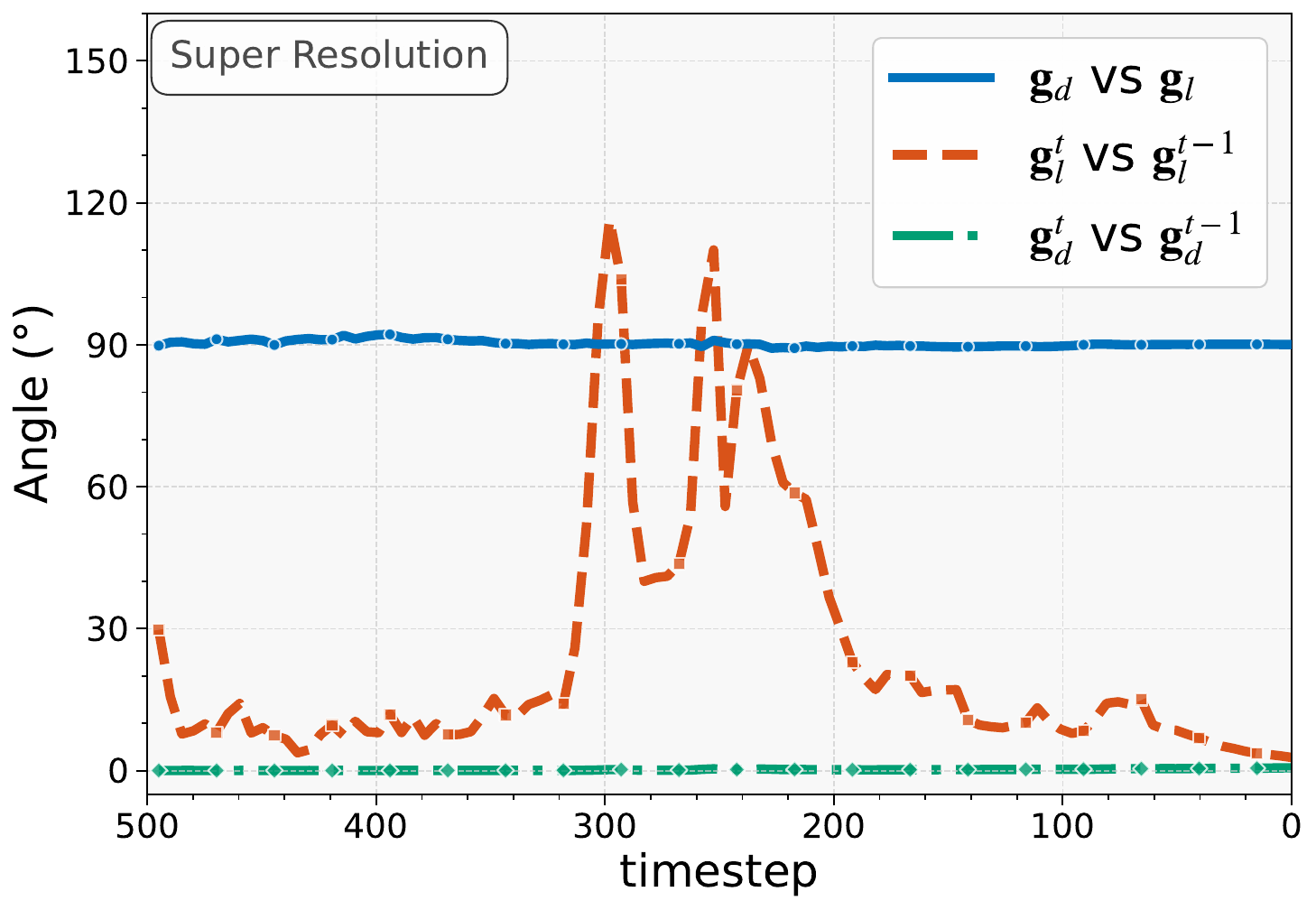}
        \includegraphics[width=0.49\linewidth]{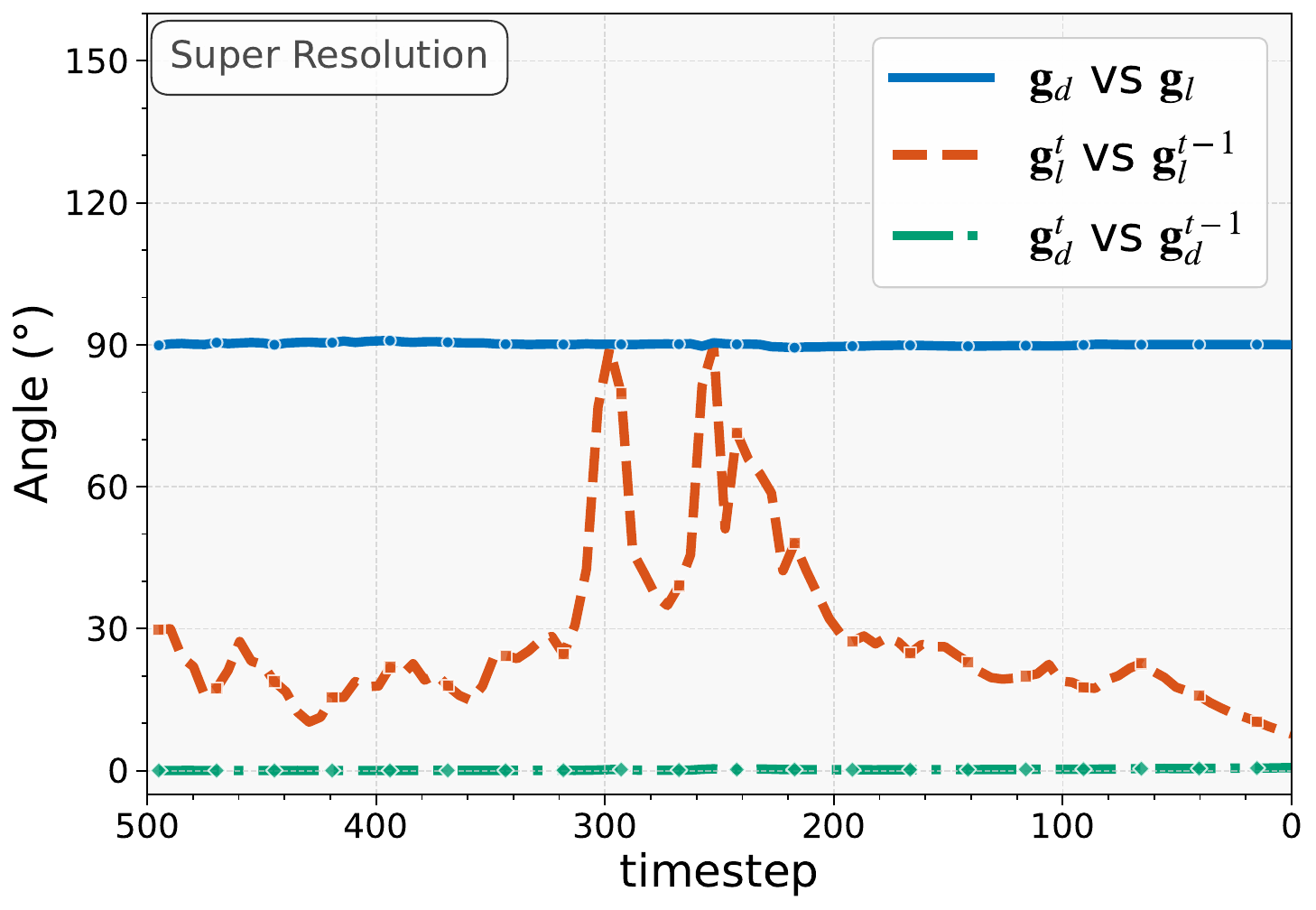}
        \caption{Sample 5}
    \end{subfigure}
        \begin{subfigure}{0.45\textwidth}
        \centering
        \includegraphics[width=0.49\linewidth]{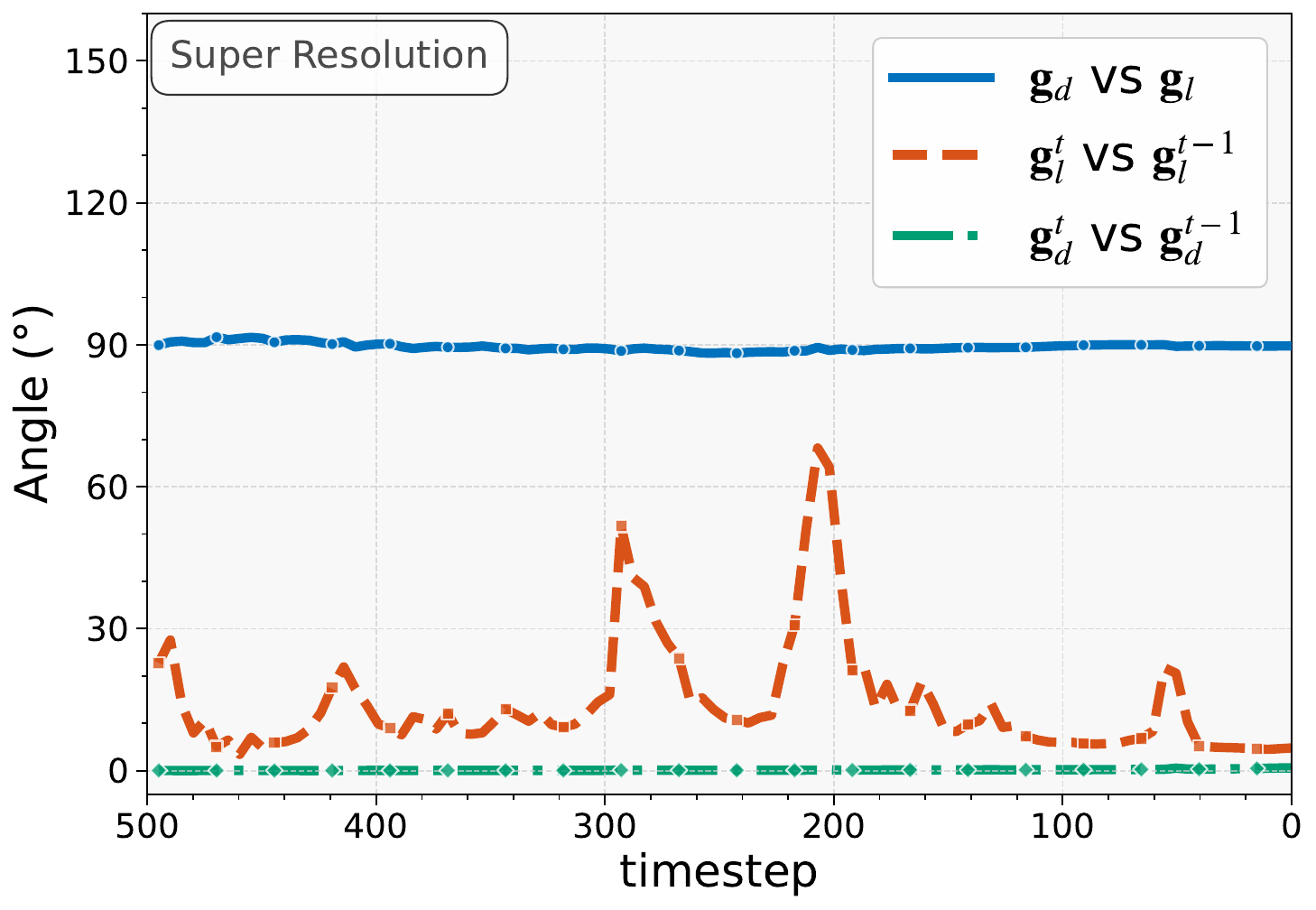}
        \includegraphics[width=0.49\linewidth]{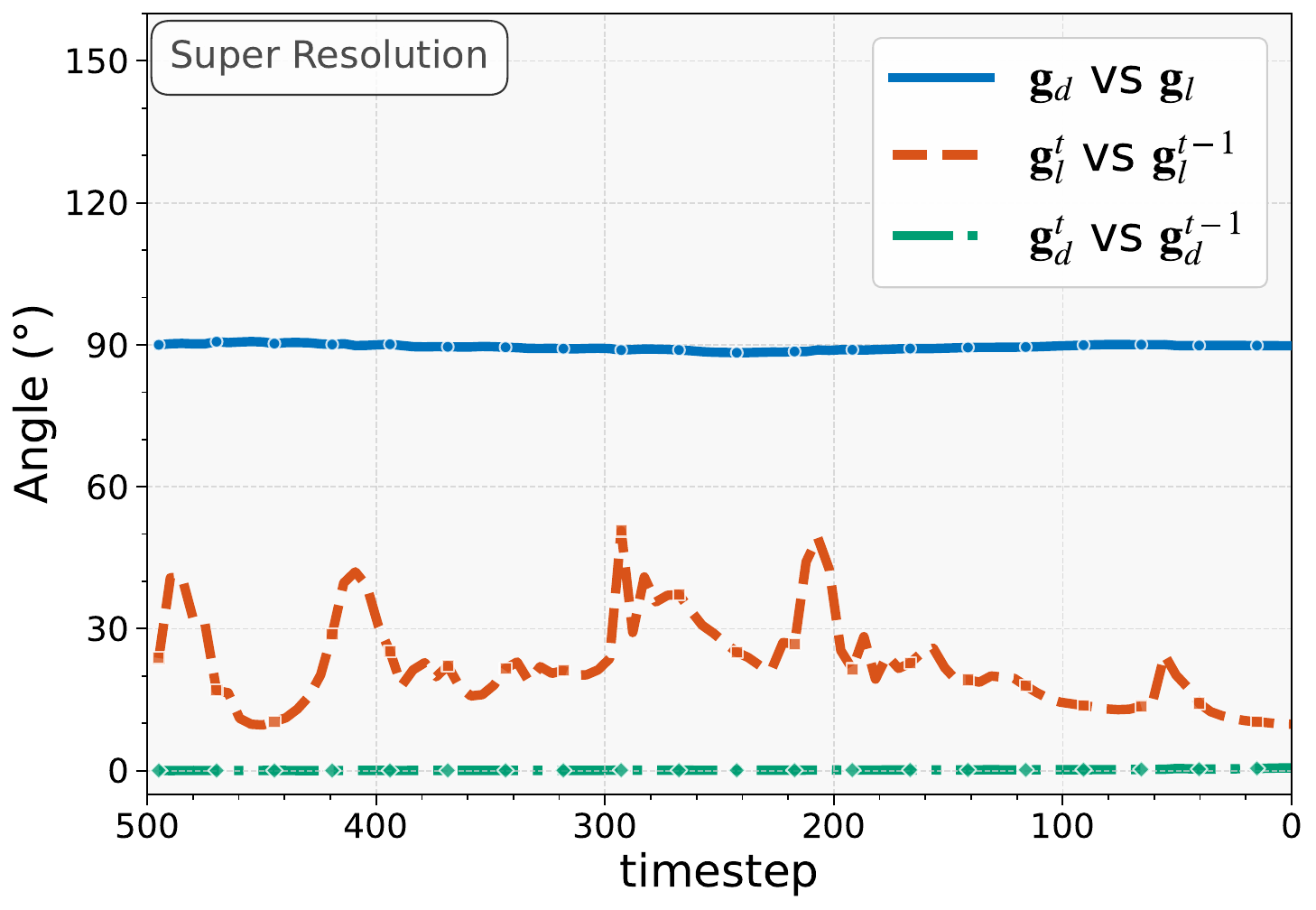}
        \caption{Sample 6}
    \end{subfigure}
    \caption{Gradient angles dynamics before (left) and after ADM smoothing (right) on SR $\times$4.}
    \label{fig:last4}
\end{figure*}


\end{document}